\documentclass[accepted]{uai2025} 
                        
\usepackage[american]{babel}

\usepackage{natbib} 
\usepackage{mathtools} 
\usepackage{booktabs} 
\usepackage{tikz} 

\usepackage{amsmath}
\usepackage{amssymb}
\usepackage{mathtools}
\usepackage{amsthm}
\usepackage[normalem]{ulem}
\usepackage{subcaption}

\usepackage{xcolor}

\DeclareMathOperator*{\argmin}{arg\,min}

\usepackage[capitalize,noabbrev]{cleveref}

\theoremstyle{plain}
\newtheorem{theorem}{Theorem}[section]

\theoremstyle{definition}

\newtheorem{assumption}[theorem]{Assumption}
\theoremstyle{remark}

\title{Mixup Regularization: A Probabilistic Perspective}

\author[1]{\href{mailto:<yousef.el-laham@jpmchase.com>?Subject=Your UAI 2025 paper}{Yousef El-Laham}{}}
\author[1]{Niccol\`{o} Dalmasso}
\author[1]{Svitlana Vyetrenko}
\author[1]{Vamsi K. Potluru}
\author[1]{Manuela Veloso}
\affil[1]{%
    J.P. Morgan AI Research\\
    New York, NY, USA
}
  
\begin{document}
\maketitle

\begin{abstract}
In recent years, mixup regularization has gained popularity as an effective way to improve the generalization performance of deep learning models by training on convex combinations of training data. While many mixup variants have been explored, the proper adoption of the technique to conditional density estimation and probabilistic machine learning remains relatively unexplored. This work introduces a novel framework for mixup regularization based on probabilistic fusion that is better suited for conditional density estimation tasks. For data distributed according to a member of the exponential family, we show that likelihood functions can be analytically fused using log-linear pooling. We further propose an extension of probabilistic mixup, which allows for fusion of inputs at an arbitrary intermediate layer of the neural network. We provide a theoretical analysis comparing our approach to standard mixup variants. Empirical results on synthetic and real datasets demonstrate the benefits of our proposed framework compared to existing mixup variants. 
\end{abstract}

\section{Introduction}
\label{sec:intro}
Mixup regularization has become a notable technique for improving generalization in deep learning for supervised learning tasks \citep{zhang2017mixup}. Models trained with mixup backpropagate the loss function on random convex combinations of training sample pairs, using a mixing coefficient sampled from a Beta distribution. Theoretically, mixup has been analyzed as a form of \emph{vicinal risk minimization} (VRM) \citep{chapelle2000vicinal, zhang2020does, carratino2022mixup}, contrasting with \emph{empirical risk minimization} (ERM) \citep{vapnik1999overview}. Empirical results show mixup effectively improves out-of-sample performance while maintaining competitive performance with state of the art models in data modalities beyond tabular data, such as image \citep{guo2019mixup, liu2023mixmae, islam2024diffusemix, wang2024enhance}, natural language processing \citep{guo1905augmenting, sun2020mixup, zhang2020seqmix}, graph \citep{verma2021graphmix, han2022g, jeongigraphmix} and speech \citep{tokozume2018between, zhang2022contrastive}. Overall, many variations have been introduced to address different aspects of the original strategy, including the choice of mixing coefficient distribution, modifications for different data modalities, and label/feature mixing strategies; we refer the reader to \cite{jin2024survey} for a comprehensive survey.

In tandem with these developments and the rise of large generative language models, the idea of combining or ``fusing" multiple deep learning models has gained attention as a means to improve model capacity \citep{cai2024survey}. Key frameworks include mixture of experts, which partitions the input space and assigns different models to different regions of the space, and product of experts, which combines model outputs by multiplying their probability distributions \citep{jacobs1991adaptive, hinton2002training} and have been successfully incorporated into transformer architectures to train large scale language models \citep{lepikhin2020gshard, jiang2024mixtral, liu2024deepseek}. Related to this are deep ensembling and its variants, such as deep Gaussian mixture ensembling, have been adopted to improve performance and robustness \citep{lakshminarayanan2017simple, el-laham2023deep}. These frameworks can be viewed as probabilistic fusion of multiple predictors; see \citet{koliander2022fusion} for a review on fusing probability density functions.

Our work is motivated by the success of probabilistic fusion, which has not yet found its footing in the mixup literature. Traditionally, mixup fuses data (or their embeddings) directly rather than operating on statistical manifolds, which are represented by random variables and their corresponding density functions. This limits mixup's application to random variables; adapting mixup to more general settings like conditional density estimation or probabilistic machine learning could leverage probabilistic fusion benefits. In this paper, we introduce \textit{Prob}abilistic \textit{Mix}up (\textsf{ProbMix}), a general framework for handling uncertainty in mixup regularization by adapting the methodology to a probabilistic setting using the idea of probabilistic fusion. As this is the first approach of its kind, we consider supervised learning tasks using both tabular and time series data, and comparing with the most used variants in such settings: (i) vanilla mixup \citep{zhang2017mixup}, (ii) manifold mixup \citep{verma2019manifold, el2024augment}, which constructs mixup augmentations on intermediate layers of the neural network and (iii) local mixup \citep{guo2019mixup} which performs mixup augmentations locally based on the assigned class of the data to address the problem of manifold intrusion.

The contributions of this work are: 
\begin{enumerate}
    \item We present a novel reformulation of mixup from a probabilistic perspective called \textsf{ProbMix}, that regularizes an arbitrary model by fusing likelihood functions from different training samples. 
    We show that log-linear fusion of likelihoods is analytically tractable for 
    exponential families members, allowing for easy implementation for both classification and regression settings.
    \item We propose an extension of \textsf{ProbMix} called \textsf{M-ProbMix} that allows for probabilistic fusion at any intermediate layer of the conditional density estimator. 
    \item We provide theoretical results showing that for certain choices of the fusion function, mixup and manifold mixup are special cases of \textsf{ProbMix} and \textsf{M-ProbMix}. 
    \item We demonstrate the competitive or superior performance  of \textsf{ProbMix} and \textsf{M-ProbMix} on classification and regression tasks on several real datasets in terms of uncertainty calibration on out-of-sample data.
\end{enumerate}


\section{BACKGROUND}

\subsection{Problem Setting}
\label{sec: problem_setting}
This work studies generalization in supervised learning problems from the perspective of statistical learning theory.  In supervised learning, we are interested in learning a function $f: x\in{\cal X}\subseteq\mathbb{R}^{d_x} \to y\in{\cal Y}\subseteq\mathbb{R}^{d_y}$, for some $d_x, d_y \in \mathbb{N}^+$. Suppose that the underlying random variables $X$ and $Y$  have a joint probability density function (pdf) $(X, Y)\sim p_{\rm data}(x, y)$. Our goal is to learn a function $f\in{\cal F}$ that minimizes the risk $R[f]$ defined as:
\begin{align}
    \label{eq: expected_risk}
    R[f] &= \mathbb{E}[\ell(f(x), y)] \\
    &= \int_{{\cal X}\times {\cal Y}} \ell(f(x), y) p_{\rm data}(x, y) {\rm d}x{\rm d}y,
\end{align}
where  $\ell: {\cal Y}\times{\cal Y}\rightarrow \mathbb{R}$ is a loss function that measures the discrepancy between a prediction $\hat{y}=f(x)$ and the true output $y$. Typically, one restricts $f$ to belong to a parametric family of functions $f_\theta\in {\cal F}_{\theta}$ defined by parameters $\theta\in\Theta$. In this context, the goal is to learn the parameters $\widehat\theta$ such that they solve the following optimization problem:
\begin{equation}
    \label{eq: minimizing_expected_risk}
    \widehat\theta = \argmin _{\theta \in \Theta} R(\theta), 
\end{equation}
where $R(\theta)=\mathbb{E}[\ell(f_\theta(x), y)]$. Since $p_{\rm data}(x, y)$ is unknown, it is common to find the optimal $\widehat\theta$ by minimizing an approximation of \eqref{eq: minimizing_expected_risk}. The most common approach is ERM, where parameters are learned by minimizing an approximation of the risk using the empirical distribution of a dataset of i.i.d. observations ${\cal D}=\{(x_i, y_i)\}_{i=1}^n$:
\begin{equation}
    \label{eq: erm}
    \widehat\theta_{\rm ERM} = \argmin_{\theta\in\Theta} \frac{1}{n}
    \sum_{i=1}^n \ell(f_{\theta}(x_i), y_i).
\end{equation}
The optimization problem in \eqref{eq: erm} is a result of considering the following empirical approximation to $p_{\rm data}(x, y)$:
\begin{equation}
    \label{eq: empirical_approx}
    p_{\rm data}(x, y) \approx \frac{1}{n}\sum_{i=1}^n \delta_{X, Y}(x_i, y_i),
\end{equation}
where $\delta_{X, Y}(x, y)$ is used to denote a Dirac measure centered at $(x, y)$. Importantly, for conditional density estimation tasks, when the loss is defined as $\ell(f_\theta(x), y)=-\log p_\theta(y|x)$, where $-\log p_\theta(y|x)$ denotes the negative log-likelihood (NLL),  ERM is equivalent to maximum likelihood estimation of the parameters. 


Under the assumption that the training data are independent and identically distributed (i.i.d) and reflect the distribution of out-of-sample data, the ERM principle will lead to a model that generalizes well to out-of-sample data in the limit of infinite training data. This is due to the fact that \eqref{eq: erm} will produce the same result as the true risk minimization problem in \eqref{eq: minimizing_expected_risk} as $n\rightarrow\infty$. Typically, neither of these requirements are satisfied in practice; that is to say that: (a) training data are not infinite and in some cases scarcely available, and (b) out-of-sample data may not conform exactly to the distribution of the observed dataset. To improve performance on out-of-sample data, different approximations of $p_{\rm data}(x, y)$ can be considered to regularize $f_\theta$ to have better generalization properties. 

\subsection{Vicinal Risk Minimization and Mixup}
\label{sec: background}
To combat the issue of limited data availability and out-of-sample distribution mismatch, the 
VRM principle can be utilized. VRM does not directly use the empirical density of the data to approximate the risk, but rather uses a ``perturbed" version of it. Let $\tilde{p}_{\nu, {\cal D}}(x, y)$ denote a joint pdf called the {\rm vicinal distribution} which depends on training examples $(x_i, y_i)\in {\cal D}$ and potentially some additional hyperparameters $\nu$. Then, under the joint pdf $\tilde{p}_{\nu, {\cal D}}(x, y)$, the risk can be approximated as follows:
\begin{align}
    \label{eq: expected_risk_vrm}
    \tilde{R}_{\nu}(\theta) &= \int_{{\cal X}\times {\cal Y}} \ell(f_\theta(x), y) \tilde{p}_{\nu, {\cal D}}(x, y){\rm d}x{\rm d y}
\end{align}
Note that $\tilde{p}_{\nu, {\cal D}}(x, y)$ extends the computation of the risk from the exact values of the pair $(x_i, y_i)$ -- as in $p_{\rm data}(x,y)$ in equation~\ref{eq: empirical_approx} -- to a neighborhood of $(x_i, y_i)$.
In general, regularization based on data augmentation, adversarial training, and label smoothing can be viewed as a form of VRM. To that end, it can be shown that classical mixup and its variants can be viewed as a form of VRM.

\paragraph{Vanilla Mixup.}
Mixup is a VRM technique that constructs augmented samples by taking random convex combinations of existing ones. The corresponding vicinal distribution in vanilla mixup is:
\begin{equation}
    \label{eq: vicinal_distribution_mixup}
    \tilde{p}_{\alpha, {\cal D}}(x, y) = \frac{1}{n^2} \sum_{i=1}^n \sum_{j=1}^n \mathbb{E}_{\lambda}\left[\delta_{X, Y}(\tilde{x}_{i, j, \lambda}, \tilde{y}_{i, j, \lambda})\right],
\end{equation}
where $\lambda$ is a mixing coefficient, usually assumed to follow a beta distribution ${\cal B}(\alpha, \alpha)$, with equal shape and scale $\alpha>0$. We define $\tilde{x}_{i, j, \lambda}$ and $\tilde{y}_{i, j, \lambda}$ as
\begin{align}
    &\tilde{x}_{i, j, \lambda} = \lambda x_i + (1-\lambda) x_j \\
    &\tilde{y}_{i, j, \lambda} = \lambda y_i + (1-\lambda) y_j
\end{align}
Based on this vicinal distribution, the overall loss function that is minimized in vanilla mixup is the following:
\begin{equation}
    \label{eq: standard_mixup_risk}
    \tilde{R}_{\alpha}^{\rm mix}(\theta) = \frac{1}{n^2} \sum_{i=1}^n \sum_{j=1}^n \mathbb{E}_{\lambda}\left[\ell(f_\theta(\tilde{x}_{i, j, \lambda}), \tilde{y}_{i, j, \lambda})\right]
\end{equation}
As $\alpha\rightarrow\infty$, the random variable $\lambda$ converges to a degenerate random variable such that $\mathbb{P}(\lambda=\frac{1}{2})=1$. In contrast, as $\alpha\rightarrow 0$, the random variable $\lambda$ converges to a Bernoulli random variable such that $\mathbb{P}(\lambda=0)=\mathbb{P}(\lambda=1)=\frac{1}{2}$, in which case, mixup reduces to ERM as in \eqref{eq: erm}. 
We refer the reader to Appendix \ref{app: review_mixup} for a detailed review of manifold mixup and local mixup.

\section{Our Methodology}
\label{sec: methodology}
In this section, we introduce a general methodology for extending mixup regularization to statistical manifolds which we call \emph{\textbf{ Prob}abilistic \textbf{Mix}up} (\textsf{ProbMix}). We begin by discussing how to apply mixup to the likelihood functions obtained via a generic conditional density estimator for both regression and classification using linear and log-linear fusion.  Following the methodology for fusing likelihood functions, we present an extension of our approach, called \emph{\textbf{M}anifold \textbf{Prob}abilistic \textbf{Mix}up} (\textsf{M-ProbMix}), that allows for probabilistic mixup at any intermediate layer by introducing a conditional density mapping at the desired layer. As long as the density mapping is an exponential family member, we show that applying \textsf{ProbMix} to an intermediate layer of a neural network is analytically tractable. Figure \ref{fig: fusion_mixup_vs_probmix} provides a summary of the training forward pass of \textsf{ProbMix} as compared to mixup-based approaches.

\begin{figure*}[t]
    \centering
    \includegraphics[trim={0cm, 13.5cm, 0cm, 0cm},clip, width=\linewidth]{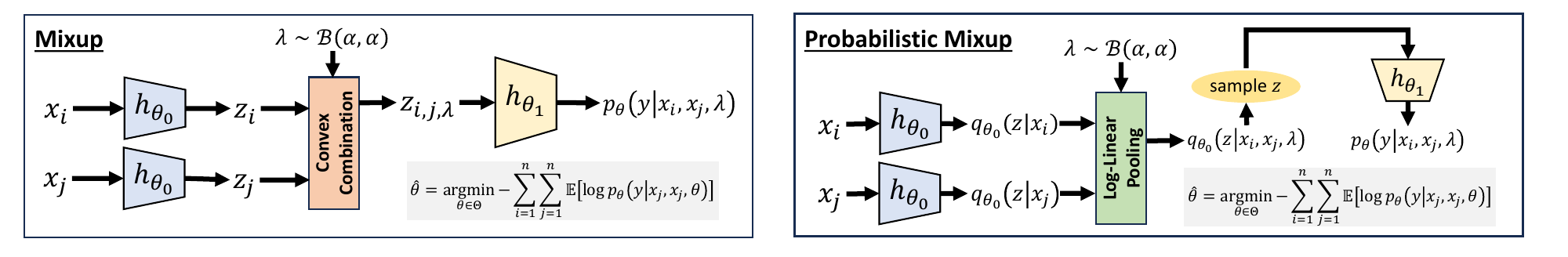}
    \caption{Summary of training forward passes for both mixup vs. \textsf{ProbMix}.}
    \label{fig: fusion_mixup_vs_probmix}
\end{figure*}

\subsection{Data Generating Process}
\label{sec: dgm_prob_mixup}
The key distinction between \textsf{ProbMix} and other mixup regularization techniques is the assumption that the responses $y$ are assumed to have been independently generated from an interpolated conditional density function
\begin{equation}
    \label{eq: fusion_prob_mixup}
    \tilde{p}_\theta(y|x_i, x_j, \lambda)=g_\lambda^x(p_\theta(y|x_i), p_\theta(y|x_j))
\end{equation}
obtained by some fusion function $g_\lambda^x$, rather than the conditional density $p_\theta(y|x_{i, j, \lambda})$,   
which conditions based on interpolated features. Let $p(\lambda; \alpha)$ denote a fixed and known pdf over $\lambda$ with tunable hyperparameters $\alpha$ and let ${\cal G}=({\cal D}, {\cal E}, {\cal W})$ denote a weighted graph defined over the observed dataset. We consider the following data generating process in \textsf{ProbMix}:
\begin{align}
    &\lambda \sim p(\lambda; \alpha), \quad (i, j) \sim {\cal G} \\
    &y \sim \tilde{p}_\theta(y|x_i, x_j, \lambda)
\end{align}
where we use the notation $(i, j)\sim {\cal G}$ to denote the generation of an edge $(i, j)\in {\cal E}$ from the graph ${\cal G}$ based on the weights ${\cal W}$. Furthermore, for a given edge $(i, j)$, we assume that the observed label $\tilde{y}$ has a vicinal density
\begin{equation}
    \tilde{y} \sim p_\beta(\tilde{y}|y_i, y_j, \lambda)=g_\lambda^y(s_\beta(\tilde{y}|y_i), s_\beta(\tilde{y}|y_j)),
\end{equation}
obtained by some fusion function $g_\lambda^y$, where $s_{\beta}(\tilde{y}|y)$ denotes a kernel centered around $y$ with hyperparameter $\beta$. 
To summarize, the modeling assumption of \textsf{ProbMix} differs from vanilla mixup in the following two ways:
\begin{enumerate}
    \item {\bf Data generating process}: In \textsf{ProbMix} the responses are generated from a fusion of conditional densities (guided by a fusion function $g_\lambda^x$) each of which are conditioned on distinct features $x_i$ and $x_j$, respectively. In vanilla mixup, 
    the responses are generated from a single density conditioned on a mixture of features. 
    \item {\bf Observed responses}: In \textsf{ProbMix}, the responses are assumed to be latent, but with known density $g_\lambda^y(s_\beta(\tilde{y}|y_i), s_\beta(\tilde{y}|y_j))$. In mixup, the responses are observed and assumed to be a convex combination of the corresponding edge that is being mixed. 
\end{enumerate}

\subsection{Optimizing Model Parameters}
\label{subsec: probmixup_optimization_criteria}
The goal of \textsf{ProbMix} is to maximize the expected log-likelihood of the latent observations $\tilde{y}$, which can be expressed using the law of iterated expectations as:
\begin{align*}
    l(\theta; &\alpha, {\cal G}) = \mathbb{E}_{\lambda, \tilde{y}}\left[\log \tilde{p}_\theta(\tilde{y}|x_i, x_j, \lambda)\right] \\ 
    &= \mathbb{E}_{\lambda}\left[\mathbb{E}_{\tilde{y}}\left[\sum_{(i, j)\in {\cal E}} w_{i, j} \log \tilde{p}_\theta(\tilde{y}|x_i, x_j, \lambda) \Bigg| \lambda \right]\right]\\
    &= \sum_{(i, j)\in{\cal E}} w_{i, j} \mathbb{E}_{\lambda}\left[\mathbb{E}_{\tilde{y}}[\log \tilde{p}_\theta(\tilde{y}|x_i, x_j, \lambda)|\lambda]\right],
\end{align*}
where the inner expectation is taken w.r.t. ${p}_\beta(\tilde{y}|y_i, y_j, \lambda)$  and the outer expectation is taken w.r.t. $p(\lambda; \alpha)$. Equivalently, we can obtain the corresponding loss function from a VRM perspective by replacing the loss function in \eqref{eq: local_mixup_loss} with the expected NLL of the fusion function:
\begin{equation}
    \label{eq: risk_prob_mixup}
    \tilde{R}_{\alpha, {\cal G}}^{\mathbb{P}}(\theta) = -\hspace{-0.3cm}\sum_{(i, j)\in{\cal E}} w_{i,j} \mathbb{E}_{\lambda}[\mathbb{E}_{\tilde{y}}[\log p_\theta(\tilde{y}|x_i, x_j, \lambda)|\lambda]]
\end{equation}
Depending on the choice of fusion functions $g_\lambda^x$ and $g_{\lambda}^y$, the overall loss function in \eqref{eq: risk_prob_mixup} will exhibit different regularization effects.  We remark that one can also maximize the logarithm of the expected likelihood, which is related to the former by Jensen's inequality. We provide more details about this alternative optimization criterion in Appendix \ref{app: loss_discussion}. 

\subsubsection{Monte Carlo Approximation} 
To minimize $\tilde{R}_{\alpha, {\cal G}}^{\mathbb{P}}(\theta)$, we can use the Monte Carlo approach to obtain stochastic gradients where instead of taking gradients with respect to $\tilde{R}_{\alpha, {\cal G}}^{\mathbb{P}}(\theta)$, we take gradients of an estimator given by:
\begin{equation}
    \label{eq: risk_prob_mixup_monte_carlo}
    \tilde{R}_{\alpha, {\cal G}}^{\mathbb{P}}(\theta) \approx -\sum_{(i, j)\in{\cal E}}\frac{w_{i, j}}{K}\sum_{k=1}^K\log p_\theta(\tilde{y}^{(k)}|x_i, x_j, \lambda^{(k)}),
\end{equation}
where $\lambda^{(k)} \sim p(\lambda; \alpha)$ and $\tilde{y}^{(k)}\sim p_\beta(\tilde{y}|y_i, y_j, \lambda^{(k)})$ for $k=1,\ldots, K$. Typically, a single sample ($K=1$) is used.

\subsubsection{Linear vs. Log-Linear Fusion}\label{sec:linear-log-linear-fusion}
The \textsf{ProbMix} framework requires choosing the fusion functions $g_\lambda^x$ and $g_\lambda^y$. Two popular choices often found in the probabilistic fusion literature are linear pooling:
\begin{equation}
    \label{eq: linear_pooling}
    g_\lambda^x(p_\theta(y|x_i), p_\theta(y|x_j)) = \lambda p_\theta(y|x_i) + (1-\lambda)p_\theta(y|x_j)
\end{equation}
and log-linear pooling:
\begin{equation}
    \label{eq: log-linear_pooling}
    g_\lambda^x(p_\theta(y|x_i), p_\theta(y|x_j)) \propto  \left[p_\theta(y|x_i)\right]^\lambda\left[p_\theta(y|x_j)\right]^{1-\lambda}
\end{equation}
In this work, we utilize the log-linear pooling function, since in the case of exponential family members, the fusion result also belongs to exponential family of probability distributions. Important special cases of this result include the categorical distribution and the Gaussian distribution, which are often the assumed statistical models in classification and regression tasks, respectively. The detailed proof and relevant analytical expressions related to this result can be found in Appendix~\ref{thm: exponential_families}.

\subsubsection{Why ProbMix over Vanilla Mixup?} \label{sss: illustrative_example}
Here, we provide an intuitive example demonstrating a scenario where \textsf{ProbMix} is preferred over vanilla mixup. Consider the following data generating process: 
\begin{equation}
    \label{eq: illustrative_example_dgp}
    y = x^3 + (0.5x^2+1)\epsilon, \quad \epsilon\sim {\cal N}(0, 1),
\end{equation}
where the ground truth mean and variance functions are $\mu_\theta(x)=x^3$ and $\sigma_\theta^2(x)=(0.5x^2+1)^2$, respectively. Suppose two samples $(x_1, y_1)=(5, 130)$ and $(x_2, y_2)=(-5, -120)$ are observed. We would like to understand the regularization effect of both mixup and \textsf{ProbMix} on the ground truth mean and variance functions based on the two observed samples for a mixing coefficient of $\lambda=0.8$.  In the case of vanilla mixup, the predicted conditional is: 
\begin{equation}
    \label{eq: illustrative_example_mixup}
    p_{\sf Mix}(\tilde{y}|x_1, x_2,\lambda=0.8) = {\cal N}(\tilde{y}|27, 30.25),
\end{equation}
while for \textsf{ProbMix} it is:
\begin{equation}
    \label{eq: illustrative_example_probmixup}
    p_{\sf ProbMix}(\tilde{y}|x_1, x_2,\lambda=0.8) = {\cal N}(\tilde{y}|75, 182.25).
\end{equation}
Assuming $\beta\rightarrow 0$ for \textsf{ProbMix}, the interpolated observation $\tilde{y}=\lambda y_1+(1-\lambda)y_2$ for which the likelihood is evaluated is the same for both approaches and is $\tilde{y}=80$. Since the interpolated observation $\tilde{y}=80$ lies in the right tail of $p_{\sf Mix}$, gradient updates made to the mean and variance functions based on mixup will substantially alter the model, despite the fact that $\mu_\theta(x)$ and $\sigma_\theta^2(x)$ are the ground truth mean and variance functions. In contrast, \textsf{ProbMix}'s fused density is much better calibrated to the interpolated target. This demonstrates that in this example, mixup enforces a strong linear bias on the mean and variance functions, which can be inappropriate when the true relationship is nonlinear or when the input features being fused are far apart. By working on the statistical manifold when mixing, \textsf{ProbMix} avoids this strong bias. We provide details of our calculations and conditional density plots related to this illustrative example in Appendix \ref{app: illustrative_example}.

\subsection{Manifold Probabilistic Mixup} 
We now discuss an extension of \textsf{ProbMix}, called \textsf{M-ProbMix}, that allows for probabilistic fusion in an arbitrary embedding defined by an intermediate layer of the neural network. The idea behind the approach is to consider that our predictor  $f_\theta = h_{\theta_1}\circ h_{\theta_0}$ is the composition of the mappings $h_{\theta_0}$ and $h_{\theta_1}$. The mapping $h_{\theta_0}$ maps the input of the predictor to the parameters of the density of an embedding $z$ (e.g., Gaussian distribution with diagonal covariance matrix), while $h_{\theta_1}$ maps from random embedding to the conditional density of the response. Mapping the inputs to a density function in the intermediate layers enables the use of probabilistic fusion to mix samples at the embedding level. 

Let $q_{\theta_0}(z|x)$ denote the parametric density of an embedding $z$ given some input feature $x$. The expected log-likelihood in \textsf{M-ProbMix} can be determined as:
\begin{align}
    &l(\theta; \alpha, {\cal G}, L_{mix}) = \mathbb{E}[\log \tilde{p}_\theta(y|x_i, x_j, \lambda)] \\
    &=\mathbb{E}_\lambda\left[\mathbb{E}_{\tilde y}\left[\sum_{(i, j)\in {\cal E}} w_{i, j}\log p_\theta(\tilde{y}|x_i, x_j, \lambda)\Bigg| \lambda \right]\right],
\end{align}
where the density $ p_\theta(\tilde{y}|x_i, x_j, \lambda)$ is given by:
\begin{equation}
    \label{eq: int_manifold_prob_mix}
   p_\theta(\tilde{y}|x_i, x_j, \lambda) = \int p_{\theta_1}(\tilde{y}|z)q_{\theta_0}(z|x_i, x_j, \lambda) {\rm d}z
\end{equation}
and the density of $z$ given $x_i$, $x_j$ and $\lambda$ is the fusion of the densities of the random embeddings according to some fusion function $g_{\lambda}^z$:
\begin{equation}
    q_{\theta_0}(z|x_i, x_j, \lambda) = g_\lambda^z\left(q_{\theta_0}(z|x_i), q_{\theta_0}(z|x_j)\right).
\end{equation}
We can readily obtain the corresponding risk as:
\begin{equation}
    \tilde{R}_{\alpha, {\cal G}}^{\mathbb{P}, {\cal M}}(\theta) = -\sum_{(i, j)\in {\cal E}} w_{i, j}  \mathbb{E}_{\lambda}[\mathbb{E}_{\tilde{y}}[\log p_{\theta}(\tilde{y}|x_{i}, x_{j}, \lambda)|\lambda]]
\end{equation}
Put simply, \textsf{M-ProbMix} can be viewed as a probabilistic extension to manifold mixup, whereby an intermediate layer maps to a random variable rather than a fixed transformation of the input features.  Just like \textsf{ProbMix}, the risk of \textsf{M-ProbMix} can be approximated using a Monte Carlo estimate, whereby sampling is additionally done at the embedding level in order to approximate the integral in 
\eqref{eq: int_manifold_prob_mix}. Importantly, if one chooses $q_{\theta_0}(z|x)$ to belong to a member of the exponential family, then log-linear pooling can be readily applied. A standard and convenient choice is $q_{\theta_0}(z|x)$ is a Gaussian distribution with diagonal covariance matrix, since the reparameterization trick \citep{kingma2013auto} can be readily applied. Finally, we highlight that the choice of the embedding distribution is agnostic to the learning task, as \textsf{M-ProbMix} focuses on fusing the distributions of the underlying embeddings, rather than the likelihoods themselves.   

\section{Theoretical Insights}\label{sec: theory}
In this section, we provide theoretical insights by comparing the proposed \textsf{ProbMix} and \textsf{M-ProbMix} with mixup and manifold mixup. For simplicity in the presentation, we assume the following:

\begin{assumption}[Negative Log-Likelihood Loss]\label{assumption: nll}
    The loss function $\ell(f_\theta(x), y)=-\log p_\theta(y|x)$. In the case of classification, $\log p_\theta(y|x)$ corresponds to the log probabilities of each class, 
    while in the case of regression, $\log p_\theta(y|x)$ is assumed to be either a homoscedastic or heteroscedastic Gaussian log-likelihood function. 
\end{assumption}

\begin{assumption}[Expected Loss over Labels]\label{assumption: loss-over-labels}
    In the case of classification, label mixing is done by mixing one-hot-encoded vectors, and the risk is taken by taking the expected value over the mixed label probabilities. That is,
    \begin{align*}
        &\mathbb{E}_\lambda[\log p_\theta(y_{i,j, \lambda}|x_{i,j, \lambda})] = \\ &\mathbb{E}_\lambda[\lambda\log p_\theta(y_i|x_{i, j, \lambda})] + \mathbb{E}_\lambda[(1-\lambda)\log p_\theta(y_j|x_{i, j, \lambda})] 
    \end{align*}
    In the case of \textsf{ProbMix}, this corresponds to using the following perturbation distribution for the responses:
    \begin{equation*}
        g_\lambda^y(s_{\beta}(\tilde{y}|y_i), s_{\beta}(\tilde{y}|y_j)) = \lambda \delta_{y_i} + (1-\lambda)\delta_{y_j}
    \end{equation*}
\end{assumption}

Under Assumption~\ref{assumption: loss-over-labels}, the expectation term over the labels for the expected risk of \textsf{ProbMix} and \textsf{M-ProbMix} breaks down into two separate terms due to linearity of expectation. 
Hence, when comparing mixup and \textsf{ProbMix} across different settings, we can simply focus on comparing log-likelihoods terms, i.e., $\log p_\theta(y|x_{i, j, \lambda})$ for mixup methods and $\log p_\theta(y|x_i, x_j, \lambda)$ for probabilistic mixup methods.

\subsection{\textsf{ProbMix} Operates Mixup on the Outputs}

Theorems~\ref{theo: mixup-logits-class} and \ref{theo: mixup-output-regr} show that, under log-linear fusion $g^x_\lambda$ of the likelihoods,  \textsf{ProbMix} can be thought of as operating mixup on the output layers in both classification and regression tasks.

\begin{theorem}[\textsf{ProbMix} as mixup of Logits]\label{theo: mixup-logits-class}
    Under Assumptions~\ref{assumption: loss-over-labels} and \ref{assumption: nll} and log-linear fusion $g^x_\lambda$ of categorical likelihoods, \textsf{ProbMix} is equivalent to vanilla mixup on the logits.
\end{theorem}

\begin{theorem}[\textsf{ProbMix} as Mixup of Means] \label{theo: mixup-output-regr}
    Under Assumptions~\ref{assumption: loss-over-labels} and \ref{assumption: nll} and log-linear fusion $g^x_\lambda$ of likelihoods, 
    \textsf{ProbMix} is equivalent to vanilla mixup of the output means.
\end{theorem}

\begin{proof}[Proof Sketch.] For both theorems, the proof proceeds constructively showing that the likelihoods of \textsf{ProbMix} and mixup on the output layers are proportional to the same quantities.
\end{proof}

\subsection{Mixup as a Special Case of \textsf{ProbMix}} 

Theorems~\ref{theo: mixup-log-reg} and \ref{theo: mixup-lin-reg} show that in both the classification and regression tasks, \textsf{ProbMix} when using a linear mapping $f_\theta$ and log-linear pooling $g^x_\lambda$ reduces to vanilla mixup. Theorem~\ref{theo: manifold-prob} shows that log-linear pooling of homoscedastic Gaussian embeddings makes \textsf{M-ProbMix} reducing to manifold mixup, as long as embedding means are propagated during training. 
We refer the reader to Appendix~\ref{app: proof} for full proofs.

\begin{theorem}[Mixup and \textsf{ProbMix} for Multiclass Logistic Regression]\label{theo: mixup-log-reg}
    In classification tasks, under Assumptions~\ref{assumption: loss-over-labels} and \ref{assumption: nll}, when setting $g_\lambda^x$ as log-linear pooling of categorical distributions and using a multi-class logistic regression learner  $f_\theta$, \textsf{ProbMix} reduces to vanilla Mixup.
\end{theorem}
\begin{proof}[Proof Sketch]
Consider $f_\theta=h_1\circ h_{\theta}$, where $h_1(z)$ 
is the softmax function and $h_{\theta}(x)=Ax+b$ is a linear function with $A=[a_1, \ldots, a_y]^\intercal\in\mathbb{R}^{d_y\times d_x}$ and $b\in\mathbb{R}^{d_y}$. The proof proceeds constructively by showing that the probability of the $k^{th}$ class is identical for both \textsf{ProbMix} in the settings above and vanilla Mixup.
\end{proof}

\begin{theorem}[Mixup and \textsf{ProbMix} for Linear Regression]\label{theo: mixup-lin-reg}
    In regression tasks, under Assumptions~\ref{assumption: loss-over-labels} and \ref{assumption: nll}, when setting $g_\lambda^x$ as log-linear pooling of homoscedastic Gaussian distributions and using a linear regression learner  $f_\theta(x)=Ax+b$, \textsf{ProbMix} reduces to vanilla mixup. 
\end{theorem}
\begin{proof}[Proof Sketch]
The proof proceeds constructively by showing that the log-likelihood is proportional to the same quantities for both \textsf{ProbMix} in the settings above and traditional mixup.
\end{proof}

In the case of probabilistic manifold mixup with the choice of the hidden distribution as a homoscedastic Gaussian,
\begin{theorem}[\textsf{M-ProbMix} under Homoscedastic Gaussian Approximation is Manifold Mixup]\label{theo: manifold-prob}
    Assume a learner $f_{\theta, \phi} = d_\phi \circ h_\theta$, where $h_\theta$ and $d_\phi$ are encoder and decoder, respectively, and means are propagated directly during both training and inference, i.e., $f_{\theta, \phi} = d_\phi(h_\theta(x))$. Then, under Assumptions~\ref{assumption: loss-over-labels} and \ref{assumption: nll}, \textsf{M-ProbMix} using a log-linear fusion a homoscedastic Gaussian embeddings is equivalent to manifold mixup.
\end{theorem}

\begin{proof}[Proof Sketch]
As log-linear fusion of Gaussians is Gaussian (see Section~\ref{sec:linear-log-linear-fusion} and Appendix~\ref{thm: exponential_families}), and the means are propagated during training and inference, one can show the log-likelihoods of \textsf{M-ProbMix} and manifold mixup are equal.
\end{proof}

\section{Practical Considerations}


%
\paragraph{Sampling strategies:} When selecting pairs of samples for training, one option is to uniformly sample across all possible pairs in the training data, as in classical mixup. This corresponds to uniformly sampling edges in a fully connected graph where training samples are nodes. However, the edges between a pair of points  $(x_i, x_j)$ can be assigned weights $w_{i, j}$ to enforce non-uniform sampling strategies. In this work,   we consider two types of sampling graphs: a \emph{fully-connected graph}, such that $w_{i, j}=\frac{1}{n^2}$ for all edges $(i, j)\in{\cal E}$; and a \emph{nearest neighbors graph}, where $w_{i, j} \propto 1$ for all $(i, j)\in {\cal E}$ such that $x_j \in {\cal C}_K(x_i)$, where ${\cal C}_K(x_i)$ denotes the set of $K$ nearest neighbors of $x_i$. We refer to variants of our methods that utilize the fully-connected graph as $\textsf{ProbMix}$ and $\textsf{M-ProbMix}$ and the methods that utilize the nearest neighbors graph as \textsf{Loc$^\mathsf{K}$ProbMix} and $\textsf{Loc$^\mathsf{K}$M-ProbMix}$. By default, we utilize $K=5$ neighbors, unless otherwise specified.

\paragraph{Mixing and perturbation distributions:} In this work, we consider $\lambda\sim {\cal B}(\alpha, \alpha)$, where $\alpha\in (0, 1)$. With regards to the choice of the response perturbation distribution $s_\beta(\tilde{y}|y_i)$, since we are utilizing log-linear pooling, we need to guarantee that the distribution $s_\beta(\tilde{y}|y_i)$ is positive (i.e., has nonzero probability density over ${\cal Y}$). 

In regression tasks we use an isotropic Gaussian with variance $\beta>0$, i.e., 
$$s_\beta(\tilde{y}|y_i)={\cal N}(\tilde{y}|y_i, \beta\mathbf{I}_{d_y}).$$
In classification tasks, we bias the probabilities of each label (treated as a one-hot-encoded vector) by a positive constant $\beta>0$ such that the resulting distribution satisfies 
$$\mathbb{P}(\tilde{y}=k|y_i)\propto \mathbb{P}(y_i=k)+\beta.$$ 
We note that in classification settings we normalize for the perturbed response distribution to be a valid probability distribution. We conducted ablation studies on both a toy regression and classification dataset to explore different settings of $\alpha$ and $\beta$: more details can be found in Appendix \ref{app: toy_data}.

\paragraph{Mixing layer and embedding distribution:} For \textsf{M-ProbMix}, the choice of the embedding and its distributional form are important considerations. In this work, we choose the embedding layer to be the first layer after feature extraction. In the case of simple regression or classification tasks, this could be after the first of second layer of a fully connected network. For more complex architectures, such as LSTMs and transformers, we use the flattened output (potentially transformed to a lower dimensional space) as the embedding layer. In terms of the distributional form, we utilize a diagonal Gaussian parameterized as follows:
\begin{equation}
    q_{\theta_0}(z|x) = {\cal N}(z|h_{\theta_0}(z), \mathrm{diag}(\sigma_{\theta_0}^2)),
\end{equation}
where $\sigma^2_{\theta_0}:{\cal X}\rightarrow \mathbb{R}^{d_z}$ is a variance network that maps the input data to the diagonal covariance matrix of the embedding. Importantly, for computational efficiency, $\sigma^2_{\theta_0}$ should share parameters with $h_{\theta_0}(x)$, especially for larger architectures. Finally, we highlight that for this choice of $q_{\theta_0}(z|x)$, the reparameterization trick can be applied to reduce the variance of stochastic gradients during training. 

\section{Experiments}
In this section, we provide an empirical evaluation comparing our approaches \textsf{ProbMix}, \textsf{M-ProbMix}, \textsf{Loc$^\mathsf{K}$ProbMix}, and \textsf{Loc$^\mathsf{K}$M-ProbMix} with \textsf{ERM} and classical variants of mixup regularization: vanilla mixup (\textsf{Mix}), manifold mixup (\textsf{M-Mix}), and local mixup (\textsf{Loc$^\mathsf{K}$Mix}). We additionally compare to a combination of manifold mixup and local mixup (\textsf{Loc$^\mathsf{K}$M-Mix}). We note that local variants of both \textsf{ProbMix} and \textsf{Mix} assume a nearest neighbors sampling graph with $K=5$ neighbors. We also split the training dataset into 80\% train and 20\% validation, unless noted otherwise. Once training is complete, the model parameters that provide the smallest loss on the validation dataset is used for evaluation. 

\subsection{Toy Datasets}

\paragraph{Toy regression:} Consider the following data generating process for the toy regression dataset:
\begin{equation}
    y_i = x_i^3 + \epsilon_i,
\end{equation}
where $\epsilon_i\sim{\cal N}(0, 9)$. We generate $n=100$ training examples such that $x_i\sim{\cal U}(-4, 4)$. Test examples are generated by randomly sampling features $x_i\sim {\cal U}(4, 6)$, which are considered out-of-distribution with respect to training data. This experimental setup allows us to test the extrapolation capabilities of each method.  For each method, we use a two-layer multi-layer perceptron (MLP) with 128 and 64 hidden units per layer and train with full-batch gradient descent for $E=500$ epochs with a learning rate of $\eta=0.01$. For manifold-based approaches, we set the mixing layer as the first layer of the MLP (i.e., $d_z=128$).  Figure~\ref{fig: main_regression_examples} shows a comparison of the extrapolation performance across the different methods. More detailed box plots showing the performance of each method with respect to both mean-squared error (MSE) and negative log-likelihood (NLL) averaged over 10 runs can be found in Appendix~\ref{app: toy_regression_data}. 

\begin{figure*}[t]
    \centering
    \begin{subfigure}{0.24\textwidth}
        \centering
        \includegraphics[trim={0, 0, 0, 1.25cm}, clip, width=\textwidth]{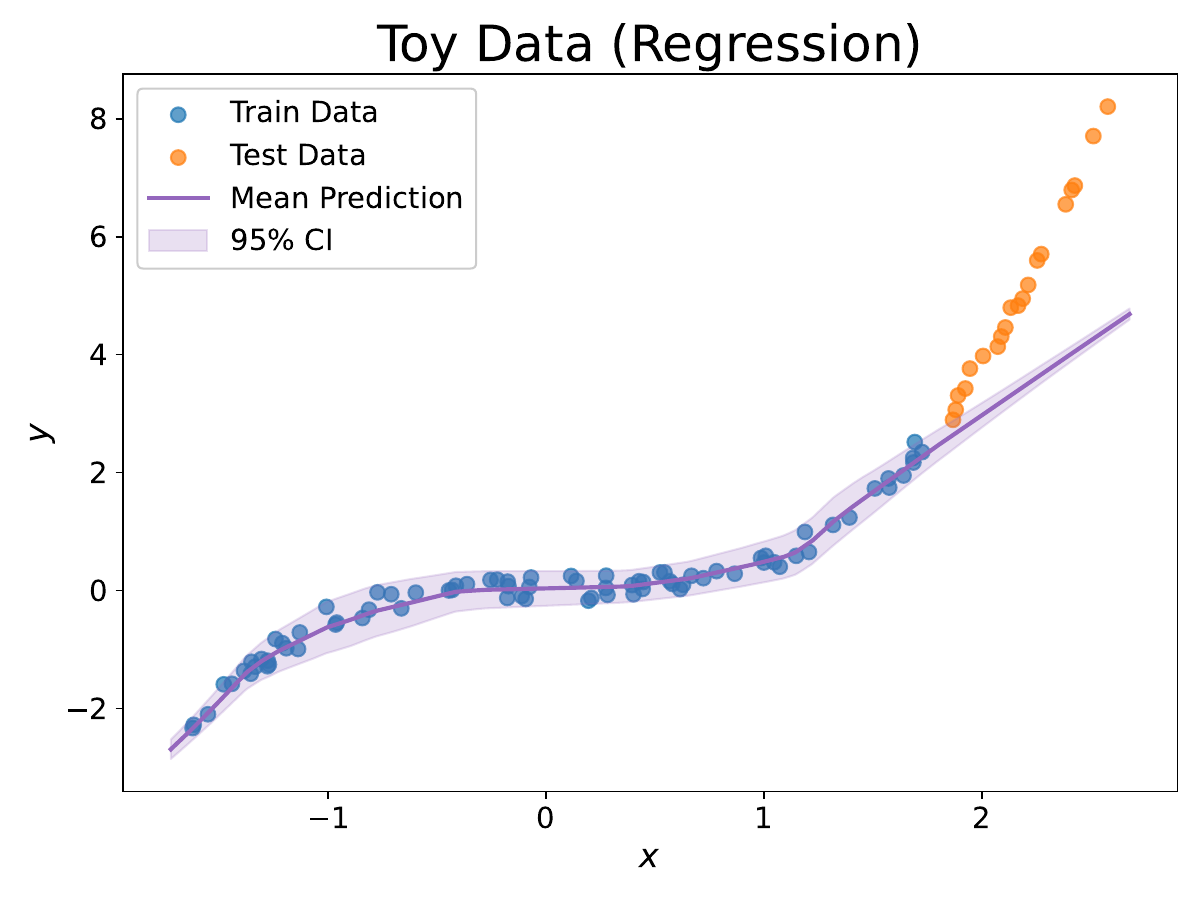}
        \caption{\textsf{Mix}.}
    \end{subfigure}%
    \hfill
    \begin{subfigure}{0.24\textwidth}
        \centering
        \includegraphics[trim={0, 0, 0, 1.25cm}, clip, width=\textwidth]{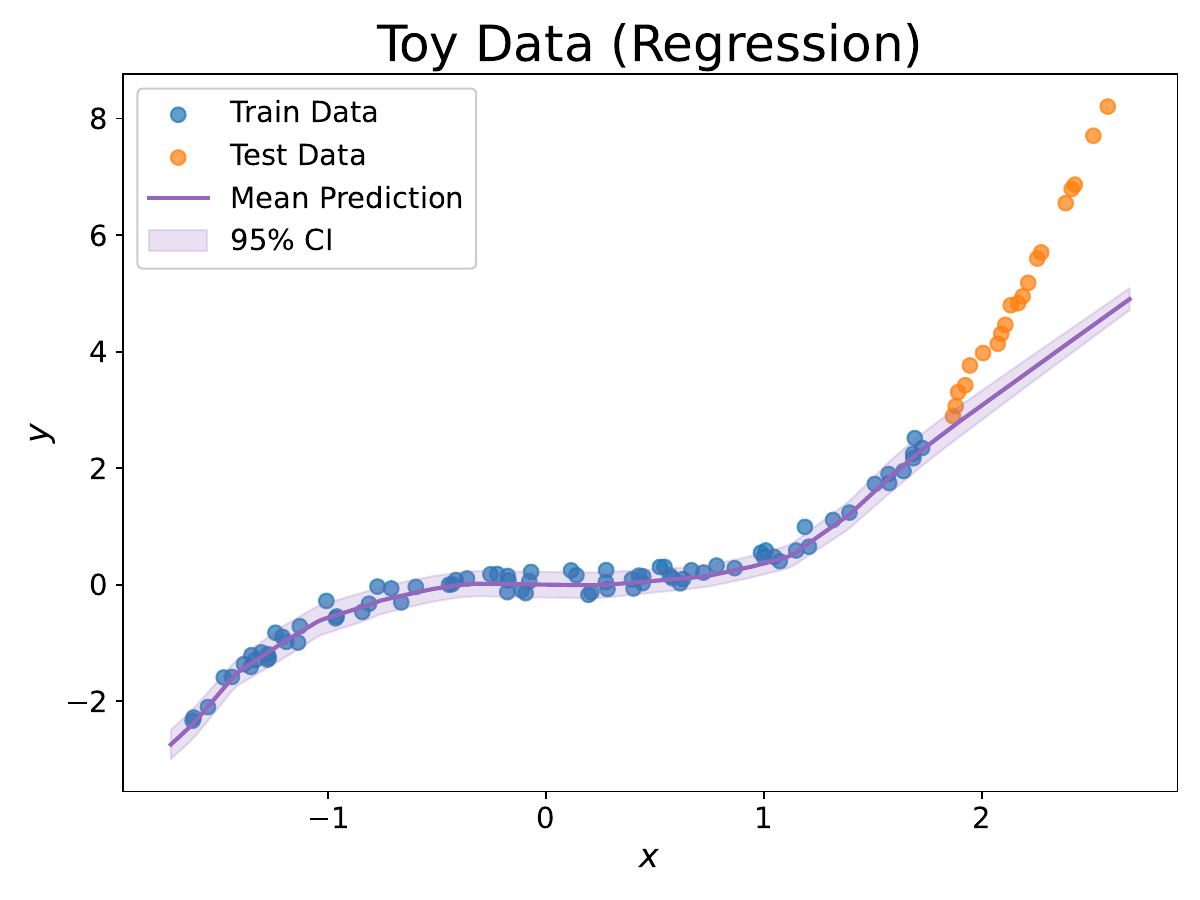}
        \caption{\textsf{Loc$^\mathsf{K}$Mix}.}
    \end{subfigure}%
    \begin{subfigure}{0.24\textwidth}
        \centering
        \includegraphics[trim={0, 0, 0, 1.25cm}, clip, width=\textwidth]{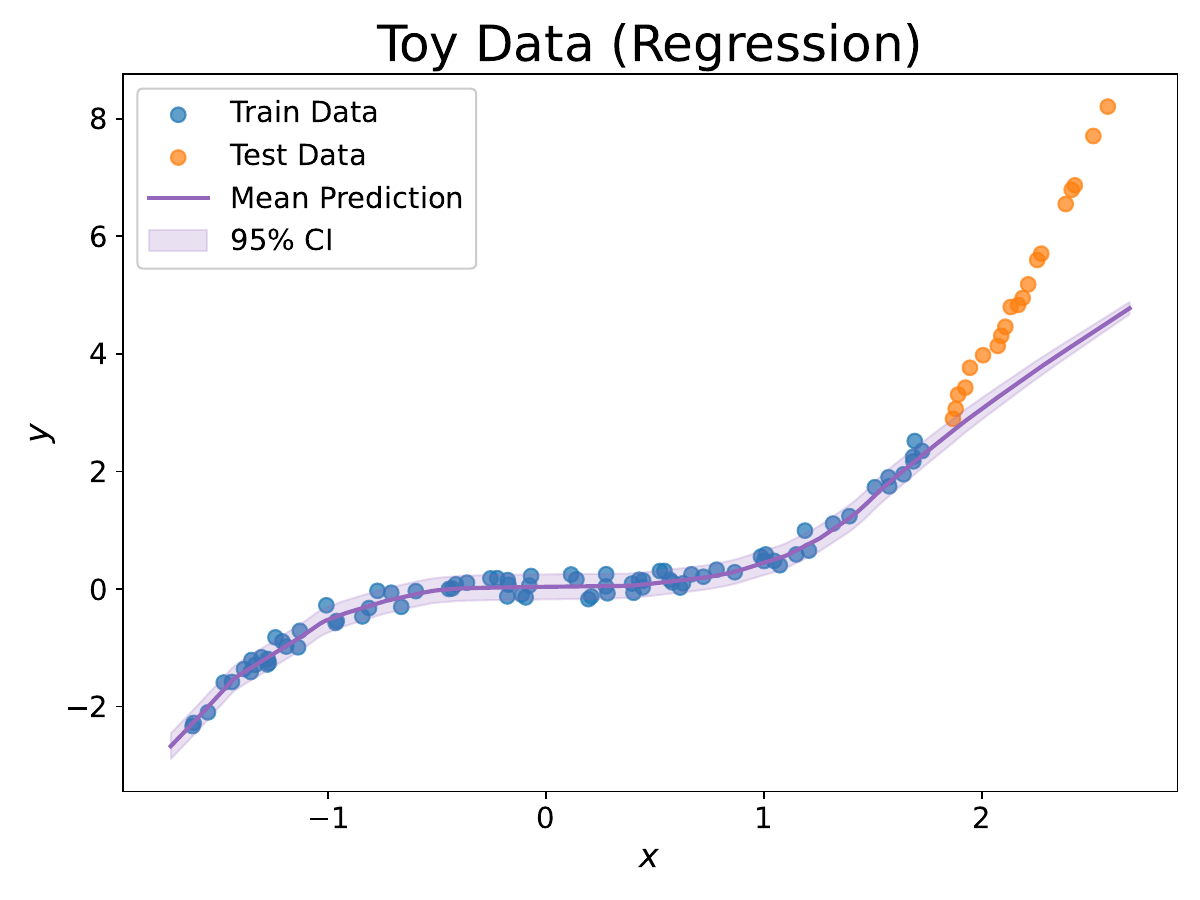}
        \caption{\textsf{ProbMix}.}
    \end{subfigure}%
    \hfill
    \begin{subfigure}{0.24\textwidth}
        \centering
        \includegraphics[trim={0, 0, 0, 1.25cm}, clip, width=\textwidth]{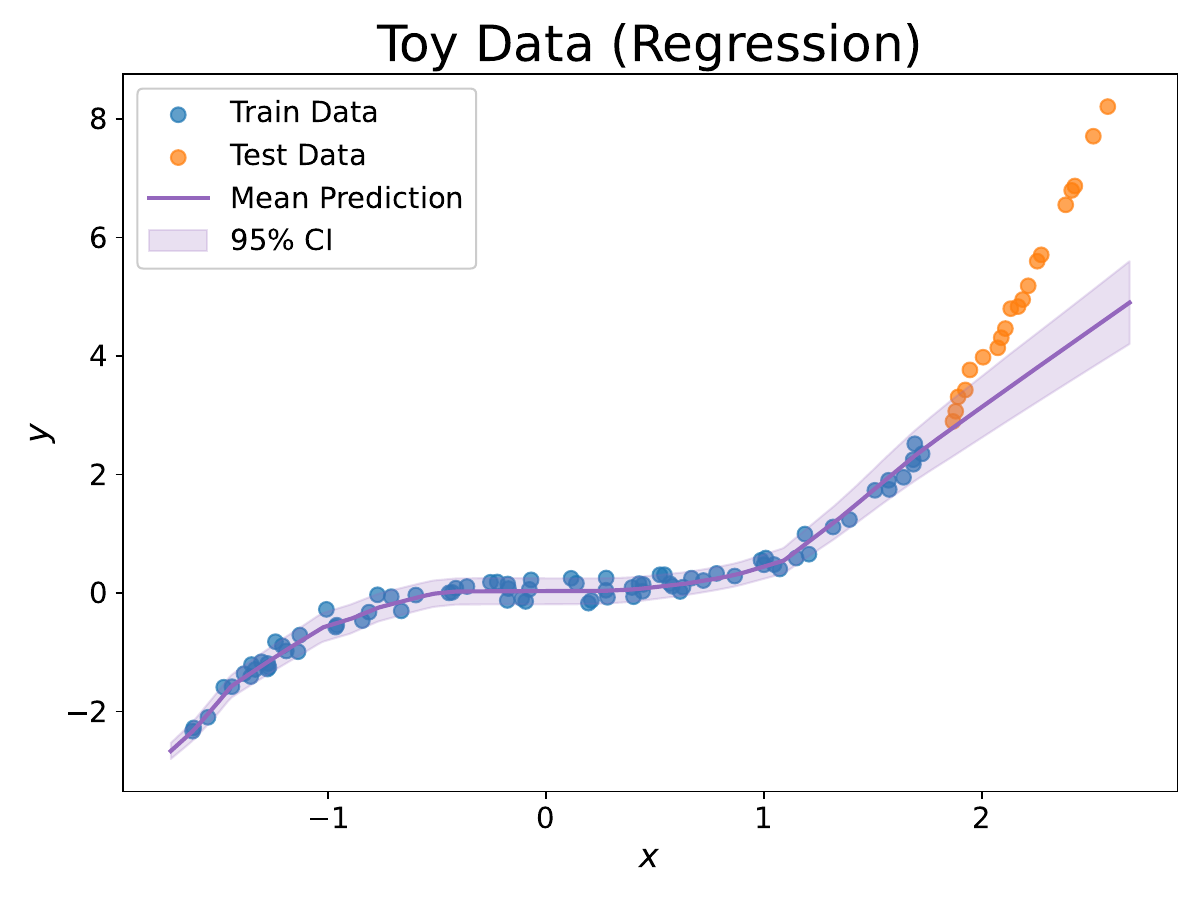}
        \caption{\textsf{Loc$^\mathsf{K}$ProbMix}.}
    \end{subfigure}%
    \caption{Visual example of different approaches for toy regression. \textsf{Mix} and \textsf{Loc$^\mathsf{K}$Mix} show poor performance on test samples, while \textsf{Loc$^\mathsf{K}$ProbMix}, although still mismatching in terms of the mean, outperforms mixup variants in terms of uncertainty calibration on out-of-sample data.}
    \label{fig: main_regression_examples}
\end{figure*}

In this example, while \textsf{Mix} achieves lower MSE than \textsf{ERM}, the NLL is much larger on average, implying worse calibration of the conditional density estimates. The other variants of mixup like $\textsf{M-Mix}$, \textsf{Loc$^\mathsf{K}$Mix}, and $\textsf{Loc$^\mathsf{K}$Mix}$ achiever similar MSE and NLL to \textsf{ERM} on average, but with much higher variance. This also implies that these other variants of mixup variation do not add any value to out-of-distribution extrapolation, but rather only introduce variability in training. We see that \textsf{ProbMix} and \textsf{Loc$^\mathsf{K}$ProbMix} achieve similar performance to the mixup variants, noting that in both cases of classical mixup regularization and probabilistic mixup, introducing locality in the sampling graph improves performance vastly in terms of NLL. This can be attributed to the fact that without locality, likelihoods for dissimilar (distant) features and their corresponding responses are fused, introducing a large bias. We observe the best performance in terms of both MSE and NLL with \textsf{M-ProbMix} and \textsf{Loc$^\mathsf{K}$M-ProbMix} methods.

Additionally, we provide results comparing the performance of linear and log-linear pooling function for \textsf{ProbMix} in Appendix \ref{app: exp_pooling_function}, as well as a comparison to latent variable modeling without mixup regularization in Appendix \ref{app: exp_latent_variable_modeling}. 

\paragraph{Toy classification:} We consider a three class toy classification problem from noisy ring-shaped distribution, where the features $x_i\in\mathbb{R}^2$ with label $y_i=k$ is generated according to:
\begin{equation}
    x_i = \begin{bmatrix} r_k\cos(\omega_i) \\ r_k\sin(\omega_i)\end{bmatrix} 
  +\epsilon_i,
\end{equation}
where $\omega_i\sim\mathcal{U}[0, 2\pi]$ and $\epsilon_i\sim {\cal N}(0, 0.09\mathbf{I}_2)$. We generate $n=100$ training examples from this generating process, where approximately an equal number of samples are generated per class with class radii defined by $r_1=0.5$, $r_2=1.5$, and $r_3=2.5$. We generate test samples from the same distribution as training data and utilize the same architecture and training hyperparameters as the toy regression dataset. Detailed results comparing the decision boundaries, test accuracy and NLL, averaged over 10 runs can be found in Appendix \ref{app: toy_classification_data}. 

In terms of test accuracy, we observe that all methods perform similarly, with \textsf{ERM} and \textsf{Loc$^\mathsf{K}$M-ProbMix} performing best on average. \textsf{M-Mix} and \textsf{ProbMix} (with $\beta=0$) achieve the worst performance; however, as evidence by the NLL, \textsf{ProbMix} achieves competitive performance in terms of uncertainty calibration. We highlight that \textsf{M-Mix} introduces large bias in the model, as the NLL is about $2\times$ as large as that of \textsf{ERM}. By varying the $\beta$ parameter, we notice that \textsf{ProbMix}'s performance improves in terms of both accuracy and NLL. Our probabilistic mixup framework shows best performance when $\beta=0.01$, as in this case, all variants of our approach are competitive or outperform baseline approaches.

\subsection{UCI Regression}

For the regression experiments, we use a two-layer MLP, with $128$ and $32$ units per layer, respectively. We use a set of regressions datasets from UCI, with a 90\%/10\% train/test split, and a further 20\% of the training data used as validation data to select the best performing model over $1000$ epochs. We utilize the 20 train/test splits made available by UCI, as in \cite{el-laham2023deep}. The models are trained by running the Adam optimizer for $E=1000$ epochs and a learning rate of $\eta=0.005$.

Table~\ref{tab: NLL_uci_full} summarizes the uncertainty quantification results in terms of negative log-likelihood (NLL), while Table~\ref{tab: RMSE_uci_full} in Appendix~\ref{app: experiments} provides the prediction performance in terms of root mean-square error (RMSE). Results indicate that \textsf{ProbMix} and its variants improve the uncertainty quantification capabilities of the model in almost all datasets, also providing a better prediction accuracy in terms of RMSE. As we observed the uncertainty quantification generally improves when considering \textsf{M-ProbMix} and various, we have also included a further refinement, which we have indicated as \textsf{M-ProbMix$^\star$}. In \textsf{M-ProbMix$^\star$},  a separate network is trained to predict the embedding variance, only during inference. For fairness of comparison with other methods, for \textsf{M-ProbMix$^\star$} we halved the size of the first embedding, so that the total size of the MLP would be comparable with the other experiments (i.e., the MLP is of size $64$ and $32$, with an additional variance embedding of size $64$). Result show that \textsf{M-ProbMix$^\star$} provides competitive results in terms of uncertainty quantification, but predictions performance might degrade in terms of RMSE, likely due to the reduced overall capacity of the prediction model.

\begin{table*}[t]
\centering
\resizebox{\textwidth}{!}{%
\begin{tabular}{@{}c|c|cccc|cccccc@{}}
\cmidrule{3-12}
\multicolumn{2}{c|}{} & \multicolumn{4}{c|}{Mixup Methods} & \multicolumn{6}{c}{ProbMixup Methods} \\ \midrule
Dataset & \textsf{ERM} & \textsf{Mix.} & \textsf{Loc$^\mathsf{K}$Mix.} & \textsf{M-Mix.} & Loc$^\mathsf{K}$M-Mix. & \textsf{ProbMix.} & \textsf{Loc$^\mathsf{K}$ProbMix.} & \textsf{M-ProbMix.} & \textsf{Loc$^\mathsf{K}$M-ProbMix.} & \textsf{M-ProbMix.}$^\star$ & \textsf{Loc$^\mathsf{K}$M-ProbMix.}$^\star$ \\ \midrule
bostonHousing & 3.37 $\pm$ 0.47 & 3.32 $\pm$ 1.17 & 3.46 $\pm$ 0.76 & 3.72 $\pm$ 2.32 & 3.31 $\pm$ 0.68 & 3.21 $\pm$ 1.07 & 3.26 $\pm$ 0.66 & 3.12 $\pm$ 0.56 & 3.09 $\pm$ 0.70 & 2.57 $\pm$ 0.41 & \textbf{2.52 $\pm$ 0.21} \\
energy & 1.24 $\pm$ 1.37 & 1.14 $\pm$ 1.00 & 1.08 $\pm$ 1.18 & 1.25 $\pm$ 1.52 & 1.00 $\pm$ 0.88 & 0.77 $\pm$ 0.32 & 0.82 $\pm$ 0.52 & \textbf{0.74 $\pm$ 0.42} & 0.93 $\pm$ 0.65 & 1.38 $\pm$ 0.18 & 1.13 $\pm$ 0.17 \\
wine-quality-red & 1.96 $\pm$ 2.08 & 1.84 $\pm$ 1.19 & 1.70 $\pm$ 1.43 & 1.87 $\pm$ 2.18 & 1.99 $\pm$ 2.01 & 1.50 $\pm$ 0.57 & 1.37 $\pm$ 0.32 & 1.26 $\pm$ 0.24 & 1.15 $\pm$ 0.18 & 1.16 $\pm$ 0.63 & \textbf{1.12 $\pm$ 0.22} \\
concrete & 4.50 $\pm$ 2.72 & 3.71 $\pm$ 0.64 & 3.75 $\pm$ 0.77 & 3.81 $\pm$ 0.95 & 3.69 $\pm$ 1.14 & 3.50 $\pm$ 0.58 & 3.67 $\pm$ 0.71 & 3.35 $\pm$ 0.48 & 3.37 $\pm$ 0.47 & 3.13 $\pm$ 0.22 & \textbf{3.07 $\pm$ 0.16} \\
power-plant & 2.86 $\pm$ 0.18 & 2.82 $\pm$ 0.11 & 2.84 $\pm$ 0.09 & 2.84 $\pm$ 0.12 & 2.85 $\pm$ 0.09 & \textbf{2.79 $\pm$ 0.06} & 2.84 $\pm$ 0.16 & 2.82 $\pm$ 0.06 & 2.84 $\pm$ 0.07 & 2.85 $\pm$ 0.05 & 2.86 $\pm$ 0.05 \\
yacht & 0.44 $\pm$ 0.94 & 1.70 $\pm$ 0.63 & 0.78 $\pm$ 0.38 & 0.49 $\pm$ 0.66 & 0.32 $\pm$ 0.29 & 1.39 $\pm$ 0.42 & 1.45 $\pm$ 0.88 & 0.27 $\pm$ 0.36 & \textbf{0.26 $\pm$ 0.25} & 0.92 $\pm$ 0.24 & 1.07 $\pm$ 0.28 \\
kin8nm$^\dagger$ & -1.14 $\pm$ 0.10 & -1.01 $\pm$ 0.13 & -1.14 $\pm$ 0.11 & -1.16 $\pm$ 0.09 & -1.07 $\pm$ 0.13 & -1.12 $\pm$ 0.10 & -1.15 $\pm$ 0.08 & \textbf{-1.20 $\pm$ 0.07} & -1.18 $\pm$ 0.05 & -1.17 $\pm$ 0.05 & -1.14 $\pm$ 0.05 \\
naval-propulsion-plant$^\dagger$ & -5.89 $\pm$ 0.88 & -6.15 $\pm$ 0.13 & -6.46 $\pm$ 0.45 & -6.28 $\pm$ 0.49 & \textbf{-6.47 $\pm$ 0.44} & -6.36 $\pm$ 0.63 & -6.12 $\pm$ 0.89 & -6.13 $\pm$ 0.30 & -5.44 $\pm$ 0.63 & -5.03 $\pm$ 0.04 & -4.92 $\pm$ 0.06 \\
\bottomrule
\end{tabular}
}
\caption{Average NLL for UCI regression datasets. In all datasets (except naval-propulsion-plant), probabilistic mixup variants obtain the best performance as compared to ERM and different mixup variants. ( $^\dagger$ indicates normalization to a single integer digit, while $^\star$ indicates the use of a separate variance networks, see text for details.)}
\label{tab: NLL_uci_full}
\end{table*}

\begin{table*}[t]
\centering
\resizebox{\textwidth}{!}{
\begin{tabular}{@{}clllllllll@{}}
\toprule
\multicolumn{2}{c}{LSTM}                                               & \multicolumn{4}{c}{Mixup Methods}                                                                                                                                              & \multicolumn{4}{c}{ProbMixup Methods}                                                                                                                                                         \\ \midrule
\multicolumn{1}{c|}{Dataset} & \multicolumn{1}{c|}{\textsf{ERM}}       & \multicolumn{1}{c}{\textsf{Mix}} & \multicolumn{1}{c}{\textsf{Loc$^\mathsf{K}$Mix}} & \multicolumn{1}{c}{\textsf{M-Mix}} & \multicolumn{1}{c|}{\textsf{Loc$^\mathsf{K}$M-Mix}} & \multicolumn{1}{c}{\textsf{ProbMix}} & \multicolumn{1}{c}{\textsf{Loc$^\mathsf{K}$ProbMix}} & \multicolumn{1}{c}{\textsf{M-ProbMix}} & \multicolumn{1}{c}{\textsf{Loc$^\mathsf{K}$M-ProbMix}} \\ \midrule
\multicolumn{1}{c|}{GME}     & \multicolumn{1}{l|}{$158.67 \pm 14.80$} & $106.54 \pm 5.36$                & $505.33 \pm 15.76$                               & $170.99 \pm 9.25$                  & \multicolumn{1}{l|}{$541.32 \pm 50.82$}             & $257.74 \pm 17.71$                   & $411.59 \pm 48.22$                                   & $\mathbf{67.32 \pm 3.99}$                     & $143.75 \pm 8.33$                                      \\
\multicolumn{1}{c|}{GOOG}    & \multicolumn{1}{l|}{$36.52 \pm 6.26$}   & $148.71 \pm 21.49$               & $110.14 \pm 3.66$                                & $102.01 \pm 35.98$                 & \multicolumn{1}{l|}{$80.34 \pm 15.73$}              & $141.25 \pm 40.68$                   & $52.11 \pm 4.27$                                     & $\mathbf{5.83 \pm 1.65}$                        & $11.08 \pm 2.47$                                       \\
\multicolumn{1}{c|}{NVDA}    & \multicolumn{1}{l|}{$3.17 \pm 0.73$}    & $5.43 \pm 1.10$                  & $49.79 \pm 8.65$                                 & $9.06 \pm 1.93$                    & \multicolumn{1}{l|}{$61.44 \pm 15.22$}              & $5.19 \pm 1.14$                      & $61.32 \pm 13.26$                                    & $1.67 \pm 0.14$                        & $\mathbf{1.11 \pm 0.08}$                                        \\
\multicolumn{1}{c|}{RCL}     & \multicolumn{1}{l|}{$14.22 \pm 0.96$}   & $\mathbf{0.74 \pm 0.08}$                  & $128.31 \pm 22.15$                               & $61.59 \pm 10.45$                  & \multicolumn{1}{l|}{$55.22 \pm 6.59$}               & $18.32 \pm 1.33$                     & $56.42 \pm 6.33$                                     & $2.22 \pm 0.30$                        & $6.22 \pm 0.74$                                        \\ \midrule
\multicolumn{2}{c|}{Transformer}                                       & \multicolumn{4}{c|}{Mixup Methods}                                                                                                                                             & \multicolumn{4}{c}{ProbMixup Methods}                                                                                                                                                         \\ \midrule
\multicolumn{1}{c|}{Dataset} & \multicolumn{1}{c|}{\textsf{ERM}}       & \multicolumn{1}{c}{\textsf{Mix}} & \multicolumn{1}{c}{\textsf{Loc$^\mathsf{K}$Mix}} & \multicolumn{1}{c}{\textsf{M-Mix}} & \multicolumn{1}{c|}{\textsf{Loc$^\mathsf{K}$M-Mix}} & \multicolumn{1}{c}{\textsf{ProbMix}} & \multicolumn{1}{c}{\textsf{Loc$^\mathsf{K}$ProbMix}} & \multicolumn{1}{c}{\textsf{M-ProbMix}} & \multicolumn{1}{c}{\textsf{Loc$^\mathsf{K}$M-ProbMix}} \\ \midrule
\multicolumn{1}{c|}{GME}     & \multicolumn{1}{l|}{$684.79 \pm 83.17$} & $507.55 \pm 28.04$               & $650.72 \pm 59.74$                               & $788.60 \pm 70.42$                 & \multicolumn{1}{l|}{$551.78 \pm 27.83$}             & $389.45 \pm 17.00$                   & $466.82 \pm 25.67$                                   & $\mathbf{374.15 \pm 41.05}$                     & $398.97 \pm 16.66$                                     \\
\multicolumn{1}{c|}{GOOG}    & \multicolumn{1}{l|}{$114.32 \pm 23.27$} & $147.85 \pm 10.57$               & $57.37 \pm 16.75$                                & $164.20 \pm 26.53$                 & \multicolumn{1}{l|}{$64.75 \pm 12.21$}              & $203.44 \pm 31.69$                   & $\mathbf{35.21 \pm 7.98}$                                     & $111.34 \pm 21.59$                     & $86.94 \pm 15.56$                                      \\
\multicolumn{1}{c|}{NVDA}    & \multicolumn{1}{l|}{$2.29 \pm 0.32$}    & $1.32 \pm 0.16$                  & $3.23 \pm 0.44$                                  & $2.38 \pm 0.13$                    & \multicolumn{1}{l|}{$1.90 \pm 0.11$}                & $7.77 \pm 0.67$                      & $1.78 \pm 0.03$                                      & $\mathbf{0.91 \pm 0.10}$                        & $1.81 \pm 0.08$                                        \\
\multicolumn{1}{c|}{RCL}     & \multicolumn{1}{l|}{$10.04 \pm 2.17$}   & $\mathbf{-0.05 \pm 0.02}$                 & $10.26 \pm 1.17$                                 & $8.46 \pm 1.39$                    & \multicolumn{1}{l|}{$4.85 \pm 0.78$}                & $7.46 \pm 0.55$                      & $4.82 \pm 0.65$                                      & $3.46 \pm 0.34$                        & $4.36 \pm 0.85$                                        \\ \bottomrule
\end{tabular}
}
    \caption{NLL on stock datasets for time series forecasting for LSTM model (above) and transformer model (below). Numerical values have been normalized in each dataset so that the best performing method has one or two integer digits. Probabilistic mixup variants outperforms all other methods except on the the RCL stock, for which \textsf{Mix} performs best.}
    \label{tab: full_nll_timeseries}
\end{table*}

\subsection{Financial Time Series Forecasting}
Finally, we demonstrate the performance of probabilistic mixup on a time-series forecasting task. Specifically, we used historical data obtained 
for the following stocks: Google (GOOG), Gamestop (GME), NVIDIA (NVDA), and Royal Carribean (RCL).  \footnote{Data available for download via the yfianace API \url{https://pypi.org/project/yfinance/}} . 
Let $S=[S_1, \ldots, S_T]\in\mathbb{R}^{4\times T}$ denote a stock time series, where each $S_t$ corresponds to the open, high, low, and close price for day $t$. We use the sliding window technique to extract price time-series $x_i=[S_i, \ldots, S_{i+W}]\in\mathbb{R}^{4\times W}$ and its corresponding forecast of the close price $y_i=S_{4, i+W+H}$, where $W=42$ denotes the window size and $H=21$ denotes the forecast horizon. We utilize data from the period: 01/01/2019 to 01/01/2021 and consider two different time series architectures for this forecasting task: an LSTM architecture with one hidden layer with $64$ hidden units; and a transformer architecture with 2 attention heads, 2 encoder/decoder hidden layers with 64 hidden units per layer. For both architectures, the extracted features are flattened and linearly compressed to an embedding of size $d_z=64$ (which corresponds to the embedding layer). We train the models using the Adam optimizer for $E=2000$ epochs with a learning rate of $\eta=0.0005$. Results in terms of NLL can be found for the LSTM and transformer models in Table \ref{tab: full_nll_timeseries}, respectively. Results in terms of test RMSE can be found in Table~\ref{tab: RMSE_stock_full}, while the average training NLL and average training RMSE can be found in Tables \ref{tab: NLL_stock_full_training} and \ref{tab: RMSE_stock_full_training} in Appendix~\ref{app: stocks}.  

Results show that in 3 out of the 4 stocks, probabilistic mixup variants outperform \textsf{ERM} and standard mixup variants in terms of NLL for both the LSTM and transformer architecture. We can see for the case of GME, both architectures overfit to the training data, reflected in the substantially worse performance. While \textsf{Mix} improves performance as compared to \textsf{ERM} for this dataset, we can see that probabilistic mixup approaches can achieve lower NLL, indicating better calibration of the conditional density estimates on out-of-distribution samples. We attribute this performance gain to the fact that probabilistic mixup approaches provide more conservative predictions of the variance of the conditional density (see sample plots in Appendix \ref{app: stocks}).

\section{Conclusions and future work}
In this work, we proposed a novel formulation of mixup regularization tailored for conditional density estimation tasks. Specifically, we introduced a mixup regularization scheme called \textsf{ProbMix}, which operates by probabilistically fusing the conditional densities of different input features within the model. Additionally, we proposed an extension called \textsf{M-ProbMix}, which involves fusing on a statistical manifold defined at an intermediate layer of the network. Our theoretical results demonstrate that many instances of classic mixup regularization can be viewed as specific cases within our proposed framework. Empirical results show that both \textsf{ProbMix} and \textsf{M-ProbMix} produce conditional density estimates that are significantly better calibrated for out-of-sample data. 

While the focus of this work was to devise a framework to extend mixup in the context of conditional density estimations, in future work we plan on extending \textsf{ProbMix} applications to more complex data modalities (such as images, text and speech), as well as including a comparison with existing approaches for quantifying uncertainty in deep learning models, and whether the inclusion of \textsf{ProbMix} could further improve performance (e.g, if every model in the deep ensemble approach in \citealt{lakshminarayanan2017simple} is trained instead with probabilistic mixup).
Potential future research directions include exploring the applicability of our methods to generative frameworks, such as variational autoencoders or denoising diffusion probabilistic models. Finally, we envision a theoretical analysis on the generalization performance of our proposed method could be pursued in a similar fashion as in \cite{NEURIPS2022_1626be0a}.

\begin{acknowledgements}
The authors would like to thank Nelson Vadori for providing feedback during the research ideation stage. \\ \\
\textbf{Disclaimer:} This paper was prepared for informational purposes by the Artificial Intelligence Research group of JPMorgan Chase \& Co. and its affiliates ("JP Morgan'') and is not a product of the Research Department of JP Morgan. JP Morgan makes no representation and warranty whatsoever and disclaims all liability, for the completeness, accuracy or reliability of the information contained herein. This document is not intended as investment research or investment advice, or a recommendation, offer or solicitation for the purchase or sale of any security, financial instrument, financial product or service, or to be used in any way for evaluating the merits of participating in any transaction, and shall not constitute a solicitation under any jurisdiction or to any person, if such solicitation under such jurisdiction or to such person would be unlawful.
\end{acknowledgements}


\appendix

\onecolumn

\title{Mixup Regularization: A Probabilistic Perspective\\(Appendix)}
\maketitle

\section{Review of Mixup Variants}
\label{app: review_mixup}

\subsection{Manifold Mixup.}
\label{app: manifold_mixup}
Manifold mixup is a variant of mixup that instead constructs augmented samples by mixing the features in some hidden layer of the predictor. For discussing manifold mixup, it is helpful to think of the neural network predictor as a composition of two functions $f_\theta = h_{\theta_1} \circ h_{\theta_0}$, where $h_{\theta_0}$ is referred to as the \emph{feature extractor} and $h_{\theta_1}$ is referred to as the \emph{predictor} and $\theta=\{\theta_0, \theta_1\}$. Mathematically, the loss function for manifold mixup is given by:
\begin{equation}
    \label{eq: manifold_mixup_risk}
    \tilde{R}_{\alpha, L{\rm mix}}^{{\cal M}}(\theta) = \frac{1}{n^2} \sum_{i=1}^n \sum_{j=1}^n \mathbb{E}_{\lambda}\left[\ell(h_{\theta_1}(\tilde{z}_{i, j, \lambda}), \tilde{y}_{i, j, \lambda})\right]
\end{equation}
where $\tilde{z}_{i, j, \lambda}$ is defined as a mixture of the features in after passing the input features through the feature extractor $h_{\theta_0}$:
\begin{equation}
    \tilde{z}_{i, j, \lambda} = \lambda h_{\theta_0}(x_i) + (1-\lambda) h_{\theta_0}(x_j),
\end{equation}

with the mixing parameter $\lambda$ following the same setting as in vanilla mixup. We note that manifold mixup introduces an additional hyperparameter, namely in which layer the features are mixed, i.e., how to construct the feature extractor $h_{\theta_0}$ and predictor $h_{\theta_1}$. Finally, manifold mixup selects pairs within the manifold by sampling data points uniformly at random, hence inducing a fully connected graph with all edges having the same weights.


\subsection{Local Mixup.} 
\label{app: local_mixup}
While manifold mixup may seem more principled than vanilla mixup, the problem of manifold intrusion can arise in both approaches, as the  manifold learned by the feature extractor may not have desirable properties. To combat this, several variants of mixup have been proposed that instead construct mixup augmentations locally based on a weighted graph. Consider a graph ${\cal G}=({\cal D}, {\cal E}, {\cal W})$, where ${\cal D}$ are the vertices of the graph, while ${\cal E}$ and ${\cal W}$ denote edges and weights of the graph, respectively. We use  $w_{i, j}\in {\cal W}$ to denote the weight of the edge between $(x_i, y_i)$ and $(x_j, y_j)$. For example, local mixup performs vanilla mixup augmentations locally based on some weighted graph defined over the observed dataset, thereby reducing the risk of mixing samples from different classes that lie on different manifolds. In essence local mixup considers a more general form of the vicinal distribution than vanilla mixup: 
\begin{equation}
    \label{eq: vicinal_distribution_local_mixup}
    \tilde{p}_{\alpha, {\cal G}}(x, y) =  \sum_{(i, j)\in{\cal E}} w_{i, j}\mathbb{E}_{\lambda}\left[\delta_{X, Y}(\tilde{x}_{i, j, \lambda}, \tilde{y}_{i, j, \lambda})\right],
\end{equation}
which leads to a weighted version of the vanilla mixup loss function:
\begin{equation}
    \label{eq: local_mixup_loss}
    \tilde{R}_{\alpha, {\cal G}}^{\rm local}(\theta) = \sum_{(i, j)\in {\cal E}} w_{i, j}\mathbb{E}_{\lambda}\left[\ell(f_\theta(\tilde{x}_{i, j, \lambda}), \tilde{y}_{i, j, \lambda})\right]
\end{equation}
This idea can also be trivially extended to manifold mixup.

\section{Log-Expected Likelihood Criterion}
\label{app: loss_discussion}
An alternative optimization criterion is to maximize the logarithm of the expected likelihood:
\begin{align}
    l^{\rm up}(\theta; \alpha, {\cal G}) &= \log \mathbb{E}\left[\tilde{p}_\theta(\tilde{y}|x_i, x_j, \lambda)\right] \\
    &= \log\left(\sum_{(i, j)\in{\cal E}} w_{i, j} \mathbb{E}_{\lambda}\left[\mathbb{E}_{\tilde{y}}[\tilde{p}_\theta(\tilde{y}|x_i, x_j, \lambda)|\lambda]\right]\right)
\end{align}
We note that since the logarithm is outside the expectation operator, this loss function does not have an interpretation from a VRM standpoint. Moreover, it is easy to see that, by Jensen's inequality, $l^{\rm up}(\theta; \alpha, {\cal G})$ upper bounds $l(\theta; \alpha, {\cal G})$ 
\begin{align}
    l^{\rm up}(\theta; \alpha, {\cal G}) &= \log \mathbb{E}\left[\tilde{p}_\theta(y|x_i, x_j, \lambda)\right] \\
    &\geq \mathbb{E}\left[\log \tilde{p}_\theta(y|x_i, x_j, \lambda)\right] = l(\theta; \alpha, {\cal G})
\end{align}
When $K=1$, we note that Monte Carlo estimates of $l^{\rm up}(\theta; \alpha, {\cal G})$ and $l(\theta; \alpha, {\cal G})$ are identical, meaning that from a practical perspective, gradient updates will be the same for both. Importantly, when $K>1$, the expected log-likelihood and logarithm of the expected likelihood will yield different stochastic gradient updates.

\section{Log-Linear Pooling for Exponential Family Members}
\label{thm: exponential_families}

\begin{theorem}[Log-Linear Pooling of Exponential Families]
Let $g_\lambda^x$ denote the log-linear pooling function. Suppose that $p_\theta(y|x)$ belongs to the exponential family of probability distributions, that is,
\begin{equation}
    p_{\theta}(y|x) = h(y)\exp\left(\phi_{\theta}(x)T(y) - A_\theta(x)\right),
\end{equation}
Then for two inputs $x_i$ and $x_j$, the log-linear fusion of $p_\theta(y|x_i)$ and $p_\theta(y|x_j)$ belongs to the same exponential family member. That is,
\begin{align}
    &g_\lambda^x(p_\theta(y|x_i), p_\theta(y|x_j)) =  h(y)\exp\left(\tilde\phi_{\theta, \lambda}(x_i, x_j)T(y) - \tilde{A}_{\theta, \lambda}(x_i, x_j)\right)
\end{align}
with 
\begin{align}
    &\tilde\phi_{\theta, \lambda}(x_i, x_j) = \lambda \phi_\theta(x_i)+ (1-\lambda)\phi_\theta(x_j) \\
    &\tilde{A}_{\theta, \lambda}(x_i, x_j) = \xi+\lambda A_\theta(x_i) + (1-\lambda)A_\theta(x_j), 
\end{align}
where $\xi$ denotes a normalization constant.
\end{theorem}

\begin{proof}
In the general case of the likelihood function of an example $(x, y)$ belonging to the exponential family,  we have:
\begin{align}
    p_{\theta}(y|x) = h(y)\exp\left(\phi_{\theta}(x)T(y) - A(\phi_{\theta}(x))\right),
\end{align}
withe corresponding log-likelihood:
\begin{align}
    \log p_{\theta}(y|x) = \log h(y) + \phi_{\theta}(x)T(y) - A(\phi_{\theta}(x))
\end{align}
A log-linear pooling of the likelihood function of $\theta$ for the examples $(x_i, y)$ and $(x_j, y)$ is given by:
\begin{align}
    &\log \tilde{p}_{\theta}(y|x_i, x_j) = \xi(x_i, x_j, \theta) + \lambda \log p_{\theta}(y|x_i) + (1-\lambda)\log p_{\theta}(y|x_j)\\
    &=\xi(x_i, x_j, \theta) + \lambda \left(\log h(y) + \phi_{\theta}(x_i)T(y) - A(\phi_{\theta}(x_i))\right) + (1-\lambda)\left(\log h(y) + \phi_{\theta}(x_j)T(y) - A(\phi_{\theta}(x_j))\right)\\
    &= \xi(x_i, x_j, \theta) + \log h(y) + \left(\lambda \phi_\theta(x_i)+ (1-\lambda)\phi_\theta(x_j)\right)T(y) - \left(\lambda A(\phi_\theta(x_i)) + (1-\lambda)A(\phi_\theta(x_j))\right) \\
    &= \log h(y) + \tilde\phi_\theta(x_i, x_j, \lambda) T(y) - \tilde{A}(\phi_\theta(x_i), \phi_\theta(x_j), \lambda)
\end{align}
where we define:
\begin{align}
    &\tilde\phi_\theta(x_i, x_j, \lambda) = \lambda \phi_\theta(x_i)+ (1-\lambda)\phi_\theta(x_j) \\
    &\tilde{A}(\phi_\theta(x_i), \phi_\theta(x_j), \lambda) = \lambda A(\phi_\theta(x_i)) + (1-\lambda)A(\phi_\theta(x_j)) -  \xi(x_i, x_j, \theta)
\end{align}
Thus, in general, the log-linear pooling function applied to two exponential family members from the same family results in a exponential family member. Furthermore, since the sufficient statistic $T(y)$ is preserved in the fusion result, the log-linear pooling function will always yield the same exponential family family member. Many important distributions belong to the exponential family of probability distributions. These include normal distributions, beta, gamma, categorical, and Poisson distributions, amongst others.
\end{proof}

\subsection{Log-Linear Pooling for Gaussian Regression}
In the classification setting, we have the following likelihood function for each example $(x, y)$:
\begin{align}
    p_{\theta}(y|x) = \frac{1}{\sqrt{2\pi\sigma^2_\theta(x)}}\exp\left(-\frac{(y-\mu_{\theta}(x))^2}{2\sigma_{\theta}^2(x)}\right)
\end{align}
with corresponding log-likelihood function:
\begin{align}
    \log p_{\theta}(y|x) = -\frac{1}{2}\log(2\pi)-\frac{1}{2}\log \sigma^2_\theta(x) -\frac{(y-\mu_{\theta}(x))^2}{2\sigma_{\theta}^2(x)}
\end{align}
A log-linear pooling of the likelihood function of $\theta$ for the examples $(x_i, y)$ and $(x_j, y)$ is given by:
\begin{align}
    &\log \tilde{p}_{\theta}(y|x_i, x_j) = \xi_0(x_i, x_j, \theta) + \lambda \log p_{\theta}(y|x_i) + (1-\lambda)\log p_{\theta}(y|x_j)\\
    &=\xi_1(x_i, x_j, \theta) -\frac{\lambda}{2}\log \sigma^2_\theta(x_i) -\frac{\lambda(y-\mu_{\theta}(x_i))^2}{2\sigma_{\theta}^2(x_i)} -\frac{1-\lambda}{2}\log \sigma^2_\theta(x_j) -\frac{(1-\lambda)(y-\mu_{\theta}(x_i))^2}{2\sigma_{\theta}^2(x_i)} \\
    &=\xi_1(x_i, x_j, \theta)-\frac{1}{2}\log\left(\left(\sigma^2_\theta(x_i)\right)^{\lambda}\left(\sigma^2_\theta(x_j)\right)^{1-\lambda}\right) -\frac{\lambda\sigma_{\theta}^2(x_j)(y-\mu_{\theta}(x_i))^2+(1-\lambda)\sigma_{\theta}^2(x_i)(y-\mu_{\theta}(x_j))^2}{2\sigma_{\theta}^2(x_i)\sigma_{\theta}^2(x_j)} \\
    &\implies \\
    & \tilde{p}_\theta(y|x_i, x_j) = {\cal N}\left(y\Bigg| \left(\frac{\lambda}{\sigma^2_\theta(x_i)}+ \frac{1-\lambda}{\sigma^2_\theta(x_j)}\right)^{-1}\left(\lambda\left(\frac{\mu_\theta(x_i)}{\sigma_\theta^2(x_i)}\right)+(1-\lambda)\left(\frac{\mu_\theta(x_j)}{\sigma^2_\theta(x_j)}\right)\right), \left(\frac{\lambda}{\sigma^2_\theta(x_i)}+ \frac{1-\lambda}{\sigma^2_\theta(x_j)}\right)^{-1}\right)
\end{align}

\subsection{Log-Linear Pooling for Classification}
In the classification setting, we have the following likelihood function for each example $(x, y)$:
\begin{align}
    p_{\theta}(y|x) = \prod_{k=1}^K \pi_{k, \theta}(x)^{\mathbf{1}(y=c_k)}
\end{align}
with corresponding log-likelihood:
\begin{align}
    \log p_{\theta}(y|x) = \sum_{k=1}^K \mathbf{1}(y=c_k) \log \pi_{k, \theta}(x)
\end{align}
A log-linear pooling of the likelihood function of $\theta$ for the examples $(x_i, y)$ and $(x_j, y)$ is given by:
\begin{align}
    \log \tilde{p}_{\theta}(y|x_i, x_j) &= \xi(x_i, x_j, \theta) + \lambda \log p_{\theta}(y|x_i) + (1-\lambda)\log p_{\theta}(y|x_j)\\
    &= \xi(x_i, x_j, \theta) + \lambda \left(\sum_{k=1}^K \mathbf{1}(y=c_k) \log \pi_{k, \theta}(x_i)\right) + (1-\lambda)\left(\sum_{k=1}^K \mathbf{1}(y=c_k) \log \pi_{k, \theta}(x_j)\right)\\
    &= \xi(x_i, x_j, \theta) +\left(\sum_{k=1}^K \mathbf{1}(y=c_k) \log \pi_{k, \theta}^{\lambda}(x_i)\right) + \left(\sum_{k=1}^K \mathbf{1}(y=c_k) \log \pi_{k, \theta}^{1-\lambda}(x_j)\right)\\
    &= \xi(x_i, x_j, \theta) + \sum_{k=1}^K \mathbf{1}(y=c_k) \left(\log \pi_{k, \theta}^{\lambda}(x_i) + \log \pi_{k, \theta}^{1-\lambda}(x_j)\right),
\end{align}
where $\xi(x_i, x_j, \theta)$ is a normalizing constant that depends on $x_i$, $x_j$, and $\theta$.This result implies that log-linear pooling of two categorical distributions is itself a categorical distribution. In practice, one would just need to combine the logits using a weighted arithmetic average to obtain the fused likelihood function.

\section{Illustrative Example: Heteroscedastic Gaussian Regression}\label{app: illustrative_example}
Here, we elaborate on the illustrative example from Section \ref{sss: illustrative_example}. In this illustrative example, we consider heteroscedastic Gaussian predictors of the form:
\begin{equation}
    \label{eq: heteroscesdatic_gaussian_predictor}
    y = \mu_\theta(x) + \sqrt{\sigma_\theta^2(x)} \epsilon, \quad \epsilon \sim{\cal N}(0, 1),
\end{equation}
where $\mu_\theta(x)$ denotes the \emph{mean function} and $\sigma_\theta^2(x)$ denotes the \emph{variance function}. For two inputs $(x_i, y_i)$ and $(x_j, y_j)$, vanilla mixup regularizes the model by learning to calibrate the random interpolations $\tilde{y}=\lambda y_i + (1-\lambda)y_j$ to the following:
\begin{align}
    \tilde{x} &= \lambda x_i + (1-\lambda)x_j \\
    p_{\theta}(\tilde{y}|\tilde{x})&={\cal N}(\tilde{y}|\mu_\theta(\tilde{x}), \sigma_\theta^2(\tilde{x}))
\end{align}
Note that in Mixup, the input to both $\mu_\theta$ (mean function) and $\sigma_\theta^2$ (variance function) is the interpolated feature $\tilde{x}=\lambda x_i + (1-\lambda)x_j$. Essentially, vanilla mixup biases both the mean function $\mu_\theta(x)$ and variance function $\sigma_\theta^2(x)$ to behave linearly between observed samples.
In contrast, \textsf{ProbMix} learns to calibrate the random interpolations $\tilde{y}$ to a (log-linear) fusion of the predicted conditional densities directly:
\begin{equation}
    p_{\theta}(\tilde{y}|x_i, x_j, \lambda) = {\cal N}(\tilde{y}|\mu_{\star}, \sigma_\star^2)
\end{equation}
where
\begin{align}
    \sigma_\star^2 &= \left(\frac{\lambda}{\sigma^2_\theta(x_i)}+ \frac{1-\lambda}{\sigma^2_\theta(x_j)}\right)^{-1} \\
    \mu_\star &= \sigma_\star^2 \left(\lambda\left(\frac{\mu_\theta(x_i)}{\sigma_\theta^2(x_i)}\right)+(1-\lambda)\left(\frac{\mu_\theta(x_j)}{\sigma^2_\theta(x_j)}\right)\right)
\end{align}
\textsf{ProbMix} combines statistical information based on two sources of information ($x_i$ and $x_j$) and depending on the degree of confidence placed in each source (determined by the value of $\lambda$ that is sampled), we obtain a different aggregated prediction (fused conditional density). While vanilla mixup biases the mean and variance functions to behave linearly between samples, \textsf{ProbMix} instead biases the model so that predicted conditional densities themselves behave log-linearly between samples.

In the context of the numerical example provided in Section \ref{sss: illustrative_example}, we consider the following data generating process:
\begin{equation}
    y = x^3 + (0.5x^2+1)\epsilon, \quad \epsilon\sim {\cal N}(0, 1).
\end{equation}
A perfectly fit model would learn mean and variance functions such that $\mu_\theta(x)=x^3$ and $\sigma^2_\theta(x)=(0.5x^2+1)^2$. Supposing that we observe two samples $(x_1, y_1)=(5, 130)$ and $(x_2, y_2)=(-5, -120)$, we would like to understand the regularization effect of both mixup and \textsf{ProbMix}. Consider a sample of the mixing coefficient $\lambda=0.8$.  For this value of $\lambda$, vanilla mixup's
mean and variance functions are determined by the interpolated feature $\tilde{x}=0.8\times 5 + 0.2\times-5=3$, and are given by
\begin{align}
    &\mu_\theta(\tilde{x})=\mu_\theta(3)=3^3=27 \\
    &\sigma^2_\theta(\tilde{x}) = \sigma^2_\theta(3) = (0.5\times 3^2+1)^2 = 5.5^2=30.25
\end{align}
and thus the predicted conditional density given by: 
\begin{equation}
    \label{eq: illustrative_example_mixup}
    p_{\sf Mix}(\tilde{y}|x_1, x_2,\lambda=0.8) = {\cal N}(\tilde{y}|27, 30.25).
\end{equation}
In \textsf{ProbMix}, the conditional densities are fused and result in another Gaussian conditional density, with variance and mean given by:
\begin{align}
    &\sigma^2_{\star}= \left(\frac{0.8}{\sigma^2_\theta(5)}+ \frac{0.2}{\sigma^2_\theta(-5)}\right)^{-1} =\left(\frac{0.8}{182.25}+ \frac{0.2}{182.25}\right)^{-1} = 182.25 \\
    &\mu_\star = \sigma_\star^2 \left(0.8\left(\frac{\mu_\theta(5)}{\sigma_\theta^2(5)}\right)+0.2\left(\frac{\mu_\theta(-5)}{\sigma^2_\theta(-5)}\right)\right) = 182.25 \left(0.8\left(\frac{125}{182.25}\right)+0.2\left(\frac{-125}{182.25}\right)\right) = 75
\end{align}
and thus,
\begin{equation}
    \label{eq: illustrative_example_probmixup}
    p_{\sf ProbMix}(\tilde{y}|x_1, x_2,\lambda=0.8) = {\cal N}(\tilde{y}|75, 182.25),
\end{equation}
Taking a perturbation of $\beta\rightarrow0$ for the latent observations, we can consider that both mixup and \textsf{ProbMix} evaluate the likelihood function of the model parameters for an observation of $\tilde{y}=\lambda y_1 + (1-\lambda)y_2=80$. The negative log-likelihood of the model parameters in this setting is given by:
\begin{align}
    &-\log p_{\sf Mix}(\tilde{y}=80|x_1, x_2,\lambda=0.8) \approx  49.05\\
    &-\log p_{\sf ProbMix}(\tilde{y}=80|x_1, x_2,\lambda=0.8) \approx 3.59
\end{align}
 Since the value of the interpolated observation lies in the right tail of $p_{\sf Mix}$, gradient updates based on mixup will substantially alter the model parameters, despite the fact that $\mu_\theta(x)$ and $\sigma_\theta^2(x)$ are already the ground truth functions. In contrast, \textsf{ProbMix}'s fused density is better calibrated to the interpolated observation $\tilde{y}=80$, leading to more stable and appropriate gradient updates. This demonstrates that mixup enforces a strong linear bias on the mean and variance functions, which becomes problematic when the true relationship is nonlinear or when the input features being mixed are relatively far apart. By mixing on a statistical manifold rather than the input space of the features, \textsf{ProbMix} avoids over-biasing the model. We provide plots of the data generating process, along with the conditional densities of both mixup and \textsf{ProbMix} in Figure \ref{fig: illustrative_example}.

\begin{figure}[t]
    \centering
    \includegraphics[width=\linewidth]{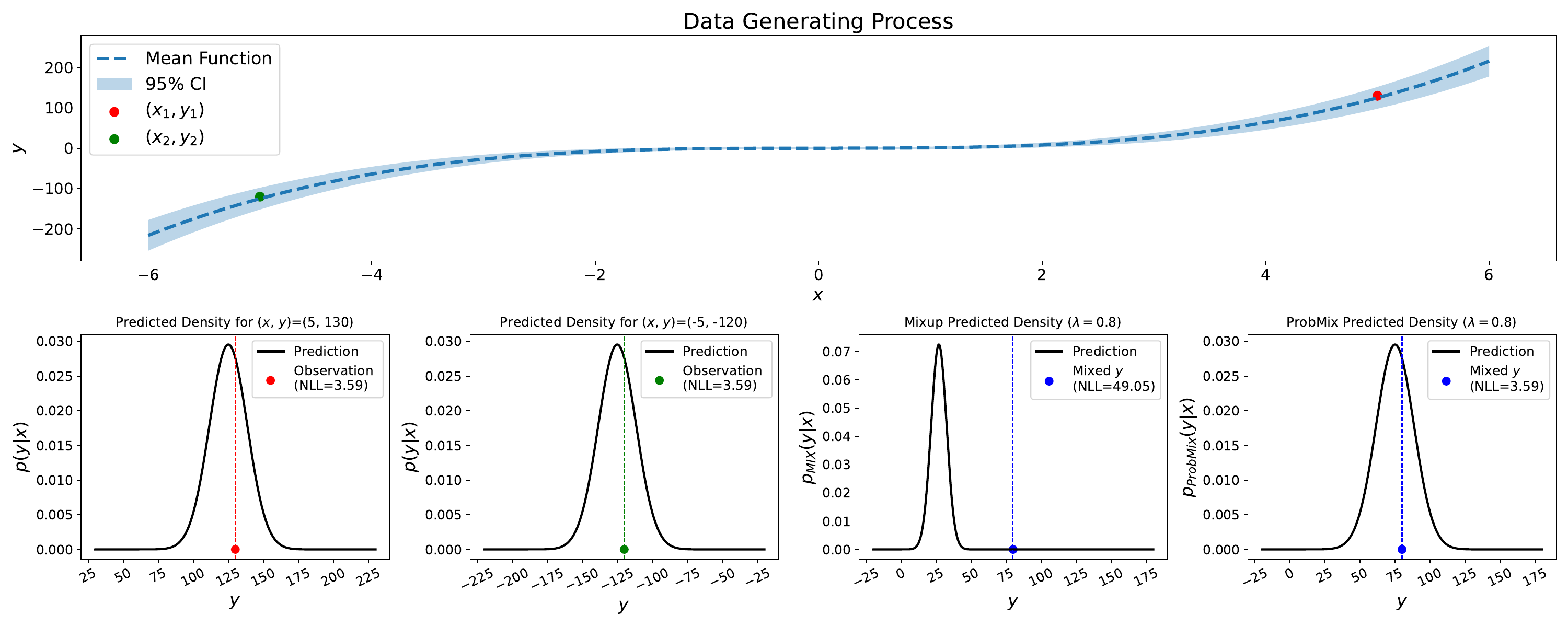}
    \caption{{\bf Top panel}: Data generating process considered in the illustrative example with mean function (blue dashed) and 95\% uncertainty band (blue shaded between 2.5\% and 97.5\% percentiles). {\bf Bottom panel far left}: Conditional density for the input $x=5$ (under $y=130$, NLL$\approx3.59$). {\bf Bottom panel center left}: Conditional density for the input $x=-5$ (under $y=120$, NLL$\approx3.59$).  {\bf Bottom panel center right}: Conditional density for vanilla mixup (under $\tilde{y}=80$, NLL$\approx49.05$).{\bf Bottom panel far right}: Conditional density for \textsf{ProbMix} (under $\tilde{y}=80$, NLL$\approx3.59$).}
    \label{fig: illustrative_example}
\end{figure}
\section{Theoretical Insights} \label{app: proof}
This section lists the proofs for the theorems in Section~\ref{sec: theory}. We first include the form of the expected risk for mixup, manifold mixup, \textsf{ProbMix} and \textsf{M-ProbMix} for completeness.

\paragraph{Mixup:} 
\begin{align*}
    \tilde{R}_{\alpha, {\cal G}}^{\rm mix}(\theta) = -\sum_{(i, j)\in {\cal E}} w_{i, j} \mathbb{E}_{\lambda}\left[\log p_\theta(y_{i, j, \lambda}|x_{i, j, \lambda})\right],
\end{align*}
where $y_{i, j, \lambda} = \lambda y_i + (1-\lambda) y_j$ and $x_{i, j, \lambda}=\lambda x_i + (1-\lambda)x_j$. 

\paragraph{Manifold Mixup:} 
\begin{align*}
    \tilde{R}_{\alpha, L_{\rm mix}}^{\rm {\cal M}}(\theta) = -\sum_{(i, j)\in {\cal E}} w_{i, j}\mathbb{E}_{\lambda}\left[\log p_{\theta_1}(y_{i, j, \lambda}|z_{i, j, \lambda})\right],
\end{align*}
where $y_{i, j, \lambda} = \lambda y_i + (1-\lambda) y_j$ and $z_{i, j, \lambda}=\lambda h_{\theta_0}(x_i) + (1-\lambda)h_{\theta_0}(x_j)$. 

\paragraph{\textsf{ProbMix}:}
\begin{align*}
    \tilde{R}_{\alpha, {\cal G}}^{\mathbb{P}}(\theta) = -\sum_{(i, j)\in {\cal E}} w_{i, j}  \mathbb{E}_{\lambda}[\mathbb{E}_{y_{i, j}}[\log p_\theta(y_{i, j}|x_i, x_j, \lambda)]]
\end{align*}
where $p_\theta(y_{i, j}|x_i, x_j, \lambda) = g_\lambda^x(p_\theta(y_{i, j}|x_i), p_\theta(y_{i, j}|x_j))$

\paragraph{\textsf{M-ProbMix}}
\begin{align*}
    \tilde{R}_{\alpha, {\cal G}}^{\mathbb{P}, {\cal M}}(\theta) = -\sum_{(i, j)\in {\cal E}} w_{i, j}  \mathbb{E}_{\lambda}[\mathbb{E}_{y_{i, j}}[\log p_{\theta}(y_{i, j}|x_{i}, x_{j}, \lambda)]]
\end{align*}
where $p_{\theta}(y_{i, j}|x_{i}, x_{j}, \lambda)$ is given by:
\begin{align*}
    p_{\theta}(y_{i, j}|x_{i}, x_{j}, \lambda) = \int p_{\theta_1}(y_{i, j}|z_{i, j})q_{\theta_0}(z_{i, j}|x_{i}, x_{j}, \lambda)dz_{i, j}
\end{align*}
and $q_{\theta_0}(z_{i, j}|x_{i}, x_{j}, \lambda) = g_{\lambda}^z(q_{\theta_0}(z_i|x_i), q_{\theta_0}(z_j|x_j))$ for some fusion function $g_\lambda^z$.

\begin{proof}[Proof of Theorem~\ref{theo: mixup-logits-class}]
    We show the likelihoods of the two approaches are proportional to the same quantity, hence equivalent. Without loss of generality, consider the likelihood for the $k^{th}$ class. For \textsf{ProbMix}, the likelihood is proportional to:
    \begin{align*}
        \log p_\theta^k(y|x_{i, j, \lambda}) &= \log g_\lambda^x(p^k_\theta(y|x_i),  p^k_\theta(y|x_j)) \\
        &= \lambda \log(p^k_\theta(y|x_i)) + (1-\lambda) \log(p^k_\theta(y|x_j) \\
        &= \lambda \log \left[\frac{e^{-f_\theta^k(x_i)}}{\sum_l e^{-f_\theta^l(x_i)}} \right] + (1-\lambda) \log \left[\frac{e^{-f_\theta^k(x_j)}}{\sum_l e^{-f_\theta^l(x_j)}}\right] \\
        &\propto \lambda \log \left( e^{-f_\theta^k(x_i)} \right) + (1-\lambda)  \log \left( e^{-f_\theta^k(x_j)}\right)
    \end{align*}
    Now, considering mixup on the logits, we obtain:

    \begin{align*}
        \log p_\theta(y|x_{i, j, \lambda}) &= \log p_\theta(y|\lambda f_\theta(x_i) + (1-\lambda)f_\theta(x_j)) \\
        &= \log \left( \frac{e^{-\lambda f_\theta^k (x_i) - (1-\lambda) f_\theta^k(x_j)}}{\sum_l e^{-\lambda f_\theta^l (x_i) - (1-\lambda) f_\theta^l(x_j)}} \right) \\
        & \propto \log \left[\left(e^{- f_\theta^k (x_i)}  \right)^\lambda \right] + \log \left[\left(e^{- f_\theta^k (x_j)}  \right)^{(1-\lambda)} \right] \\
        &= \lambda \log \left( e^{-f_\theta^k(x_i)} \right) + (1-\lambda)  \log \left( e^{-f_\theta^k(x_j)}\right)
    \end{align*}
    
\end{proof}

\begin{proof}[Proof of Theorem~\ref{theo: mixup-output-regr}]
    The key is to show the likelihood of both settings is equivalent. Considering \textsf{ProbMix} under log-linear fusion of homoscedastic Gaussian likelihoods, we have:

    \begin{align*}
        \log p_\theta(y|x_{i, j, \lambda}) &= \log g_\lambda^x(p_\theta(y|x_i),  p_\theta(y|x_j)) \\
        &= \log\left({\cal N}(y|f_\theta(x_i), \sigma^2\mathbf{I}_{d_y})^\lambda \cdot {\cal N}(y|f_\theta(x_j), \sigma^2\mathbf{I}_{d_y})^{(1-\lambda)}\right)\\
        & \propto -\lambda \frac{(y - f_\theta(x_i))^T (y - f_\theta(x_i))}{2\sigma^2} - \quad (1-\lambda) \frac{(y - f_\theta(x_j))^T (y - f_\theta(x_j))}{2\sigma^2} \\
        &\propto y^Ty - 2y^T \left(\lambda f_\theta(x_i) + (1-\lambda) f_\theta(x_j)\right)
    \end{align*}

    Now, considering the likelihood for traditional mixup when interpolating the output layer:
    \begin{align*}
        \log p_\theta(y|x_{i, j, \lambda}) &= \log p_\theta(y|\lambda f_\theta(x_i) + (1-\lambda)f_\theta(x_j)) \\
        &= C_\sigma^{d_y} - \Big[ \frac{(y - \lambda f_\theta(x_i) - (1-\lambda) f_\theta(x_j))^T \quad (y - \lambda f_\theta(x_i) - (1-\lambda) f_\theta(x_j))}{2\sigma^2}  \Big] \\
        &\propto y^Ty - 2y^T \left(\lambda f_\theta(x_i) + (1-\lambda) f_\theta(x_j)\right)
    \end{align*}

\end{proof}

\begin{proof}[Proof of Theorem~\ref{theo: mixup-log-reg}]
Consider $f_\theta=h_1\circ h_{\theta}$, where $h_1(z)=\left[\frac{e^{-z_1}}{\sum_{j=1}^{d_y} e^{-z_j}}, \ldots, \frac{e^{-z_{d_y}}}{\sum_{j=1}^{d_y} e^{-z_j}}\right]^\intercal$ is the softmax function and $h_{\theta}(x)=Ax+b$ is a linear function with $A=[a_1, \ldots, a_y]^\intercal\in\mathbb{R}^{d_y\times d_x}$ and $b\in\mathbb{R}^{d_y}$. For two inputs $x_i$ and $x_j$, the probability of the $k$th class, denoted by $p_{i, k}$ and $p_{j, k}$, respectively, can readily be determined as:
\begin{align}
    &p_{i, k} = \frac{e^{-(a_k^\intercal x_i + b_k)}}{\sum_{k'=1}^{d_y}e^{-(a_{k'}^\intercal x_i + b_{k'})}} \\
    &p_{j, k} = \frac{e^{-(a_k^\intercal x_j + b_k)}}{\sum_{k'=1}^{d_y}e^{-(a_{k'}^\intercal x_j + b_{k'})}}
\end{align}
For the mixup approach, consider the mixed input $x_{i, j, \lambda}=\lambda x_i + (1-\lambda)x_j$. The output of $f_\theta$ is given by:
\begin{align}
    f_\theta(x_{i, j, \lambda}) &= h_1(h_{\theta}(x_{i, j, \lambda})) \\
    &= h_1\left(A x_{i, j, \lambda} + b\right) \\
    &= h_1(\underbrace{\lambda A x_i+ (1-\lambda)A{x_j} + b}_{z_{i, j, \lambda}})
\end{align}
For each dimension $k=1, \ldots, d_y$ of $z_{i, j, \lambda}$, denoted by $z_{i, j, \lambda, k}=\lambda a_k^\intercal x_i + (1-\lambda) a_k^\intercal x_j + b_k$ we have that the output of $h_1$, which is the probability of class $k$, is given by:
\begin{align}
    p_k^{\rm mix} &= h_1(z_{i, j, \lambda, k}) \\
    &\propto e^{-z_{i, j, \lambda, k}} \\
    &= e^{-(\lambda a_k^\intercal x_i + (1-\lambda) a_k^\intercal x_j + b_k)} \\
    &= e^{-\lambda a_k^\intercal x_i}e^{-(1-\lambda) a_k^\intercal x_j} e^{-b_k} \\
    &= e^{-\lambda (a_k^\intercal x_i + b_k)}e^{-(1-\lambda)(a_k^\intercal x_j +b_k)} \\
    &= \left[e^{-(a_k^\intercal x_i + b_k)}\right]^\lambda\left[e^{-(a_k^\intercal x_j + b_k)}\right]^{1-\lambda} \\
    &\propto \left[h_{1}(z_{i, k})\right]^\lambda\left[h_{1}(z_{j, k})\right]^{1-\lambda} \\
    & = p_{i, k}^\lambda p_{j, k}^{1-\lambda}
\end{align}
Clearly $p_k^{\rm mix}\propto p_{i, k}^\lambda p_{j, k}^{1-\lambda}$. Therefore, in this model setting, mixup and \textsf{ProbMix} (under log-linear pooling) are equivalent.
\end{proof}

\begin{proof}[Proof of Theorem~\ref{theo: mixup-lin-reg}]
Consider that $f_\theta(x)=Ax+b$, where $A\in\mathbb{R}^{d_y\times d_x}$ and $b\in\mathbb{R}^{d_y}$. For two inputs $x_i$ and $x_j$, the log-likelihood function of $\theta$ based on an observed label $y$ is given by:
\begin{align}
    &\log p_\theta(y|x_i) = \log {\cal N}(y|f_\theta(x_i), \sigma^2\mathbf{I}_{d_y}) \\
    &\log p_\theta(y|x_j) = \log {\cal N}(y|f_\theta(x_j), \sigma^2\mathbf{I}_{d_y})
\end{align}
where so we have that:
\begin{align}
    \log p_\theta(y|x) = C_{\sigma}^{d_y}-\frac{\left(y-f_\theta(x)\right)^\intercal\left(y-f_\theta(x)\right)}{2\sigma^2},
\end{align}
where $C_{\sigma}^{d_y}=-\frac{d_y}{2}\log 2\pi\sigma^2$. For mixup, consider mixed input $x_{i, j, \lambda}=\lambda x_i + (1-\lambda)x_j$. Then, the log-likelihood function of an observation $y$ is given by:
\begin{align}
    \log p_\theta(y|x_{i, j, \lambda}) &= \log p_\theta(y|\lambda x_i + (1-\lambda)x_j) \\
    &=C_\sigma^{d_y} -\frac{\left(y-f_\theta(x_{i, j, \lambda})\right)^\intercal\left(y-f_\theta(x_{i, j, \lambda})\right)}{2\sigma^2}
\end{align}
We can expand the quadratic as: 
\begin{align*}
\left(y-f_\theta(x_{i, j, \lambda})\right)^\intercal\left(y-f_\theta(x_{i, j, \lambda})\right) =&y^\intercal y - 2\underbrace{y^\intercal f(x_{i, j, \lambda})}_{y^\intercal Ax_{i, j, \lambda} + y^\intercal b} + \underbrace{f(x_{i, j, \lambda})^\intercal f(x_{i, j, \lambda})}_{x_{i, j, \lambda}^\intercal A^\intercal Ax_{i, j, \lambda} + 2 b^\intercal Ax_{i, j, \lambda} + b^\intercal b} \\
&=y^\intercal y - 2y^\intercal A(\lambda x_i + (1-\lambda)x_j)  +  \cdots \\
&\vdots \\
&= \lambda \left(y-f_\theta(x_{i})\right)^\intercal\left(y-f_\theta(x_{i})\right) + (1-\lambda) \left(y-f_\theta(x_{j})\right)^\intercal\left(y-f_\theta(x_{j})\right)
\end{align*}
Therefore, one can readily show that:
\begin{equation}
    \log p_\theta(y|x_{i,j, \lambda}) = Z +\lambda\log p_\theta(y|x_i)+(1-\lambda)\log p_\theta(y|x_j)
\end{equation}
or equivalently that:
\begin{equation}
     p_\theta(y|x_{i,j, \lambda}) \propto \left[p_\theta(y|x_i)\right]^\lambda\left[p_\theta(y|x_j)\right]^{1-\lambda}
\end{equation}
Therefore, in this model setting, mixup and \textsf{ProbMix} (under log-linear pooling) are equivalent.
\end{proof}

\begin{proof}[Proof of Theorem~\ref{theo: manifold-prob}]

As noted in Section~\ref{sec:linear-log-linear-fusion} and Appendix~\ref{thm: exponential_families}, log-linear fusion of homoscedastic Gaussian embeddings results in a Gaussian distribution with the same variance $\sigma$ and mean $\mu_\theta$ equal to:

\begin{equation*}
    \mu_\theta(x_i, x_j) = \lambda h_\theta(x_i) + (1-\lambda) h_\theta(x_j).
\end{equation*}

If the mean is propagated during training and inference, this implies that the log-likelihood for \textsf{M-ProbMix} is:

\begin{align*}
    \log p(y|x_{i,j,k}) = \log p\left(y|d_\phi(\lambda h_\theta(x_i) + (1-\lambda) h_\theta(x_j)) \right),
\end{align*}

which is equal to the manifold mixup log-likelihood.
    
\end{proof}

\section{Experimental Results}
\label{app: experiments}

\subsection{Toy Datasets}
\label{app: toy_data}
Here, we include additional results related to the toy regression and toy classification datasets.

\subsubsection{Toy Regression Ablations}
\label{app: toy_regression_data}
Figure \ref{fig: all_regression_examples} shows the a plot of the conditional density estimator obtain via each of the baseline methods (\textsf{Mix}, \textsf{M-Mix}, and \textsf{Loc$^\mathsf{K}$Mix}, and \textsf{Loc$^\mathsf{K}$M-Mix}) as compared to the proposed appraoches (\textsf{ProbMix}, \textsf{M-ProbMix}, and \textsf{Loc$^\mathsf{K}$ProbMix}, and \textsf{Loc$^\mathsf{K}$M-ProbMix}). We can see that methods like \textsf{Mix} and \textsf{M-Mix} actually attenuate uncertainty in the out-of-sample region, which can be problematic for risk-sensitive applications. \textsf{ProbMix} also performs poorly on out-of-sample data; we attribute this to the fact that fusing the log-likelihoods during training imposes a large bias for non-neighboring samples in a regression task. This flaw is resolved in all other variants of probabilistic mixup as the density plots in Fig. \ref{fig: all_regression_examples} show that \textsf{M-ProbMix}, \textsf{Loc$^\mathsf{K}$ProbMix}, and \textsf{Loc$^\mathsf{K}$M-ProbMix} perform best in terms of capturing uncertainty on out-of-sample data. 

As an ablation study, we tested a grid of hyperparameter values defined by $\alpha\in\{0.01, 0.05, 0.10, 0.5, 1.0\}$ and $\beta\in\{0, 0.01\}$. We show a series of box plots comparing the MSE and the NLL across different methods. Results show that in the case of the regression experiment, manifold probabilistic mixup approaches (\textsf{M-ProbMix} and \textsf{Loc$^\mathsf{K}$M-ProbMix}) achieve best performance in terms of both MSE and NLL. Please refer to Figures \ref{fig: toy_regression_alpha_0p01_beta_0}-\ref{fig: toy_regression_alpha_1p0_beta_0p01} for more details on the results. 

\begin{figure*}[t]
    \centering
    \begin{subfigure}{0.24\textwidth}
        \centering
        \includegraphics[trim={0, 0, 0, 1.25cm}, clip, width=\textwidth]{figures/visual_examples/regression_examples/mix.pdf}
        \caption{\textsf{Mix}.}
        \label{fig: app_erm_toy_regression_mix}
    \end{subfigure}%
    \hfill
    \begin{subfigure}{0.24\textwidth}
        \centering
        \includegraphics[trim={0, 0, 0, 1.25cm}, clip, width=\textwidth]{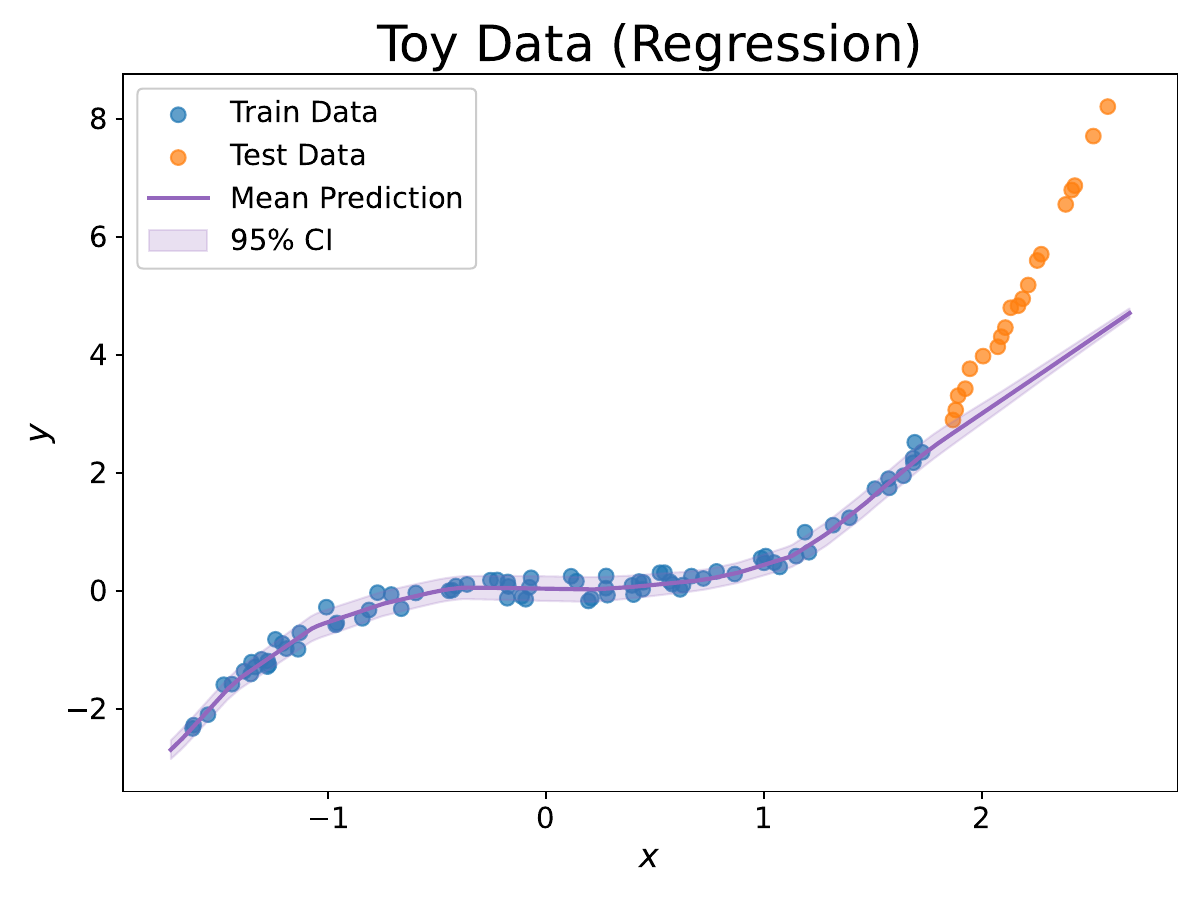}
        \caption{\textsf{M-Mix}.}
        \label{fig: app_erm_toy_regression_mmix}
    \end{subfigure}%
    \hfill
    \begin{subfigure}{0.24\textwidth}
        \centering
        \includegraphics[trim={0, 0, 0, 1.25cm}, clip, width=\textwidth]{figures/visual_examples/regression_examples/mix_local.pdf}
        \caption{\textsf{Loc$^\mathsf{K}$Mix}.}
        \label{fig: app_erm_toy_regression_lockmix}
    \end{subfigure}%
    \hfill
    \begin{subfigure}{0.24\textwidth}
        \centering
        \includegraphics[trim={0, 0, 0, 1.25cm}, clip, width=\textwidth]{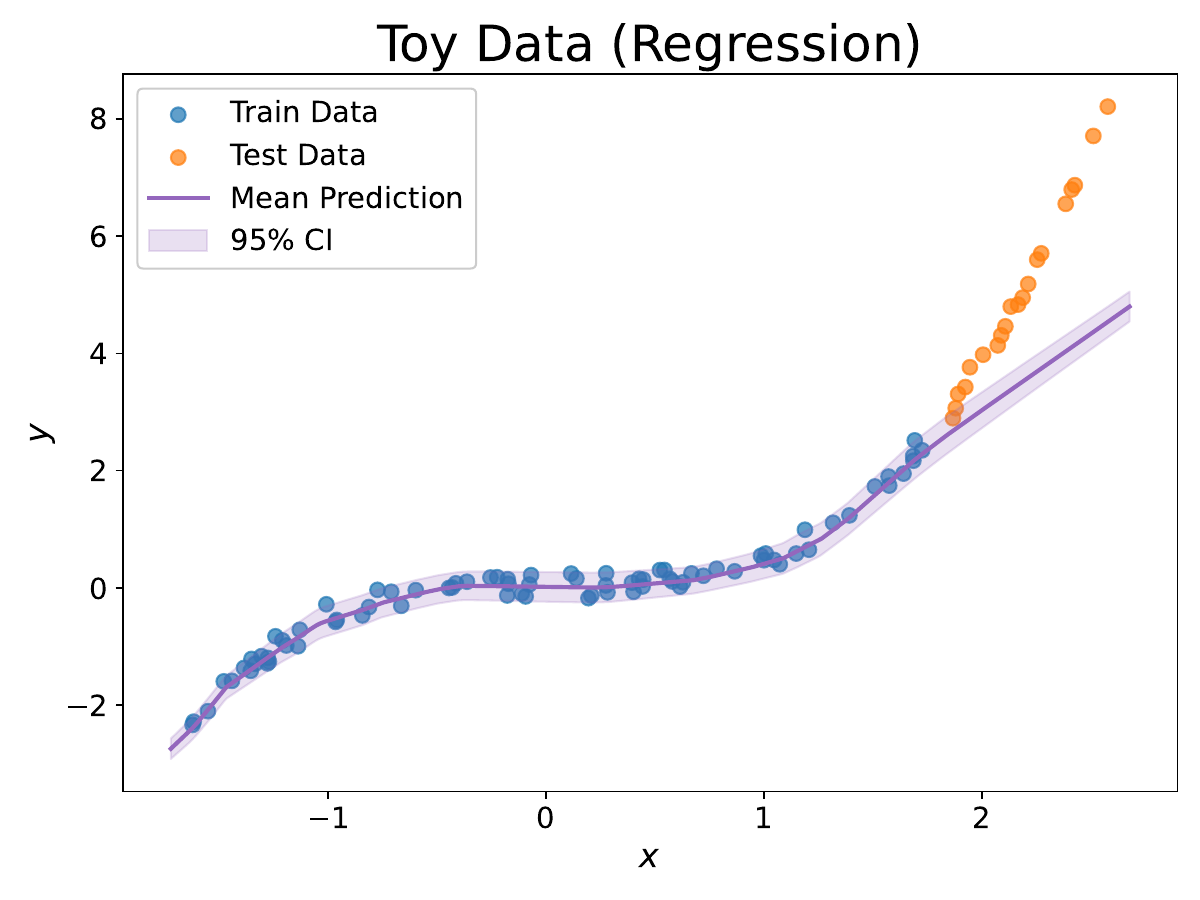}
        \caption{\textsf{Loc$^\mathsf{K}$M-Mix}.}
        \label{fig: app_erm_toy_regression_lockmmix}
    \end{subfigure}%
    \vspace{0.1cm}
    \begin{subfigure}{0.24\textwidth}
        \centering
        \includegraphics[trim={0, 0, 0, 1.25cm}, clip, width=\textwidth]{figures/visual_examples/regression_examples/probmix.pdf}
        \caption{\textsf{ProbMix}.}
        \label{fig: app_erm_toy_regression_probmix}
    \end{subfigure}%
    \hfill
    \begin{subfigure}{0.24\textwidth}
        \centering
        \includegraphics[trim={0, 0, 0, 1.25cm}, clip, width=\textwidth]{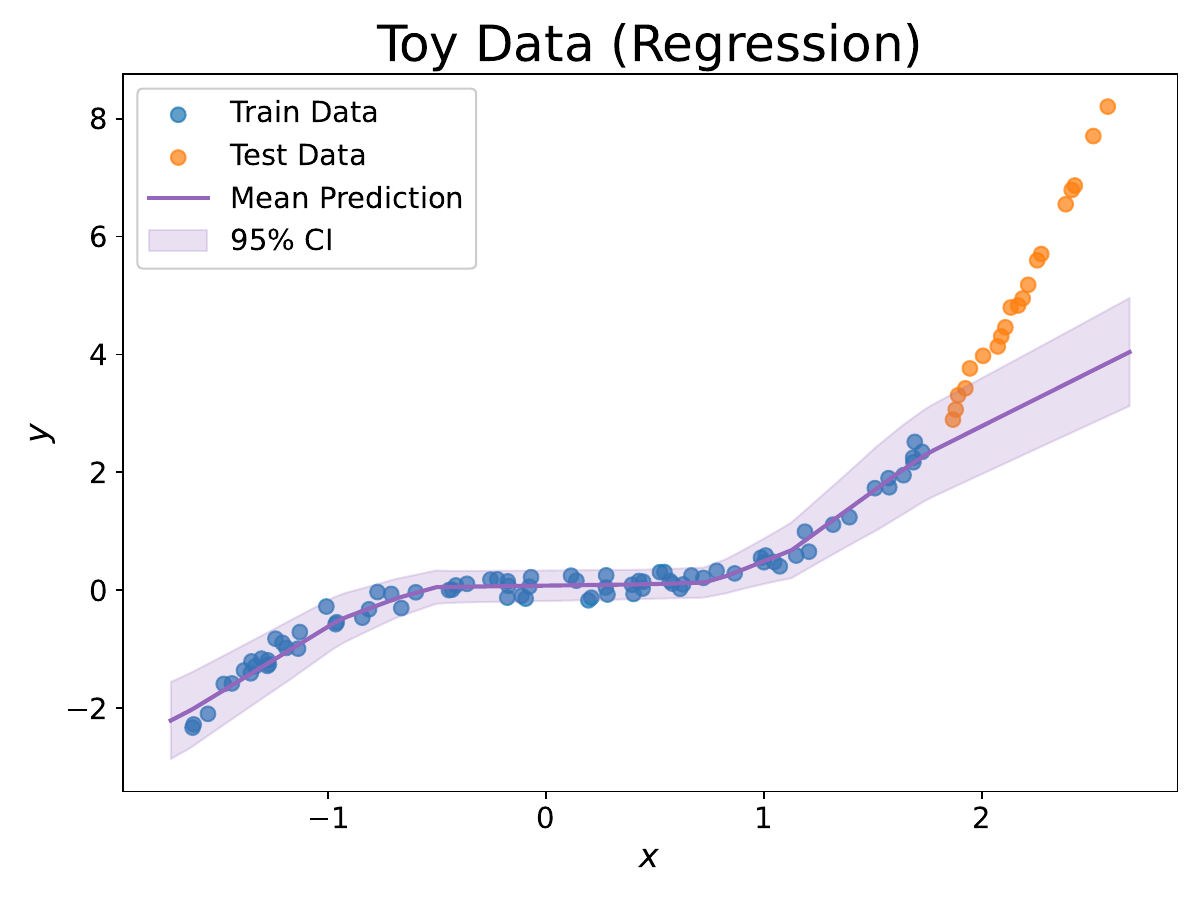}
        \caption{\textsf{M-ProbMix}.}
        \label{fig: app_erm_toy_regression_mprobmix}
    \end{subfigure}%
    \hfill
    \begin{subfigure}{0.24\textwidth}
        \centering
        \includegraphics[trim={0, 0, 0, 1.25cm}, clip, width=\textwidth]{figures/visual_examples/regression_examples/probmix_local.pdf}
        \caption{\textsf{Loc$^\mathsf{K}$ProbMix}.}
        \label{fig: app_erm_toy_regression_lockprobmix}
    \end{subfigure}%
    \hfill
    \begin{subfigure}{0.24\textwidth}
        \centering
        \includegraphics[trim={0, 0, 0, 1.25cm}, clip, width=\textwidth]{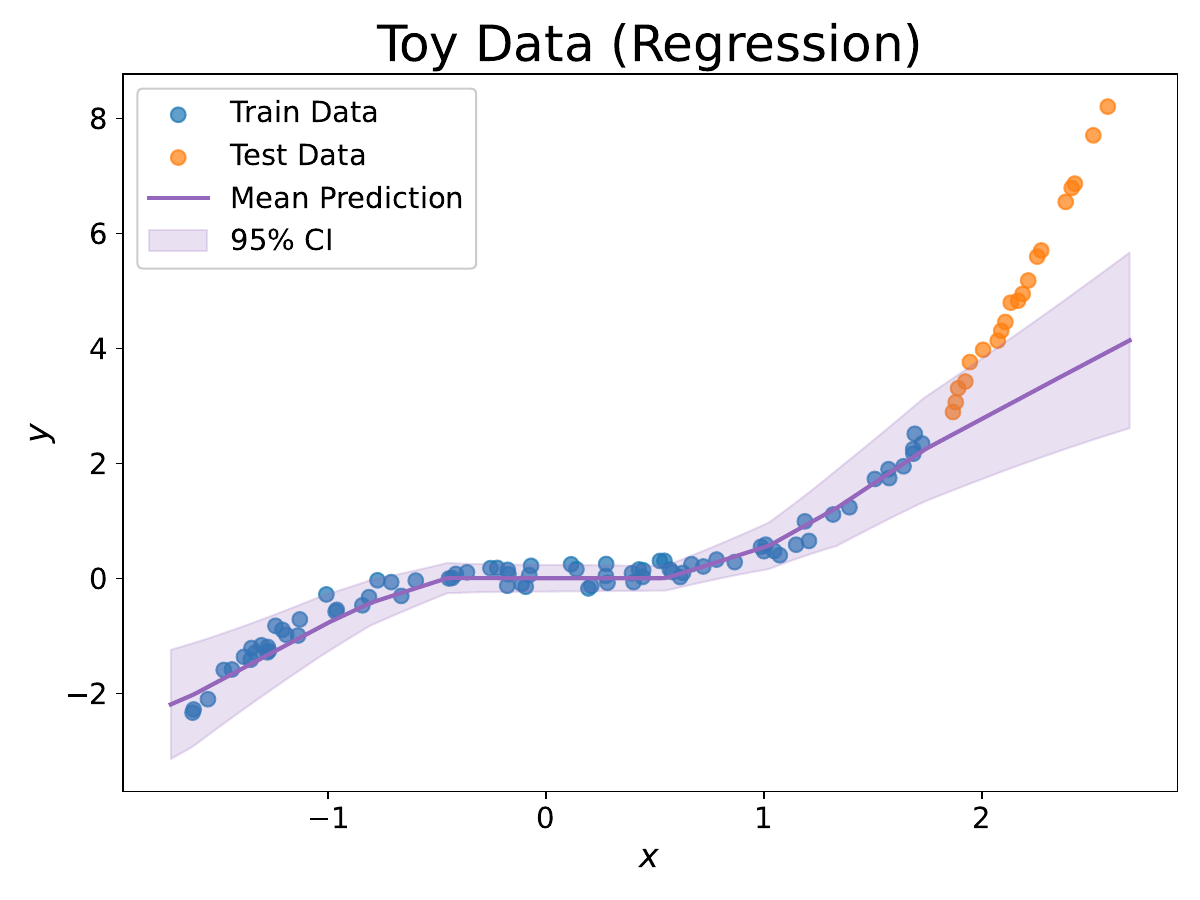}
        \caption{\textsf{Loc$^\mathsf{K}$M-ProbMix}.}
        \label{fig: app_erm_toy_regression_lockmprobmix}
    \end{subfigure}
    \caption{Visual example of different approaches for toy regression. \textsf{Loc$^\mathsf{K}$ProbMix}, and \textsf{Loc$^\mathsf{K}$M-ProbMix} have widest uncertainty bounds on out-of-sample inputs, indicaitng that these variants are best calibrated in terms of uncertainty.}
    \label{fig: all_regression_examples}
\end{figure*}

\begin{figure}[htbp]
    \centering
    \begin{subfigure}{\textwidth}
        \centering
        \includegraphics[width=0.8\textwidth]{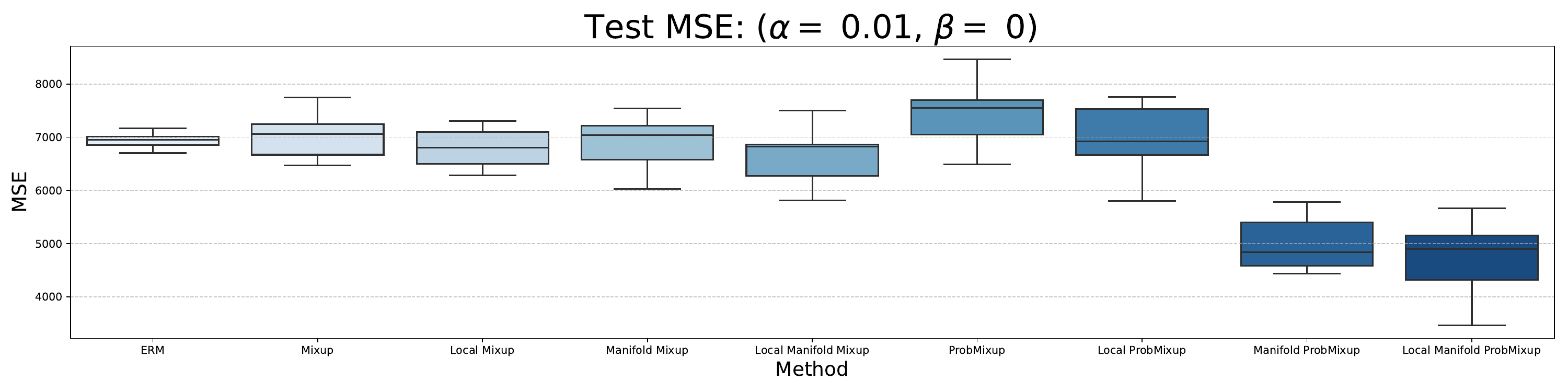}
        \caption{Average MSE.}
        \label{subfig: toy_regression_alpha_0p01_beta_0_mse}
    \end{subfigure}
    
    \vspace{1em} 

    \begin{subfigure}{\textwidth}
        \centering
        \includegraphics[width=0.8\textwidth]{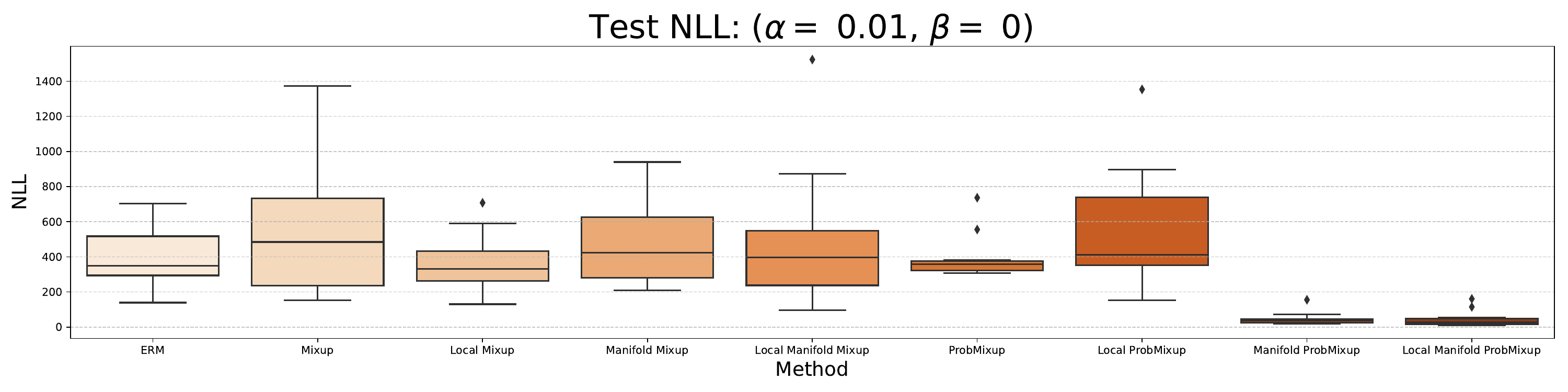}
        \caption{Average negative log-likelihood.}
        \label{subfig: toy_regression_alpha_0p01_beta_0_nll}
    \end{subfigure}
    
    \caption{Toy regression results for $\alpha=0.01$ and $\beta=0$.}
    \label{fig: toy_regression_alpha_0p01_beta_0}
\end{figure}

\begin{figure}[htbp]
    \centering
    \begin{subfigure}{\textwidth}
        \centering
        \includegraphics[width=0.8\textwidth]{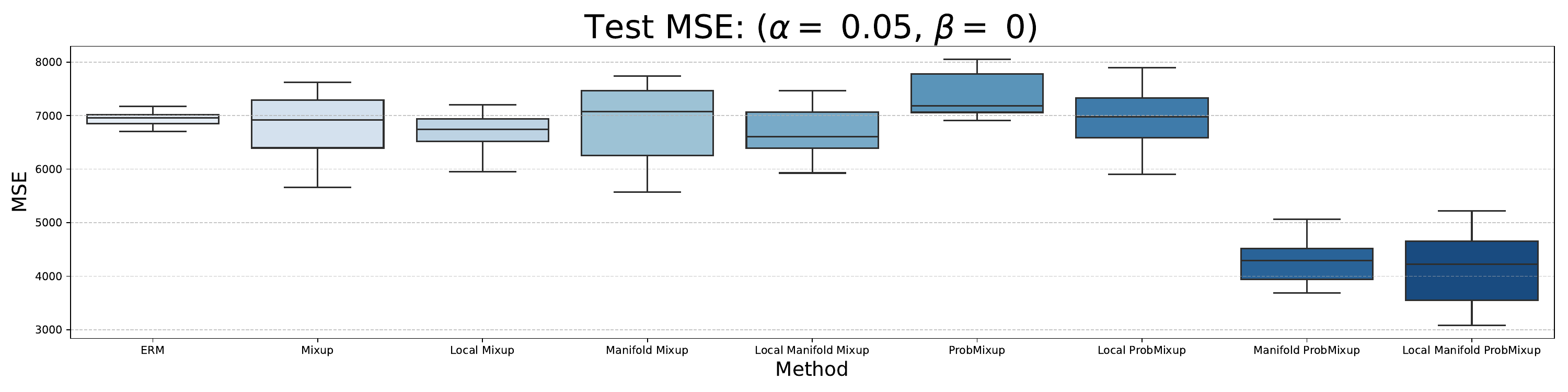}
        \caption{Average MSE.}
        \label{subfig: toy_regression_alpha_0p05_beta_0_mse}
    \end{subfigure}
    
    \vspace{1em} 

    \begin{subfigure}{\textwidth}
        \centering
        \includegraphics[width=0.8\textwidth]{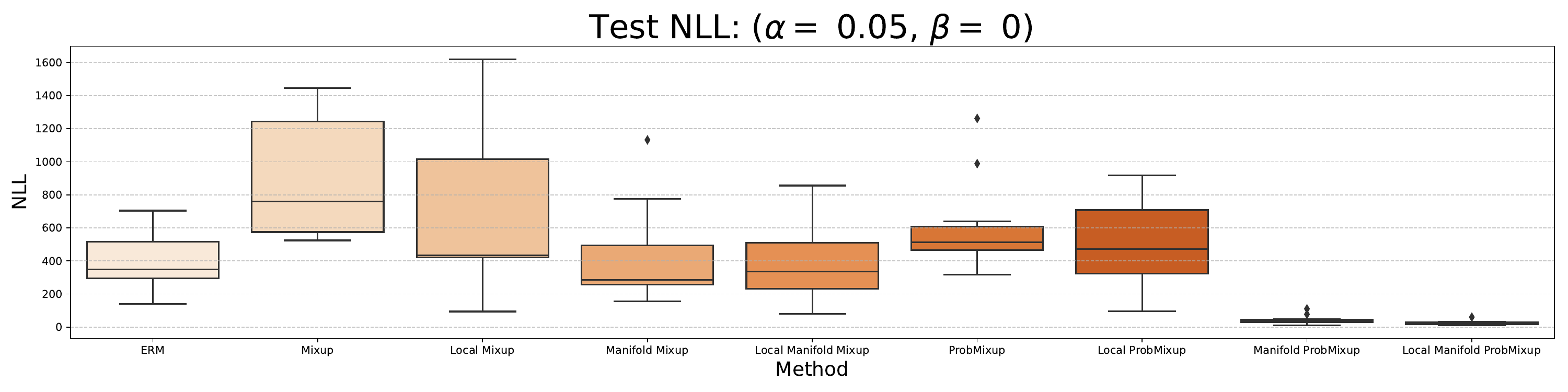}
        \caption{Average negative log-likelihood.}
        \label{subfig: toy_regression_alpha_0p05_beta_0_nll}
    \end{subfigure}
    
    \caption{Toy regression results for $\alpha=0.05$ and $\beta=0$.}
    \label{fig: toy_regression_alpha_0p05_beta_0}
\end{figure}

\begin{figure}[htbp]
    \centering
    \begin{subfigure}{\textwidth}
        \centering
        \includegraphics[width=0.8\textwidth]{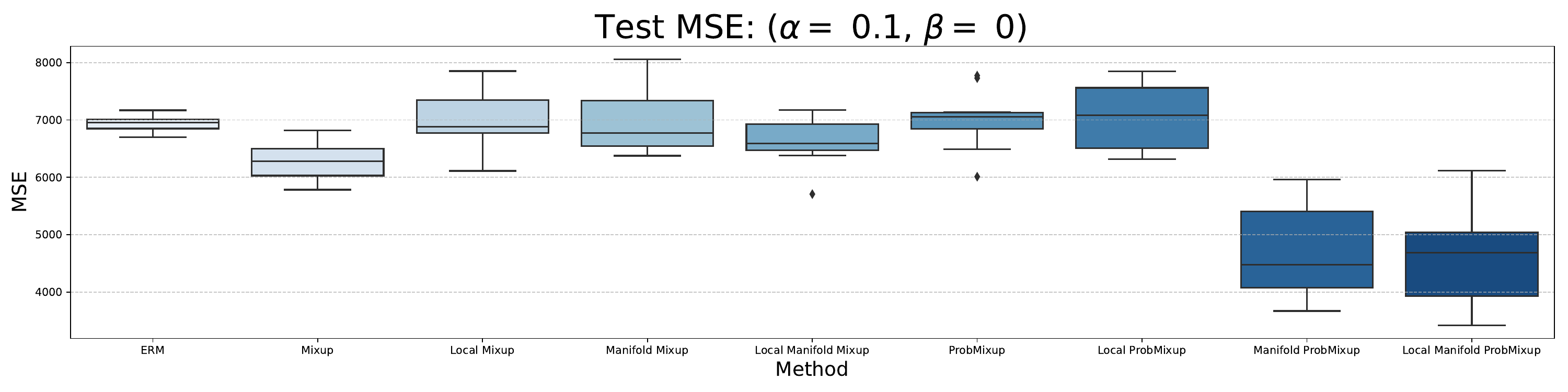}
        \caption{Average MSE.}
        \label{subfig: toy_regression_alpha_0p1_beta_0_mse}
    \end{subfigure}
    
    \vspace{1em} 

    \begin{subfigure}{\textwidth}
        \centering
        \includegraphics[width=0.8\textwidth]{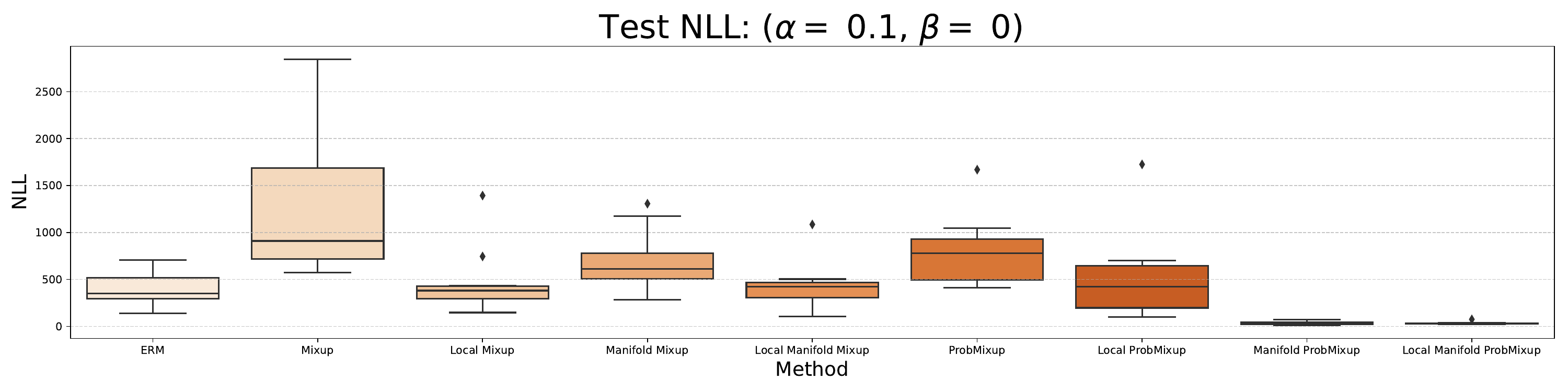}
        \caption{Average negative log-likelihood.}
        \label{subfig: toy_regression_alpha_0p1_beta_0_nll}
    \end{subfigure}
    
    \caption{Toy regression results for $\alpha=0.1$ and $\beta=0$.}
    \label{fig: toy_regression_alpha_0p1_beta_0}
\end{figure}

\begin{figure}[htbp]
    \centering
    \begin{subfigure}{\textwidth}
        \centering
        \includegraphics[width=0.8\textwidth]{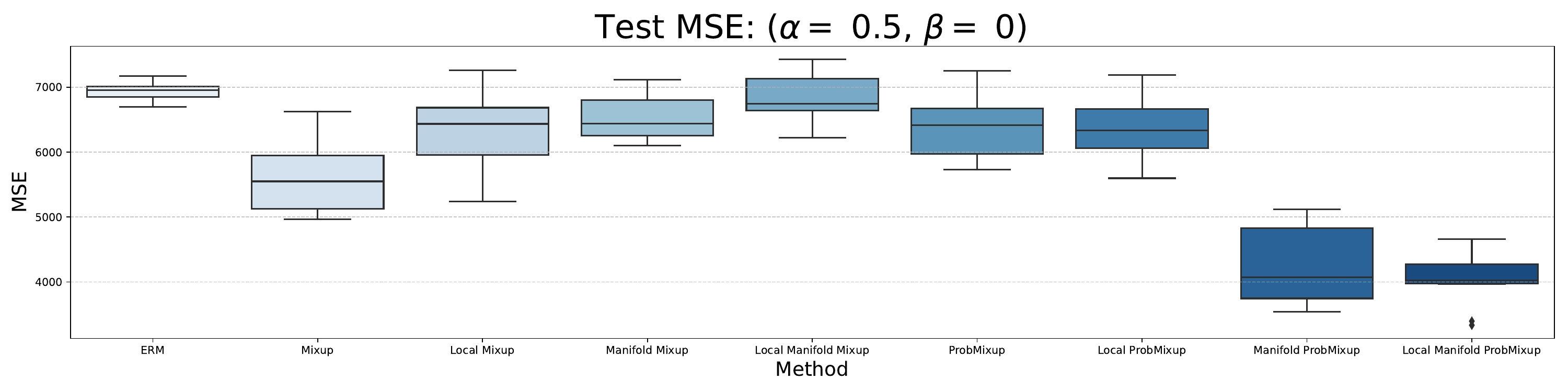}
        \caption{Average MSE.}
        \label{subfig: toy_regression_alpha_0p5_beta_0_mse}
    \end{subfigure}
    
    \vspace{1em} 

    \begin{subfigure}{\textwidth}
        \centering
        \includegraphics[width=0.8\textwidth]{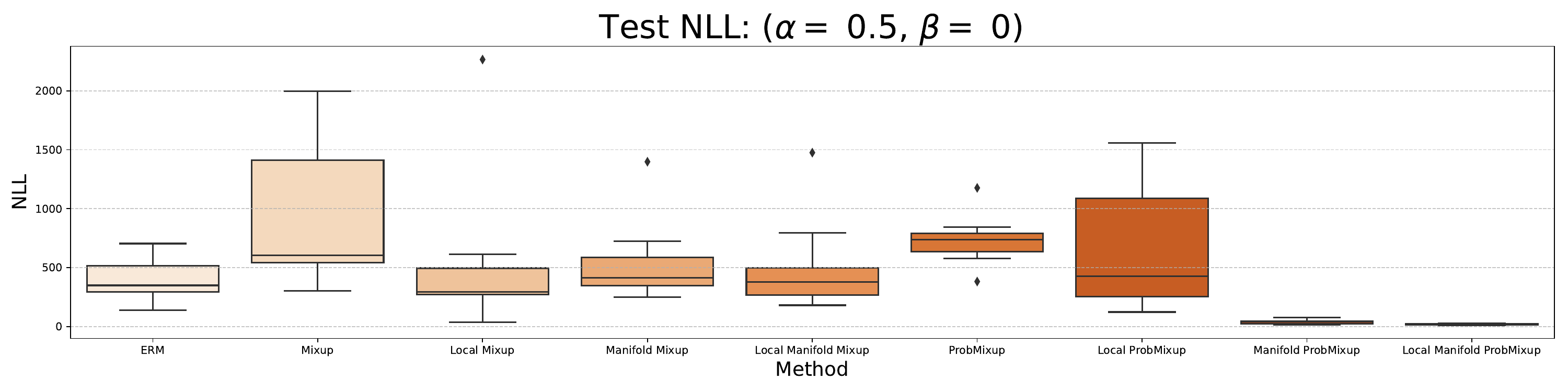}
        \caption{Average negative log-likelihood.}
        \label{subfig: toy_regression_alpha_0p5_beta_0_nll}
    \end{subfigure}
    
    \caption{Toy regression results for $\alpha=0.5$ and $\beta=0$.}
    \label{fig: toy_regression_alpha_0p5_beta_0}
\end{figure}

\begin{figure}[htbp]
    \centering
    \begin{subfigure}{\textwidth}
        \centering
        \includegraphics[width=0.8\textwidth]{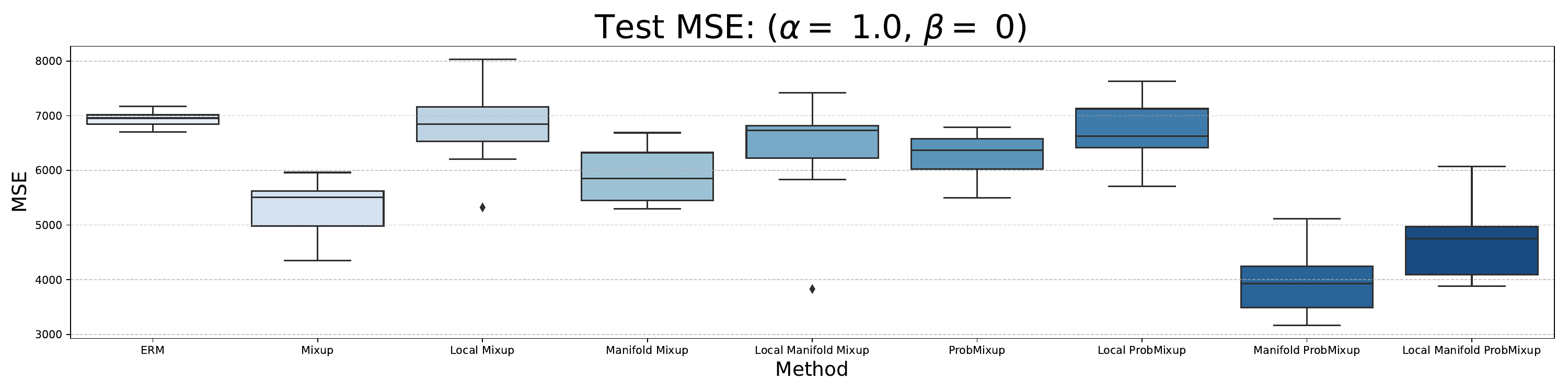}
        \caption{Average MSE.}
        \label{subfig: toy_regression_alpha_1p0_beta_0_mse}
    \end{subfigure}
    
    \vspace{1em} 

    \begin{subfigure}{\textwidth}
        \centering
        \includegraphics[width=0.8\textwidth]{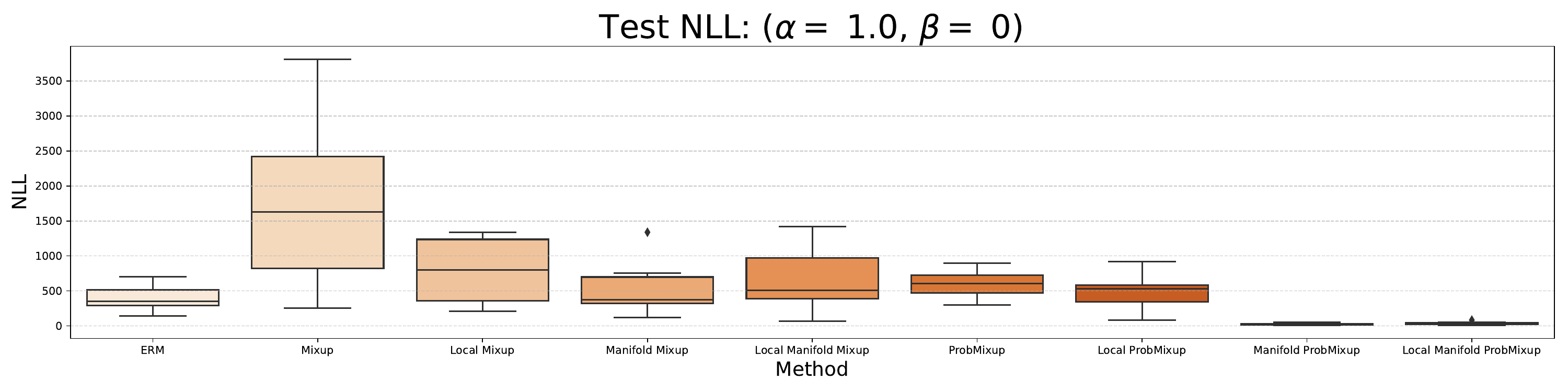}
        \caption{Average negative log-likelihood.}
        \label{subfig: toy_regression_alpha_1p0_beta_0_nll}
    \end{subfigure}
    
    \caption{Toy regression results for $\alpha=1$ and $\beta=0$.}
    \label{fig: toy_regression_alpha_1p0_beta_0}
\end{figure}

\begin{figure}[htbp]
    \centering
    \begin{subfigure}{\textwidth}
        \centering
        \includegraphics[width=0.8\textwidth]{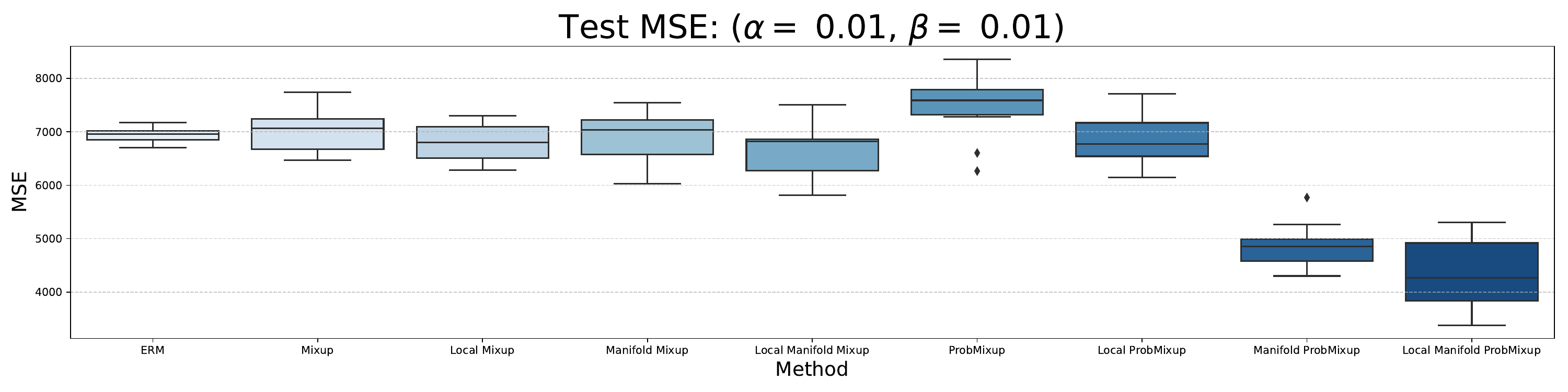}
        \caption{Average MSE.}
        \label{subfig: toy_regression_alpha_0p01_beta_0p01_mse}
    \end{subfigure}
    
    \vspace{1em} 

    \begin{subfigure}{\textwidth}
        \centering
        \includegraphics[width=0.8\textwidth]{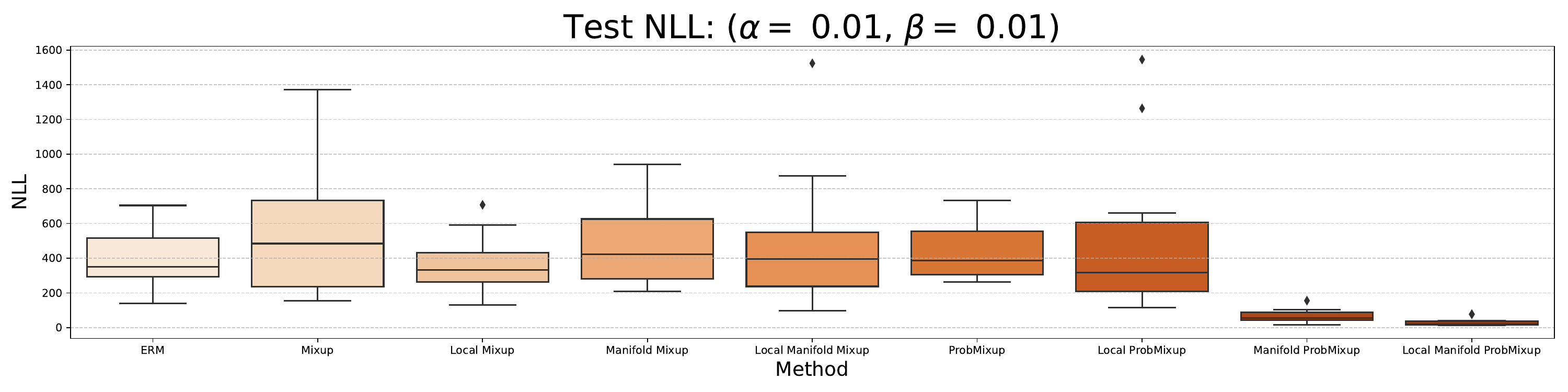}
        \caption{Average negative log-likelihood.}
        \label{subfig: toy_regression_alpha_0p01_beta_0p01_nll}
    \end{subfigure}
    
    \caption{Toy regression results for $\alpha=0.01$ and $\beta=0.01$.}
    \label{fig: toy_regression_alpha_0p01_beta_0p01}
\end{figure}

\begin{figure}[htbp]
    \centering
    \begin{subfigure}{\textwidth}
        \centering
        \includegraphics[width=0.8\textwidth]{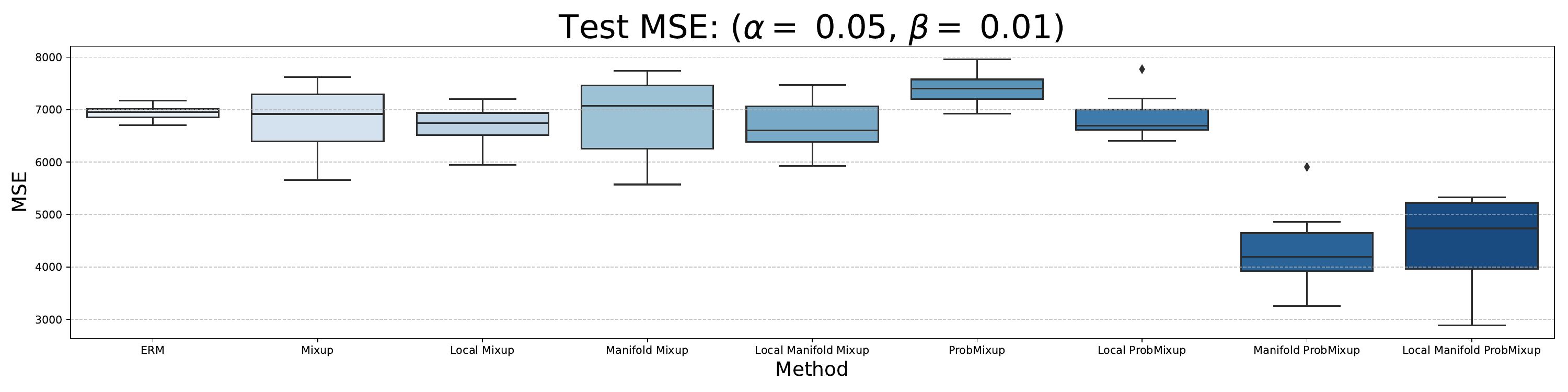}
        \caption{Average MSE.}
        \label{subfig: toy_regression_alpha_0p05_beta_0p01_mse}
    \end{subfigure}
    
    \vspace{1em} 

    \begin{subfigure}{\textwidth}
        \centering
        \includegraphics[width=0.8\textwidth]{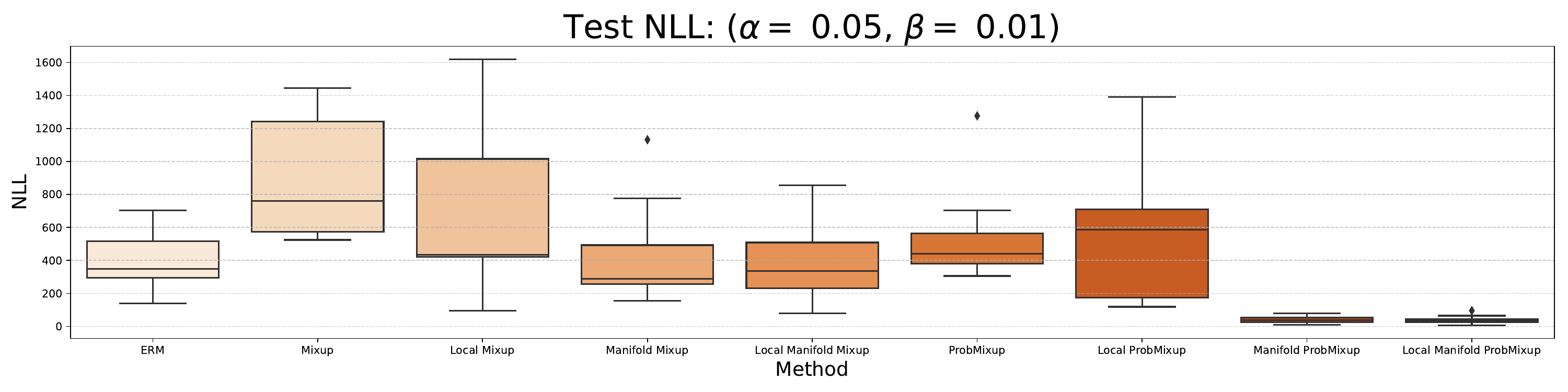}
        \caption{Average negative log-likelihood.}
        \label{subfig: toy_regression_alpha_0p05_beta_0p01_nll}
    \end{subfigure}
    
    \caption{Toy regression results for $\alpha=0.05$ and $\beta=0.01$.}
    \label{fig: toy_regression_alpha_0p05_beta_0p01}
\end{figure}

\begin{figure}[htbp]
    \centering
    \begin{subfigure}{\textwidth}
        \centering
        \includegraphics[width=0.8\textwidth]{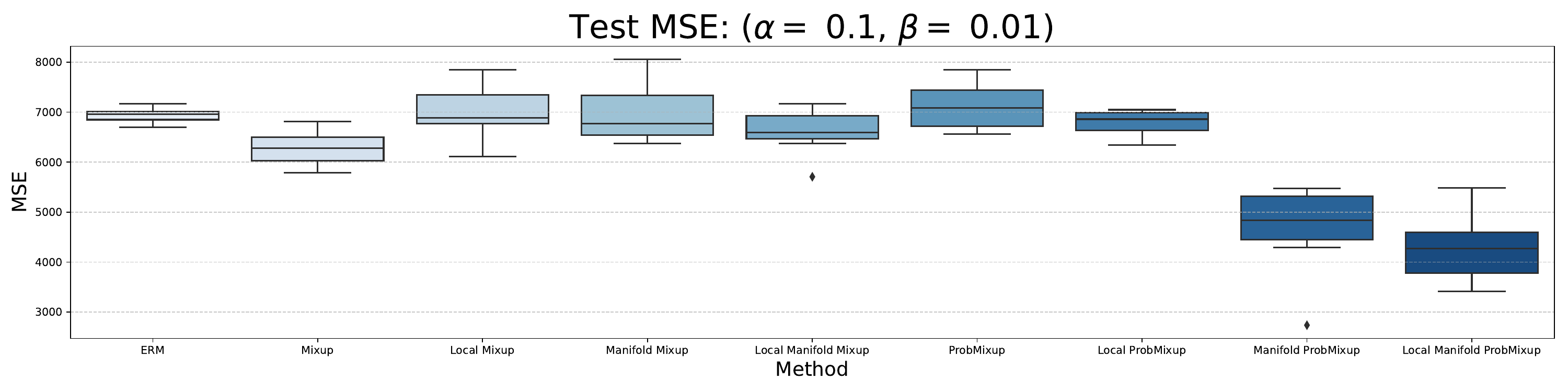}
        \caption{Average MSE.}
        \label{subfig: toy_regression_alpha_0p1_beta_0p01_mse}
    \end{subfigure}
    
    \vspace{1em} 

    \begin{subfigure}{\textwidth}
        \centering
        \includegraphics[width=0.8\textwidth]{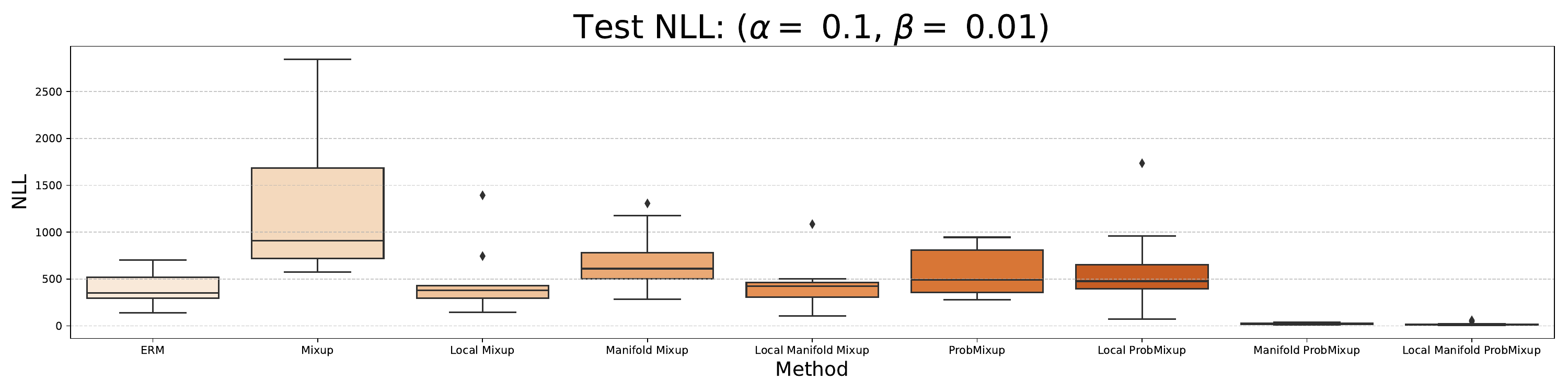}
        \caption{Average negative log-likelihood.}
        \label{subfig: toy_regression_alpha_0p1_beta_0p01_nll}
    \end{subfigure}
    
    \caption{Toy regression results for $\alpha=0.1$ and $\beta=0.01$.}
    \label{fig: toy_regression_alpha_0p1_beta_0p01}
\end{figure}

\begin{figure}[htbp]
    \centering
    \begin{subfigure}{\textwidth}
        \centering
        \includegraphics[width=0.8\textwidth]{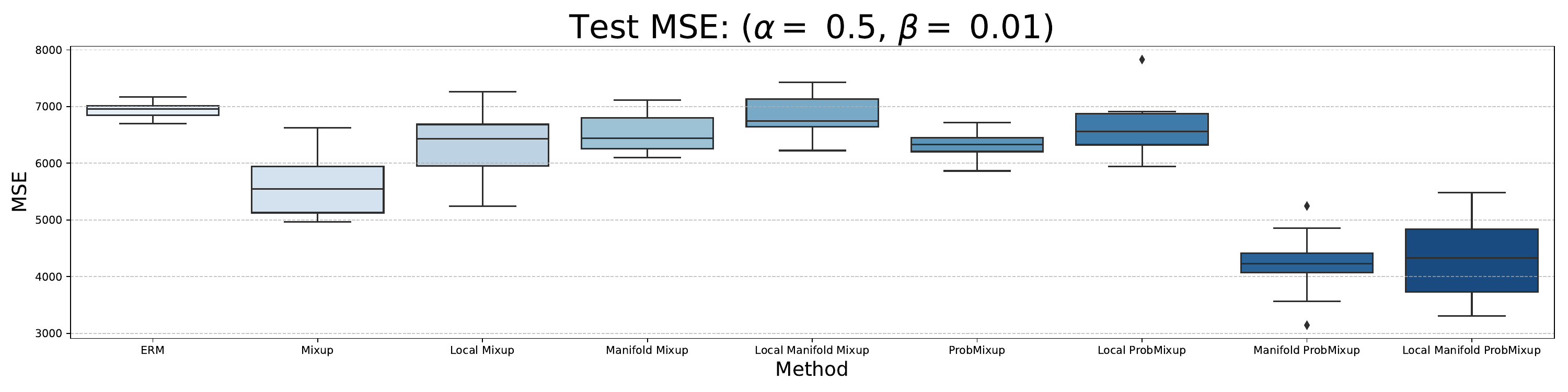}
        \caption{Average MSE.}
        \label{subfig: toy_regression_alpha_0p5_beta_0p01_mse}
    \end{subfigure}
    
    \vspace{1em} 

    \begin{subfigure}{\textwidth}
        \centering
        \includegraphics[width=0.8\textwidth]{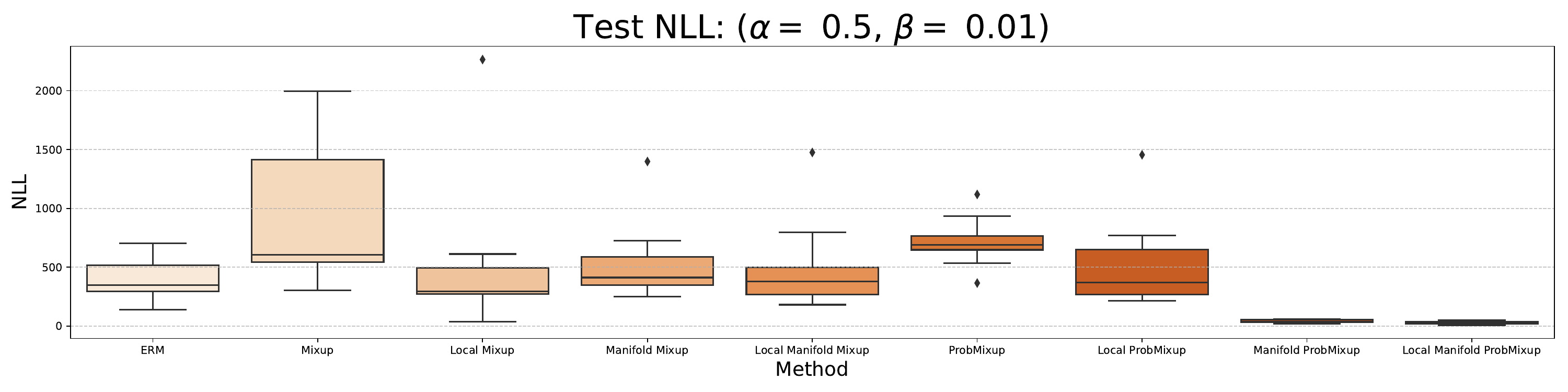}
        \caption{Average negative log-likelihood.}
        \label{subfig: toy_regression_alpha_0p5_beta_0p01_nll}
    \end{subfigure}
    
    \caption{Toy regression results for $\alpha=0.5$ and $\beta=0.01$.}
    \label{fig: toy_regression_alpha_0p5_beta_0p01}
\end{figure}

\begin{figure}[htbp]
    \centering
    \begin{subfigure}{\textwidth}
        \centering
        \includegraphics[width=0.8\textwidth]{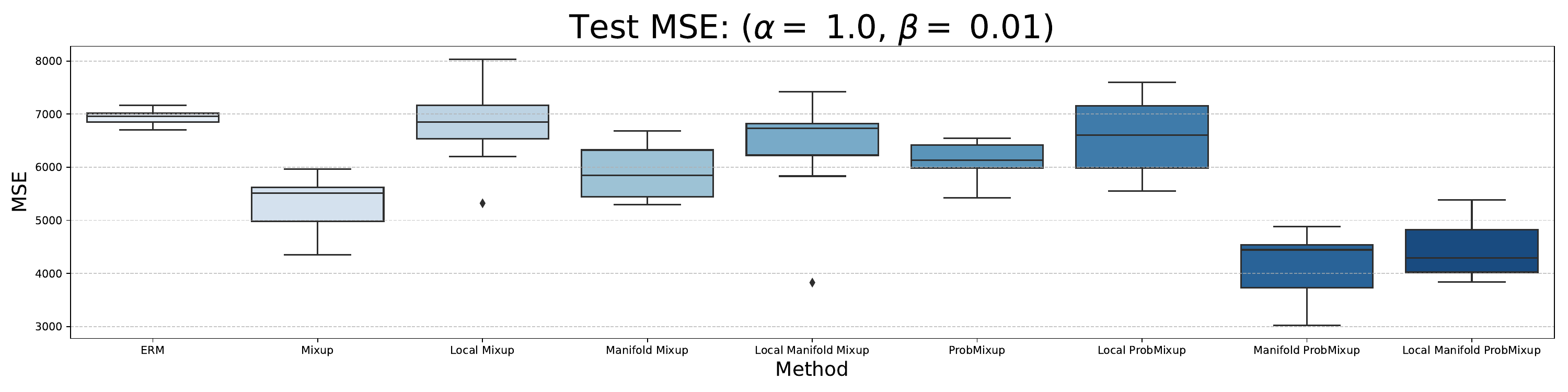}
        \caption{Average MSE.}
        \label{subfig: toy_regression_alpha_1p0_beta_0p01_mse}
    \end{subfigure}
    
    \vspace{1em} 

    \begin{subfigure}{\textwidth}
        \centering
        \includegraphics[width=0.8\textwidth]{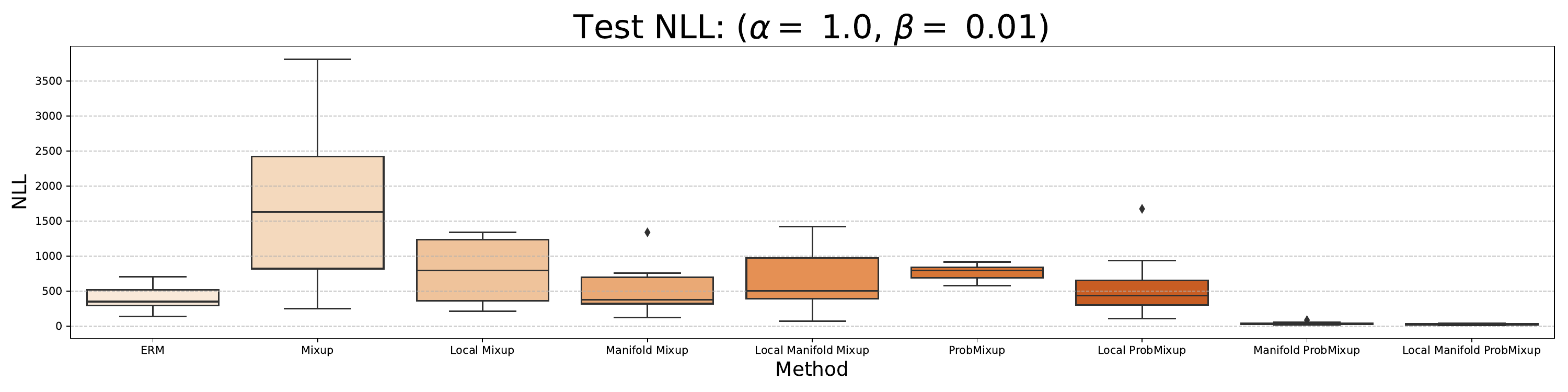}
        \caption{Average negative log-likelihood.}
        \label{subfig: toy_regression_alpha_1p0_beta_0p01_nll}
    \end{subfigure}
    
    \caption{Toy regression results for $\alpha=1.0$ and $\beta=0.01$.}
    \label{fig: toy_regression_alpha_1p0_beta_0p01}
\end{figure}

\subsubsection{Toy Regression: Linear vs. Log-Linear Pooling}
\label{app: exp_pooling_function}
Here, we provide an additional ablation study comparing the performance of linear and log-linear pooling in \textsf{ProbMix} in the context of the Gaussian regression toy example. For linear pooling, the fusion functions $g_\lambda^x$ and $g_\lambda^y$ are assumed to be the following:
\begin{align}
    \tilde{p}_\theta({y}|x_i, x_j, \lambda) = g_\lambda^x(p_\theta(y|x_i), p_\theta(y|x_j)) = \lambda p_\theta(y|x_i) + (1-\lambda)p_\theta(y|x_j)\label{eq: linear_fusion_x} \\
    p_\beta(\tilde{y}|y_i, y_j, \lambda)=g_\lambda^y(s_\beta(\tilde{y}|y_i), s_\beta(\tilde{y}|y_j)) = \lambda s_\beta(\tilde{y}|y_i) + (1-\lambda)s_\beta(\tilde{y}|y_j)\label{eq: linear_fusion_y}
\end{align}

We summarize the results in the ablation on the fusion function in terms of the average NLL for in distribution (ID) data and out-of-distribution (OOD) data in Tables \ref{tab: different_fusion_functions_ind} and \ref{tab: different_fusion_functions_ood}, respectively. Based on the results, our findings are as follows:
\begin{enumerate}
    \item For \textsf{M-ProbMix} and \textsf{Loc$^\mathsf{K}$M-ProbMix}, linear pooling actually achieves better OOD performance than log-linear pooling. We can see in Table \ref{tab: different_fusion_functions_ood}, for example with $\alpha=0.5$, that \textsf{M-ProbMix} with linear pooling achieves an average NLL of 6.1, while the best performing \textsf{ProbMix} variant with log-linear pooling (\textsf{Loc$^\mathsf{K}$M-ProbMix}) only achieves an average NLL of 16.9. 
    \item Improvement of the model in terms of OOD performance comes at a cost of ID performance. We can see that across all settings of $\alpha$, \textsf{M-ProbMix} with linear pooling achieves much higher NLL than its log-linear counterpart (e.g., for $\alpha=0.5$. 0.432 vs. -0.285) due to overcoverage of the quantified uncertainty. 
\end{enumerate}

\begin{table}[!ht]
\centering
\resizebox{\columnwidth}{!}{%
\begin{tabular}{c|cccc|cccc}
\midrule
$\alpha$ & \textsf{ProbMix.} & \textsf{Loc$^\mathsf{K}$ProbMix.} & \textsf{M-ProbMix.} & \textsf{Loc$^\mathsf{K}$M-ProbMix.} & \textsf{ProbMix.} & \textsf{Loc$^\mathsf{K}$ProbMix.} & \textsf{M-ProbMix.} & \textsf{Loc$^\mathsf{K}$M-ProbMix.} \\
& (Linear) & (Linear) & (Linear) & (Linear) & (Log-Linear) & (Log-Linear) & (Log-Linear) & (Log-Linear) \\
\midrule
$1e^{-32}$ & -0.76$\pm$0.01 & -0.66$\pm$0.02 & -0.38$\pm$0.01 & -0.30$\pm$0.01 & \textbf{-0.78$\pm$0.01} & -0.62$\pm$0.02 & -0.38$\pm$0.01 & -0.28$\pm$0.03  \\ 
0.01 & \textbf{-0.77$\pm$0.01} & -0.66$\pm$0.01 & -0.33$\pm$0.01 & -0.30$\pm$0.02 & -0.72$\pm$0.01 & -0.67$\pm$0.01 & -0.39$\pm$0.02 & -0.33$\pm$0.01  \\ 
0.05 & \textbf{-0.76$\pm$0.01} & -0.62$\pm$0.02 & -0.01$\pm$0.02 & -0.30$\pm$0.01 & -0.71$\pm$0.01 & -0.65$\pm$0.01 & -0.36$\pm$0.01 & -0.29$\pm$0.01  \\ 
0.10 & \textbf{-0.74$\pm$0.01} & -0.66$\pm$0.02 & 0.11$\pm$0.04 & -0.31$\pm$0.02 & -0.66$\pm$0.01 & -0.67$\pm$0.01 & -0.35$\pm$0.01 & -0.32$\pm$0.01  \\ 
0.50 & \textbf{-0.76$\pm$0.01} & -0.65$\pm$0.02 & 0.43$\pm$0.04 & -0.29$\pm$0.02 & -0.65$\pm$0.01 & -0.64$\pm$0.01 & -0.28$\pm$0.01 & -0.30$\pm$0.01  \\ 
1 & \textbf{-0.74$\pm$0.01} & -0.66$\pm$0.02 & 0.36$\pm$0.03 & -0.25$\pm$0.01 & -0.51$\pm$0.02 & -0.51$\pm$0.02 & -0.25$\pm$0.01 & -0.33$\pm$0.01  \\  
\bottomrule
\end{tabular}
}
\caption{Linear vs. Log-linear, changing hyper-parameter $\alpha$, in-distribution NLL.}
    \label{tab: different_fusion_functions_ind}
\end{table}


\begin{table}[!ht]
\centering
\resizebox{\columnwidth}{!}{%
\begin{tabular}{c|cccc|cccc}
\midrule
$\alpha$ & \textsf{ProbMix.} & \textsf{Loc$^\mathsf{K}$ProbMix.} & \textsf{M-ProbMix.} & \textsf{Loc$^\mathsf{K}$M-ProbMix.} & \textsf{ProbMix.} & \textsf{Loc$^\mathsf{K}$ProbMix.} & \textsf{M-ProbMix.} & \textsf{Loc$^\mathsf{K}$M-ProbMix.} \\
& (Linear) & (Linear) & (Linear) & (Linear) & (Log-Linear) & (Log-Linear) & (Log-Linear) & (Log-Linear) \\
\midrule
$1e^{-32}$ & 94.9$\pm$11.8 & 157.1$\pm$40.8 & 21.3$\pm$3.3 & \textbf{16.0$\pm$3.2} & 92.6$\pm$12.9 & 229.4$\pm$133.6 & 29.5$\pm$6.6 & 78.6$\pm$62.5  \\ 
0.01 & 103.5$\pm$15.0 & 171.1$\pm$41.0 & 28.4$\pm$8.8 & 22.7$\pm$4.3 & 78.1$\pm$8.1 & 132.7$\pm$38.7 & 27.5$\pm$6.2 & \textbf{18.6$\pm$6.3}  \\ 
0.1 & 90.2$\pm$12.7 & 180.7$\pm$30.0 & \textbf{6.6$\pm$0.9} & 17.4$\pm$3.0 & 117.8$\pm$27.0 & 157.2$\pm$18.1 & 31.0$\pm$5.4 & 17.0$\pm$2.2  \\ 
0.1 & 85.3$\pm$13.5 & 206.9$\pm$30.6 & \textbf{12.3$\pm$3.6} & 24.4$\pm$6.4 & 173.1$\pm$27.3 & 258.4$\pm$52.7 & 25.9$\pm$6.1 & 32.8$\pm$8.1  \\ 
0.5 & 83.4$\pm$10.17 & 149.2$\pm$35.5 & \textbf{6.1$\pm$1.4} & 22.0$\pm$3.9 & 313.6$\pm$59.6 & 176.2$\pm$55.0 & 47.7$\pm$12.2 & 16.9$\pm$3.0  \\ 
1 & 91.8$\pm$7.1 & 104.5$\pm$22.9 & \textbf{7.8$\pm$0.7} & 14.8$\pm$4.8 & 370.2$\pm$54.3 & 117.4$\pm$18.6 & 37.2$\pm$13.5 & 23.8$\pm$5.6  \\ 
\bottomrule
\end{tabular}
}
\caption{Linear vs. Log-linear, changing hyper-parameter $\alpha$, OOD NLL.}
    \label{tab: different_fusion_functions_ood}
\end{table}


\subsubsection{Toy Regression: Latent Variable Modeling}
\label{app: exp_latent_variable_modeling}
We further present results comparing mixup methods and \textsf{ProbMix} variants across different values of $\alpha$. The goal of this ablation is to understand whether or not latent variable modeling  (i.e., mapping to a statistical manifold in the latent space) is the main cause of the strong performance of \textsf{M-ProbMix} and \textsf{Loc$^\mathsf{K}$M-ProbMix} in the toy regression experiment. To that end, we fix $\beta=10^{-32}$ and run all baselines across different values of $\alpha\in\{10^{-32}, 0.01, 0.05, 0.1, 0.25, 0.5, 1\}$. By varying the value of $\alpha$, we can understand whether or not mixing is providing additional boost in performance on the OOD distribution NLL. We provide the OOD NLL for all baselines considered in Table \ref{tab: changing_alpha_param}. From the table, we observe that latent variable modeling alone ($\alpha=10^{-32}$) provides a huge benefit on this toy experiment, improving the negative log-likelihood by nearly an order of magnitude (in comparison to mixup variants, \textsf{ProbMix} and \textsf{Loc$^\mathsf{K}$ProbMix}). Interestingly enough, we can see that in addition to the benefits gained from latent variable modeling, performance can be further improved by incorporating \textsf{ProbMix}, as we get the best performance in the case of \textsf{M-ProbMix} when $\alpha=0.25$ ($22.9\pm 6.6$) and \textsf{Loc$^\mathsf{K}$M-ProbMix} when $\alpha=0.5$ ($16.9\pm 3.0$).

\begin{table}[!ht]
\centering
\resizebox{\columnwidth}{!}{%
\begin{tabular}{c|c|c|c|c|c|c|c|c}
\midrule
\multicolumn{1}{c|}{} & \multicolumn{4}{c|}{Mixup Methods} & \multicolumn{4}{c}{ProbMixup Methods} \\ \midrule
$\alpha$ & \textsf{Mix.} & \textsf{Loc$^\mathsf{K}$Mix.} & \textsf{M-Mix.} & Loc$^\mathsf{K}$M-Mix. & \textsf{ProbMix.} & \textsf{Loc$^\mathsf{K}$ProbMix.} & \textsf{M-ProbMix.} & \textsf{Loc$^\mathsf{K}$M-ProbMix.} \\
\midrule
$1e^{-32}$ & \textbf{87.1$\pm$10.5}	&  143.2$\pm$31.7 & 	\textbf{87.1$\pm$10.5} & 	\textbf{143.2$\pm$31.7} & 	92.6$\pm$13.0	 & 229.4$\pm$133.6	 & 29.5$\pm$6.5	 & 78.6$\pm$62.4 \\
\midrule
0.01 &	98.6$\pm$24.2 &	\textbf{110.7$\pm$11.0} &	162.5$\pm$47.9 &	152.3$\pm$45.4 &	\textbf{78.1$\pm$8.1} &	132.7$\pm$38.7 &	27.5$\pm$6.2 &	18.6$\pm$6.3 \\
0.05 &	176.3$\pm$45.5	& 167.1 $\pm$ 52.8	 & 136.2$\pm$33.5 & 147.1$\pm$25.8 &	117.8$\pm$27.0 & 157.2$\pm$18.1	& 31.0$\pm$5.4 &	\textbf{16.9$\pm$2.2} \\
0.1 &	258.92$\pm$51.1 &	202.1$\pm$41.7 &	148.5 $\pm$29.9 &	177.8$\pm$26.7 &	173.1$\pm$27.3	& 258.4$\pm$52.7 &	25.9 $\pm$6.0 &	32.8$\pm$8.1 \\
0.25 &	270.3$\pm$62.2	& 207.5$\pm$44.5 &	140.3$\pm$33.5 &	204.3$\pm$58.1 &	302.0$\pm$30.8 &	120.2$\pm$18.4 &	\textbf{22.9 $\pm$ 6.7} &	19.1$\pm$3.9 \\
0.5 & 306.0$\pm$38.9 &	219.1$\pm$57.7 &	136.4$\pm$23.8 &	202.4$\pm$34.3	& 313.6$\pm$59.6 &	176.2$\pm$55.0 &	47.7$\pm$12.2 &	\textbf{16.9 $\pm$ 3.0} \\
1 &	310.3$\pm$52.3 &	256.7$\pm$85.0 &	134.1$\pm$34.0 &	186.3$\pm$42.7	& 370.1$\pm$54.3 &	\textbf{117.4$\pm$18.6} &	37.1$\pm$13.4 &	23.8$\pm$5.6 \\
\bottomrule
\end{tabular}
}
\caption{Changing hyper-parameter $\alpha$, latent variable modeling.}
    \label{tab: changing_alpha_param}
\end{table}

\subsubsection{Toy Classification}
\label{app: toy_classification_data}
Figure \ref{fig: all_classification_examples} shows the a plot of the decision boundaries of each network obtain via each of the baseline methods (\textsf{Mix}, \textsf{M-Mix}, and \textsf{Loc$^\mathsf{K}$Mix}, and \textsf{Loc$^\mathsf{K}$M-Mix}) as compared to the proposed approaches (\textsf{ProbMix}, \textsf{M-ProbMix}, and \textsf{Loc$^\mathsf{K}$ProbMix}, and \textsf{Loc$^\mathsf{K}$M-ProbMix}). Base on the plot of the decision boundaries, we can see that for this example \textsf{ProbMix} and  \textsf{Loc$^\mathsf{K}$ProbMix} yield the most reasonable decision boundaries, where  \textsf{Loc$^\mathsf{K}$Mix} and \textsf{Loc$^\mathsf{K}$M-Mix} are close (but less noisy competitor). We can see that \textsf{Mix} and \textsf{M-Mix} suffer from the manifold intrusion issue, since the decision boundaries for the red and yellow class are mixed and deviate more from the ground truth data generating process. We observe a similar problem in the \textsf{M-ProbMix} and \textsf{Loc$^\mathsf{K}$M-ProbMix} methods, showing that in the case of this dataset fusing at the logit level gives better generalization results than fusing on some embedding. 

Similar to the case of the toy regression dataset, we conduct an ablation study for the $\alpha$ and $\beta$ parameters for the toy classification dataset. We tested a grid of hyperparameter values defined by $\alpha\in\{0.01, 0.05, 0.10, 0.5, 1.0\}$ and $\beta\in\{0, 0.01\}$. We show a series of box plots comparing the test set accuracy and the NLL across different methods. Results show that in the case of the classification experiment, probabilistic mixup approaches without manifold augmentation (\textsf{ProbMix} and \textsf{Loc$^\mathsf{K}$ProbMix}) achieve best performance in terms of both accuracy and NLL as compared to their manifold-based counterparts. Please refer to the results as shown in Figures \ref{fig: toy_classification_alpha_0p01_beta_0}-\ref{fig: toy_classification_alpha_1p0_beta_0p01} for more details on the results and the ablation studies

\begin{figure*}[t]
    \centering
    \begin{subfigure}{0.24\textwidth}
        \centering
        \includegraphics[trim={0, 0, 0, 0.8cm}, clip, width=\textwidth]{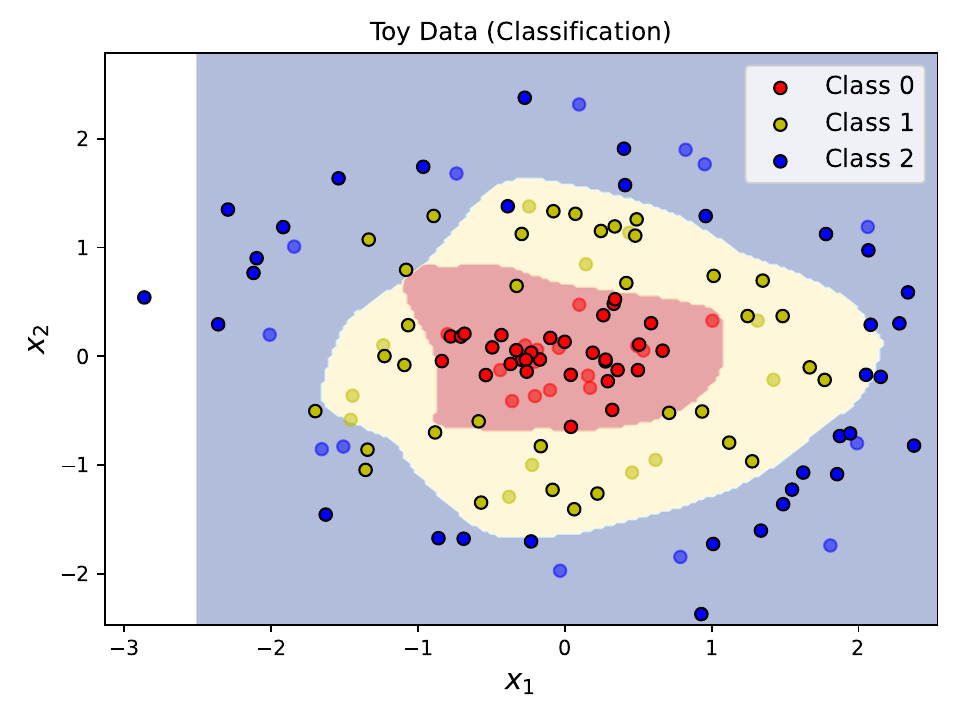}
        \caption{\textsf{Mix}.}
    \end{subfigure}%
    \hfill
    \begin{subfigure}{0.24\textwidth}
        \centering
         \includegraphics[trim={0, 0, 0, 0.8cm}, clip, width=\textwidth]{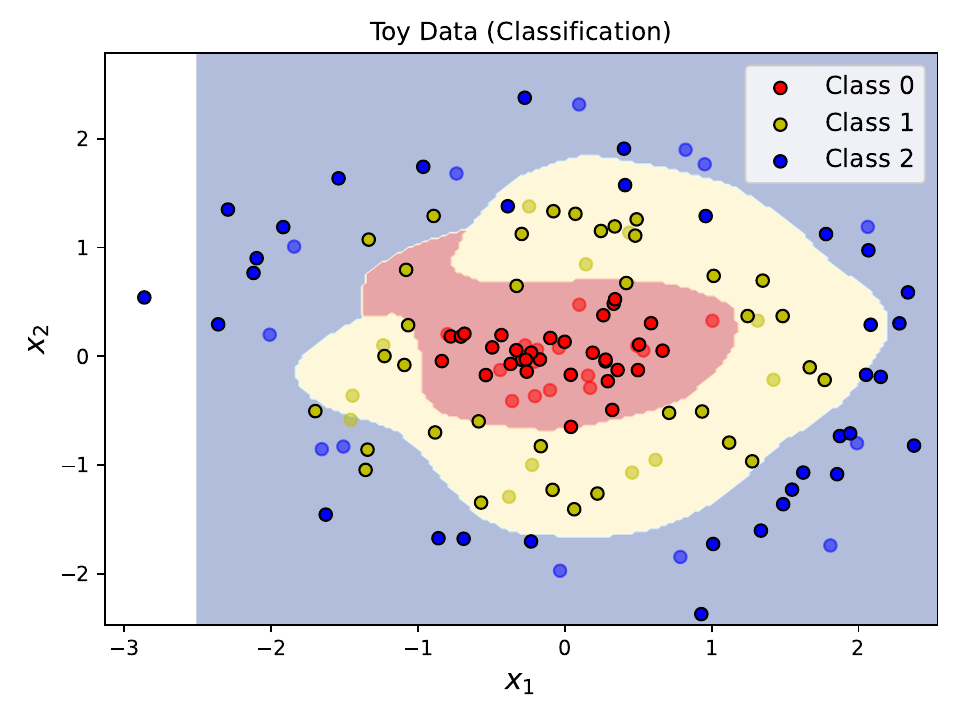}
        \caption{\textsf{M-Mix}.}
    \end{subfigure}%
    \hfill
    \begin{subfigure}{0.24\textwidth}
        \centering
        \includegraphics[trim={0, 0, 0, 0.8cm}, clip, width=\textwidth]{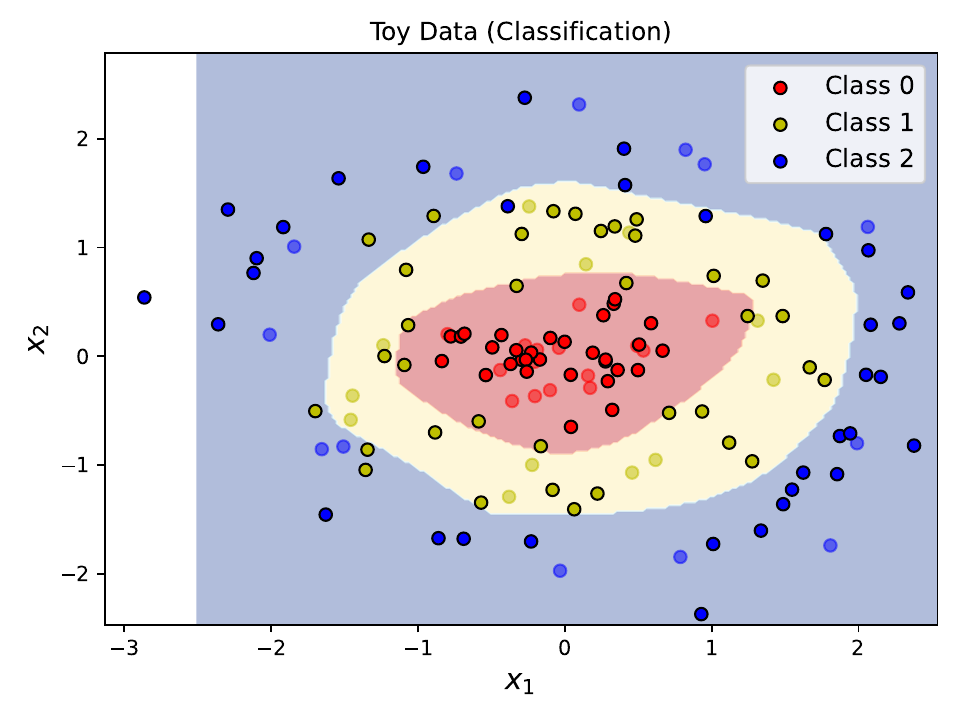}
        \caption{\textsf{Loc$^\mathsf{K}$Mix}.}
    \end{subfigure}%
    \hfill
    \begin{subfigure}{0.24\textwidth}
        \centering
        \includegraphics[trim={0, 0, 0, 0.8cm}, clip, width=\textwidth]{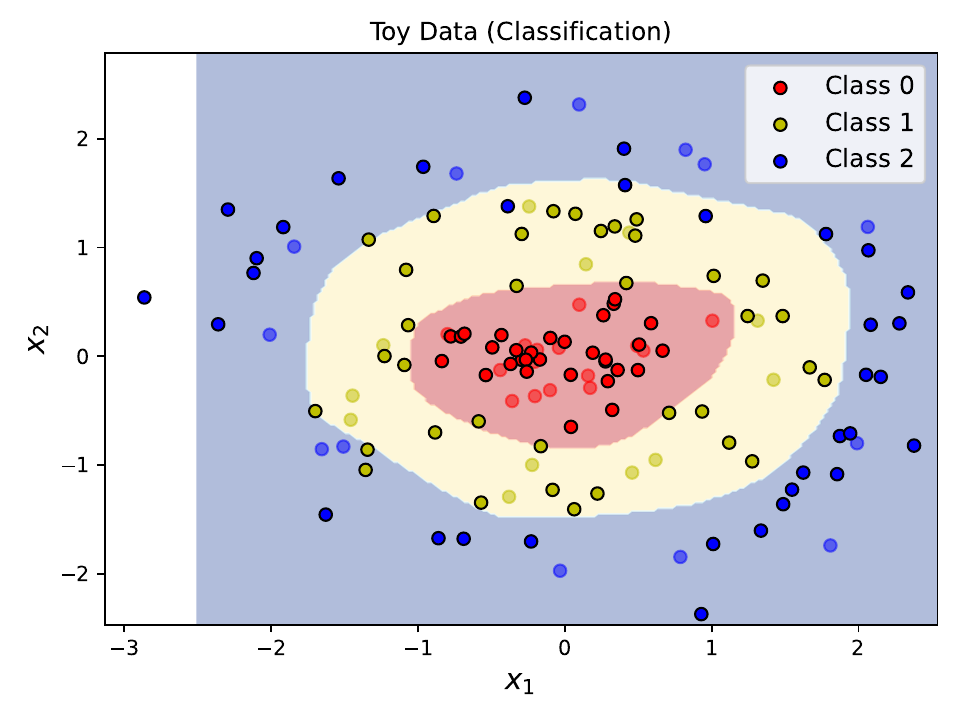}
        \caption{\textsf{Loc$^\mathsf{K}$M-Mix}.}
    \end{subfigure}%
    \vspace{0.1cm}
    \begin{subfigure}{0.24\textwidth}
        \centering
        \includegraphics[trim={0, 0, 0, 0.8cm}, clip, width=\textwidth]{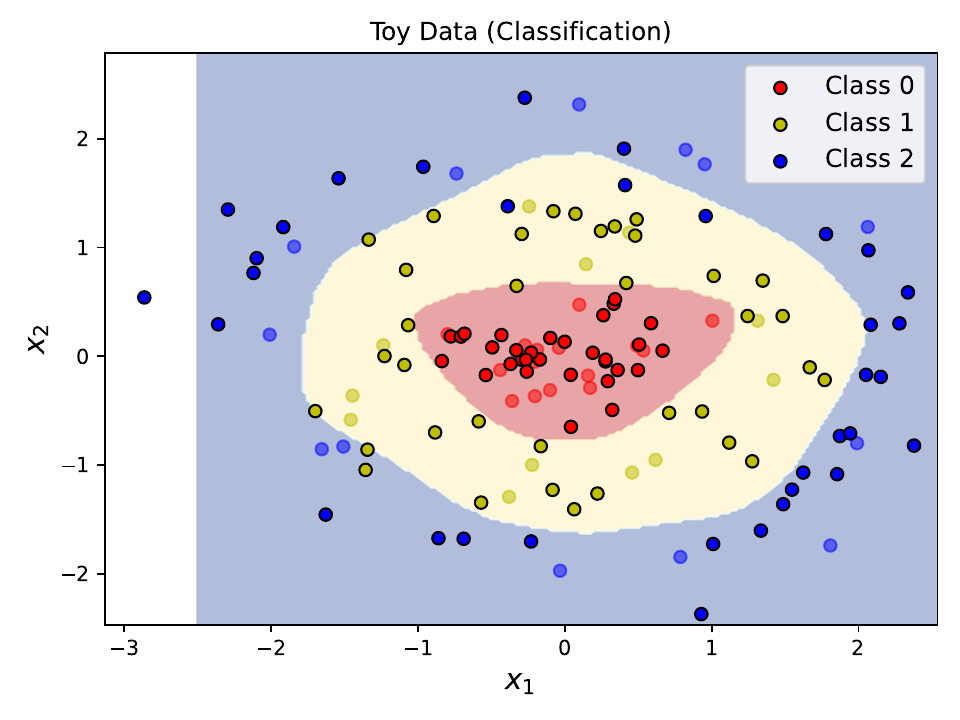}
        \caption{\textsf{ProbMix}.}
    \end{subfigure}%
    \hfill
    \begin{subfigure}{0.24\textwidth}
        \centering
        \includegraphics[trim={0, 0, 0, 0.8cm}, clip, width=\textwidth]{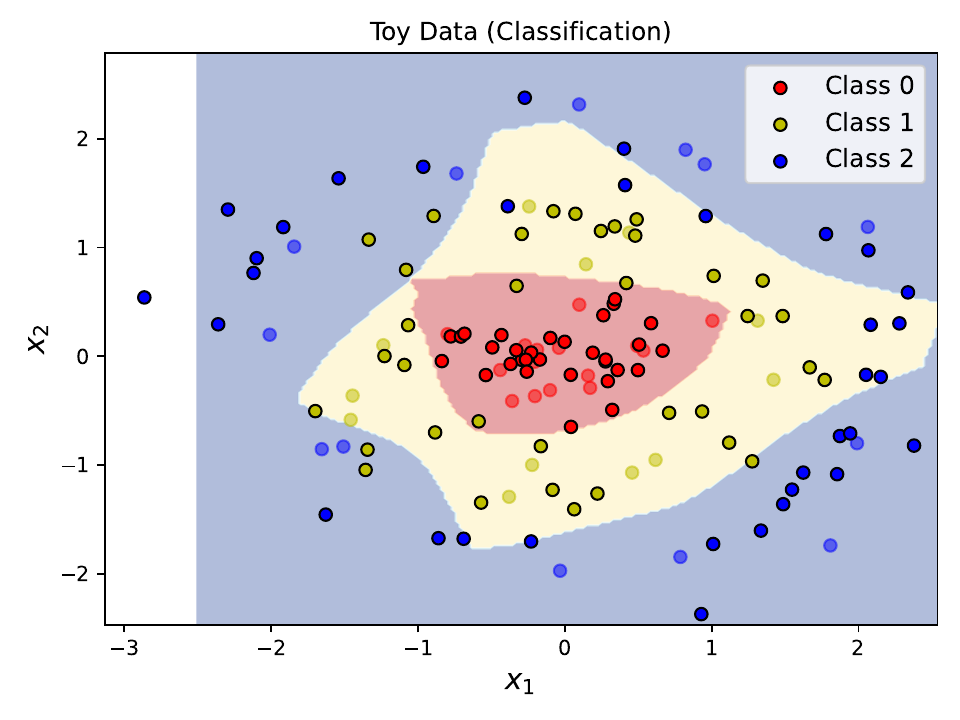}
        \caption{\textsf{M-ProbMix}.}
    \end{subfigure}%
    \hfill
    \begin{subfigure}{0.24\textwidth}
        \centering
        \includegraphics[trim={0, 0, 0, 0.8cm}, clip, width=\textwidth]{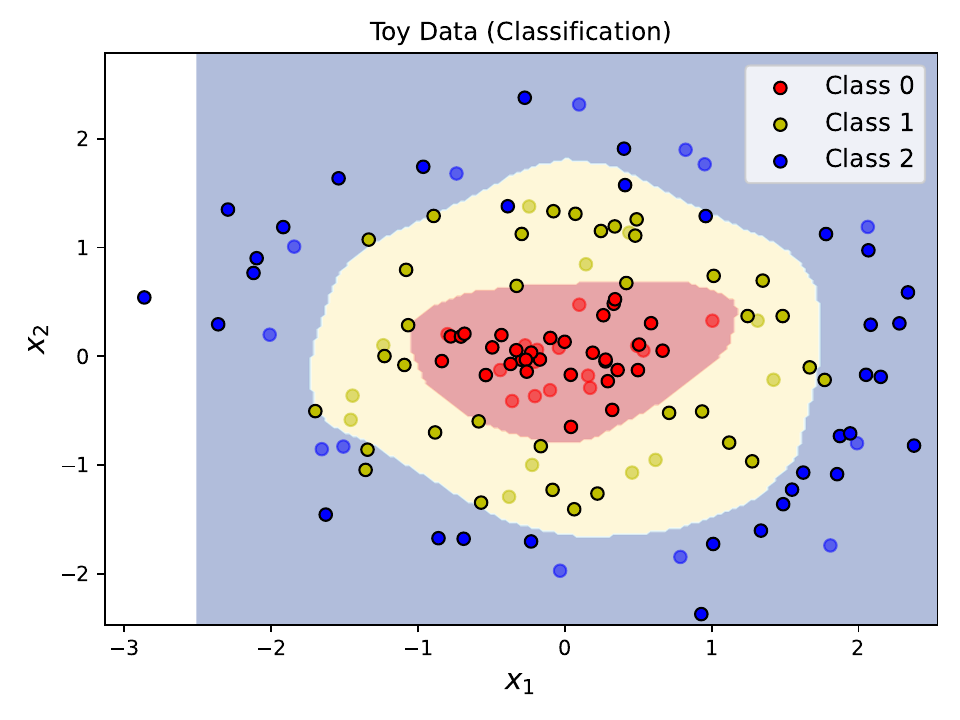}
        \caption{\textsf{Loc$^\mathsf{K}$ProbMix}.}
    \end{subfigure}%
    \hfill
    \begin{subfigure}{0.24\textwidth}
        \centering
        \includegraphics[trim={0, 0, 0, 0.8cm}, clip, width=\textwidth]{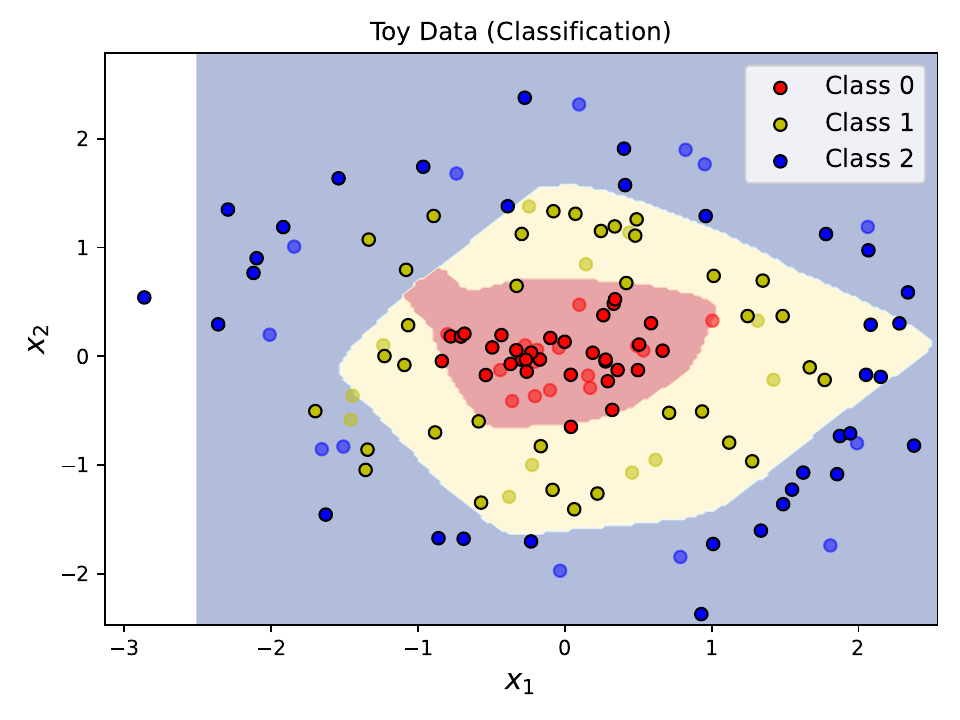}
        \caption{\textsf{Loc$^\mathsf{K}$M-ProbMix}.}
    \end{subfigure}%
    \caption{Visual example of different approaches for toy classification.}
    \label{fig: all_classification_examples}
\end{figure*}

\begin{figure}[htbp]
    \centering
    \begin{subfigure}{\textwidth}
        \centering
        \includegraphics[width=0.8\textwidth]{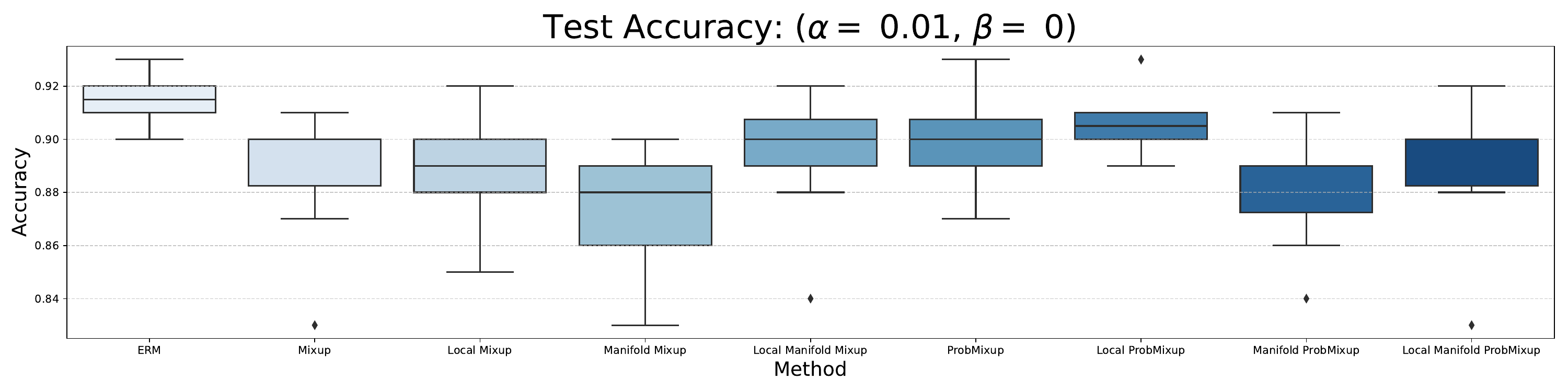}
        \caption{Average accuracy.}
        \label{subfig: toy_classification_alpha_0p01_beta_0_acc}
    \end{subfigure}
    
    \vspace{1em} 

    \begin{subfigure}{\textwidth}
        \centering
        \includegraphics[width=0.8\textwidth]{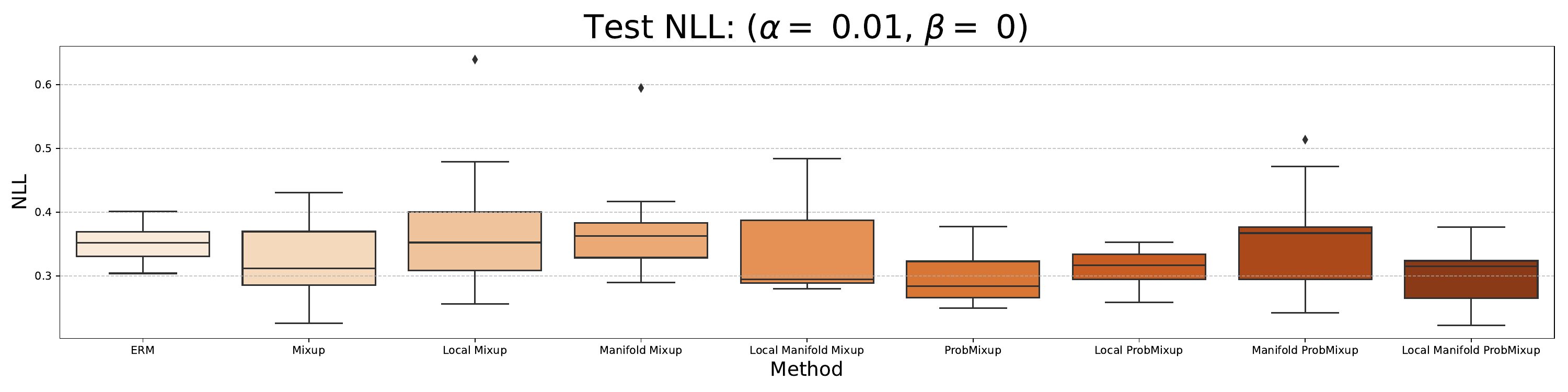}
        \caption{Average negative log-likelihood.}
        \label{subfig: toy_classification_alpha_0p01_beta_0_nll}
    \end{subfigure}
    
    \caption{Toy classification results for $\alpha=0.01$ and $\beta=0$.}
    \label{fig: toy_classification_alpha_0p01_beta_0}
\end{figure}

\begin{figure}[htbp]
    \centering
    \begin{subfigure}{\textwidth}
        \centering
        \includegraphics[width=0.8\textwidth]{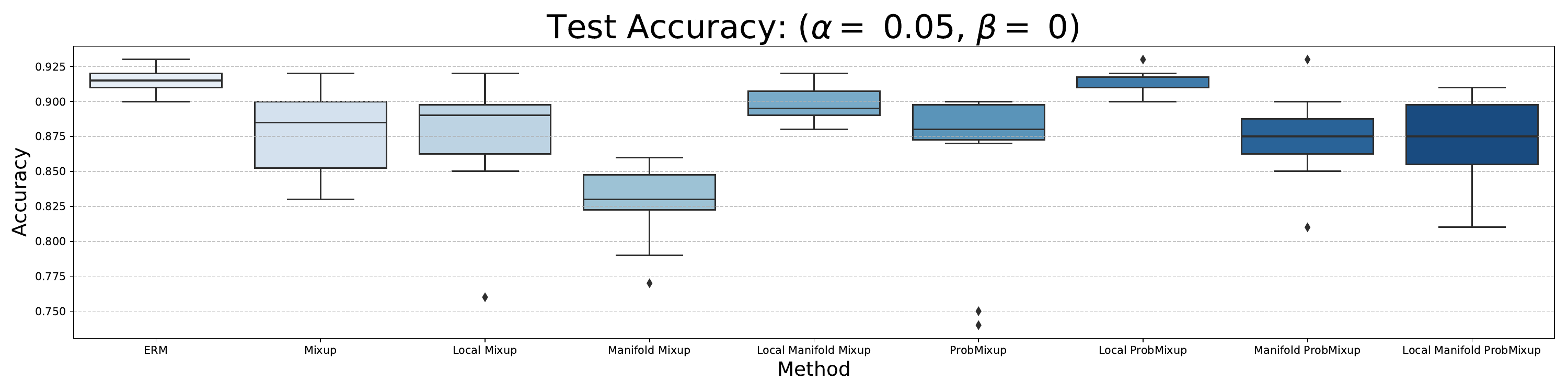}
        \caption{Average accuracy.}
        \label{subfig: toy_classification_alpha_0p05_beta_0_acc}
    \end{subfigure}
    
    \vspace{1em} 

    \begin{subfigure}{\textwidth}
        \centering
        \includegraphics[width=0.8\textwidth]{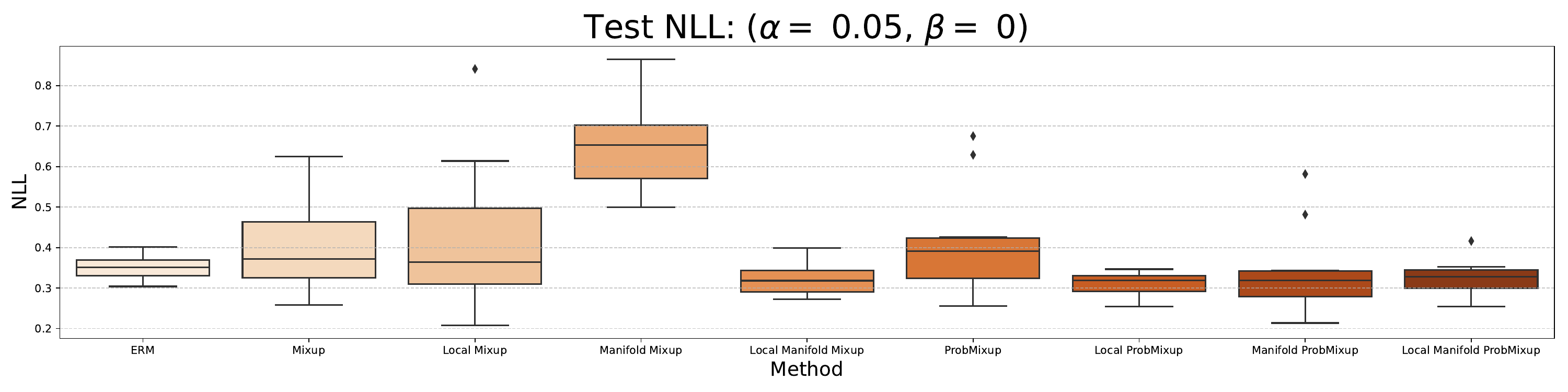}
        \caption{Average negative log-likelihood.}
        \label{subfig: toy_classification_alpha_0p05_beta_0_nll}
    \end{subfigure}
    
    \caption{Toy classification results for $\alpha=0.05$ and $\beta=0$.}
    \label{fig: toy_classification_alpha_0p05_beta_0}
\end{figure}

\begin{figure}[htbp]
    \centering
    \begin{subfigure}{\textwidth}
        \centering
        \includegraphics[width=0.8\textwidth]{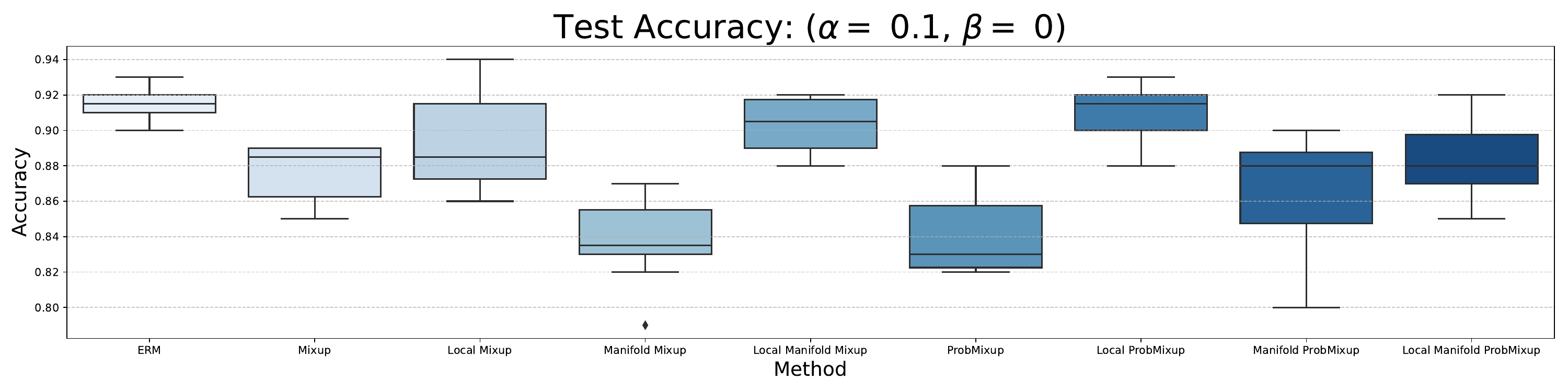}
        \caption{Average accuracy.}
        \label{subfig: toy_classification_alpha_0p1_beta_0_acc}
    \end{subfigure}
    
    \vspace{1em} 

    \begin{subfigure}{\textwidth}
        \centering
        \includegraphics[width=0.8\textwidth]{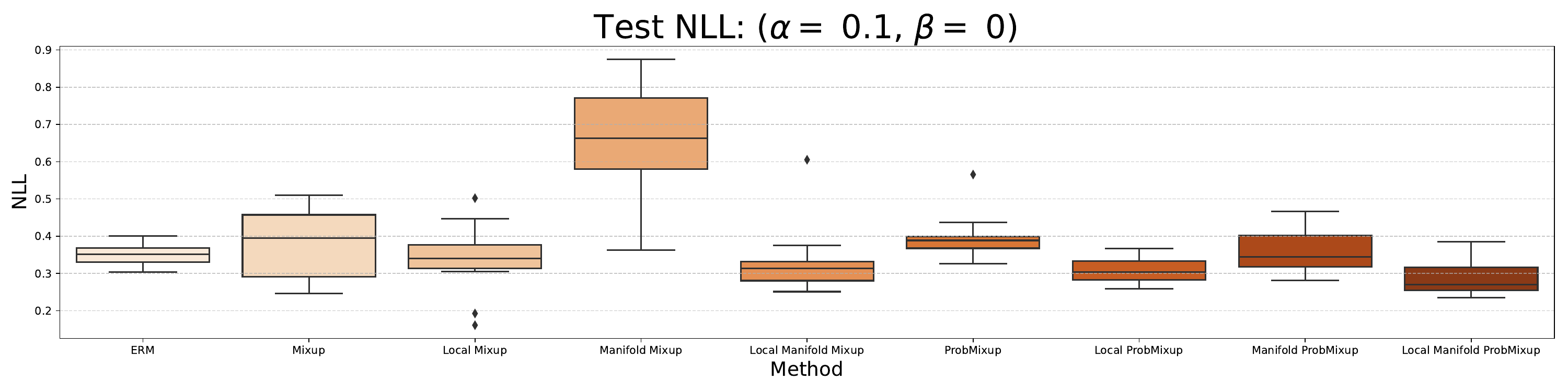}
        \caption{Average negative log-likelihood.}
        \label{subfig: toy_classification_alpha_0p1_beta_0_nll}
    \end{subfigure}
    
    \caption{Toy classification results for $\alpha=0.1$ and $\beta=0$.}
    \label{fig: toy_classification_alpha_0p1_beta_0}
\end{figure}

\begin{figure}[htbp]
    \centering
    \begin{subfigure}{\textwidth}
        \centering
        \includegraphics[width=0.8\textwidth]{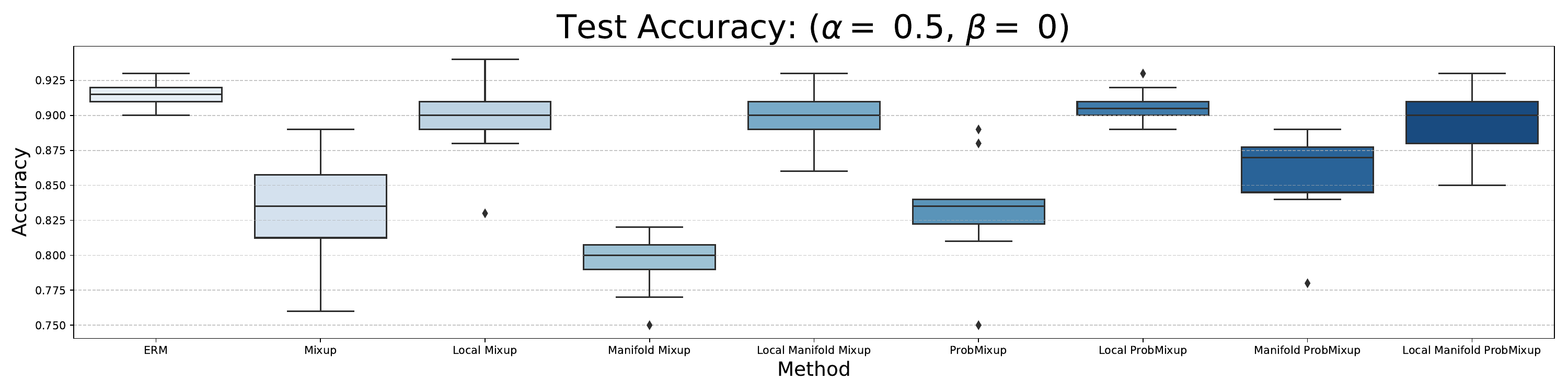}
        \caption{Average accuracy.}
        \label{subfig: toy_classification_alpha_0p5_beta_0_acc}
    \end{subfigure}
    
    \vspace{1em} 

    \begin{subfigure}{\textwidth}
        \centering
        \includegraphics[width=0.8\textwidth]{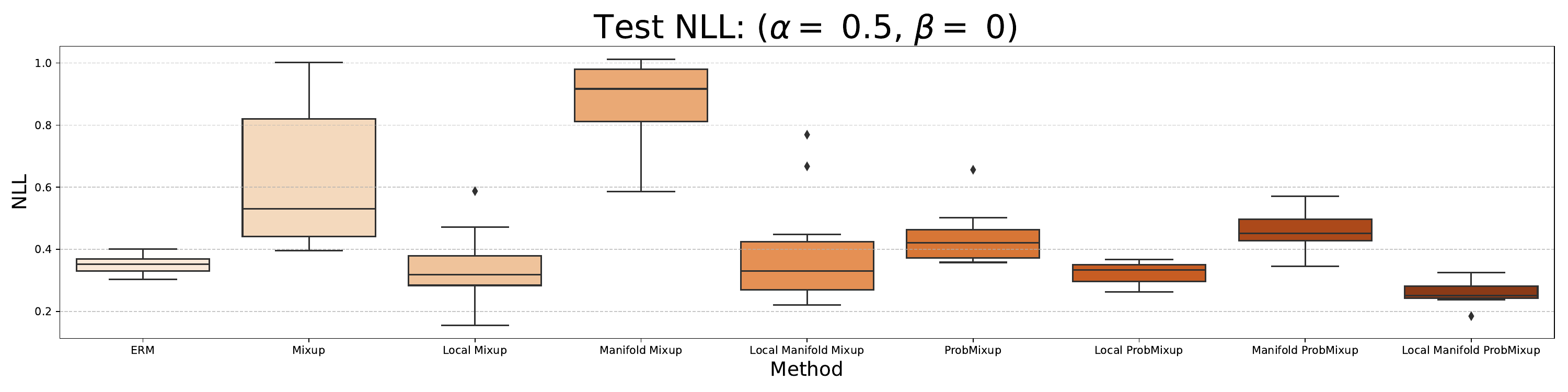}
        \caption{Average negative log-likelihood.}
        \label{subfig: toy_classification_alpha_0p5_beta_0_nll}
    \end{subfigure}
    
    \caption{Toy classification results for $\alpha=0.5$ and $\beta=0$.}
    \label{fig: toy_classification_alpha_0p5_beta_0}
\end{figure}

\begin{figure}[htbp]
    \centering
    \begin{subfigure}{\textwidth}
        \centering
        \includegraphics[width=0.8\textwidth]{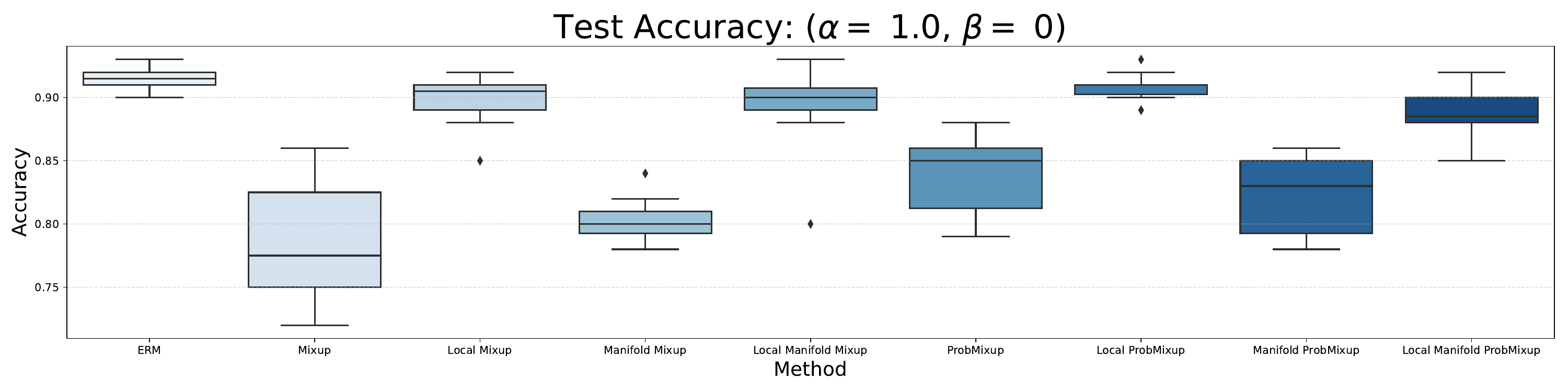}
        \caption{Average accuracy.}
        \label{subfig: toy_classification_alpha_1p0_beta_0_acc}
    \end{subfigure}
    
    \vspace{1em} 

    \begin{subfigure}{\textwidth}
        \centering
        \includegraphics[width=0.8\textwidth]{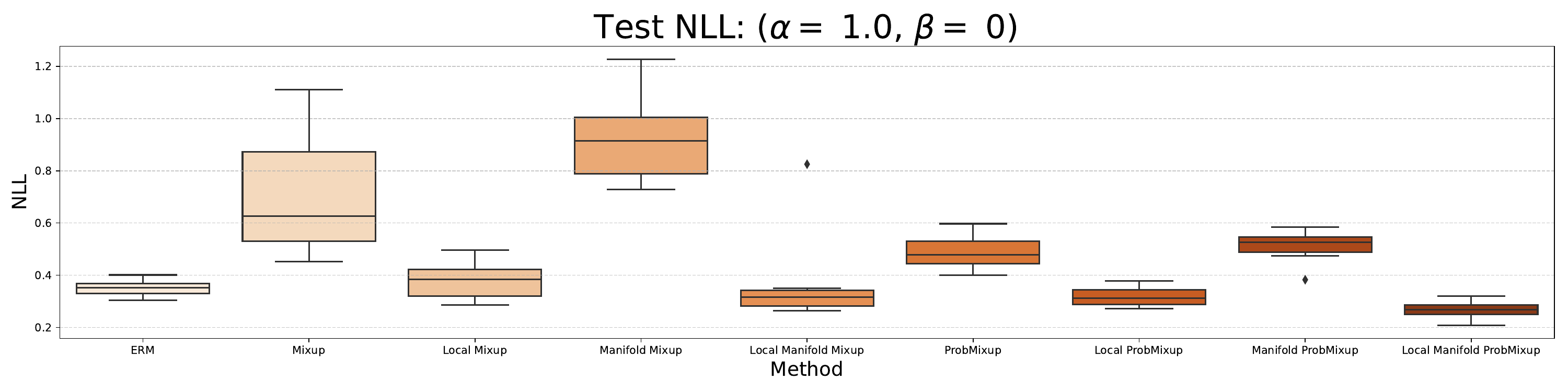}
        \caption{Average negative log-likelihood.}
        \label{subfig: toy_classification_alpha_1p0_beta_0_nll}
    \end{subfigure}
    
    \caption{Toy classification results for $\alpha=1$ and $\beta=0$.}
    \label{fig: toy_classification_alpha_1p0_beta_0}
\end{figure}

\begin{figure}[htbp]
    \centering
    \begin{subfigure}{\textwidth}
        \centering
        \includegraphics[width=0.8\textwidth]{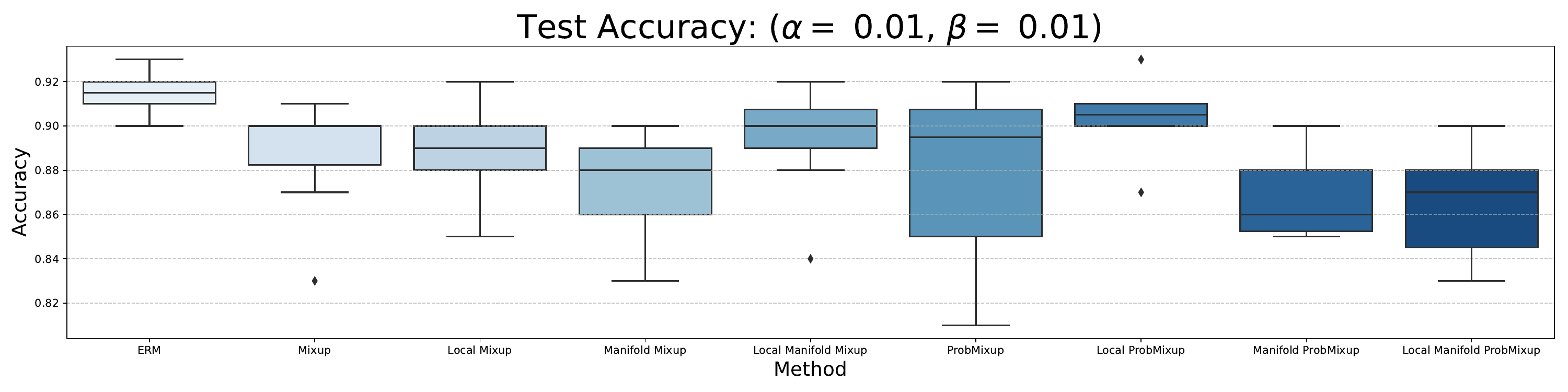}
        \caption{Average accuracy.}
        \label{subfig: toy_classification_alpha_0p01_beta_0p01_acc}
    \end{subfigure}
    
    \vspace{1em} 

    \begin{subfigure}{\textwidth}
        \centering
        \includegraphics[width=0.8\textwidth]{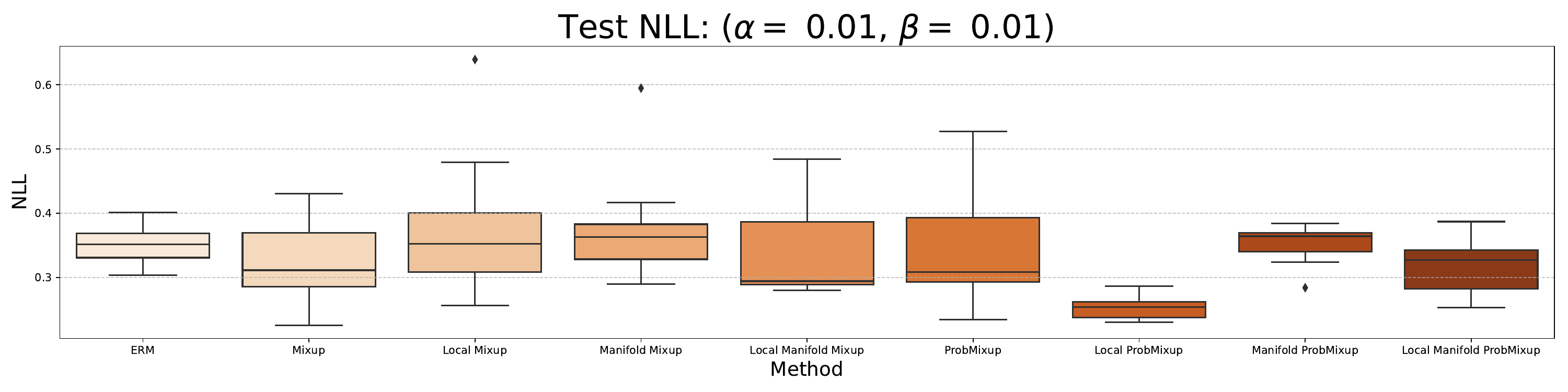}
        \caption{Average negative log-likelihood.}
        \label{subfig: toy_classification_alpha_0p01_beta_0p01_nll}
    \end{subfigure}
    
    \caption{Toy classification results for $\alpha=0.01$ and $\beta=0.01$.}
    \label{fig: toy_classification_alpha_0p01_beta_0p01}
\end{figure}

\begin{figure}[htbp]
    \centering
    \begin{subfigure}{\textwidth}
        \centering
        \includegraphics[width=0.8\textwidth]{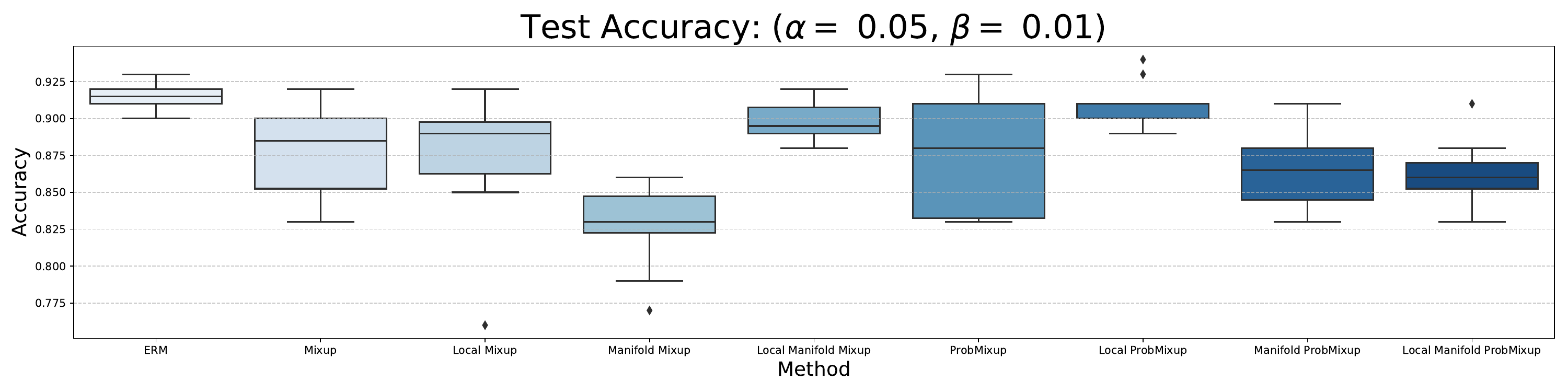}
        \caption{Average accuracy.}
        \label{subfig: toy_classification_alpha_0p05_beta_0p01_acc}
    \end{subfigure}
    
    \vspace{1em} 

    \begin{subfigure}{\textwidth}
        \centering
        \includegraphics[width=0.8\textwidth]{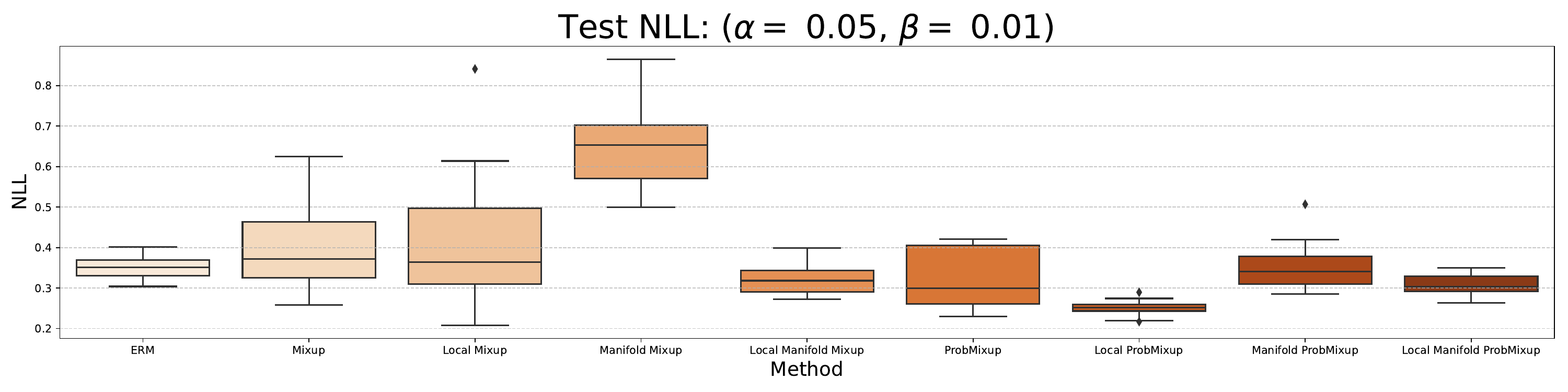}
        \caption{Average negative log-likelihood.}
        \label{subfig: toy_classification_alpha_0p05_beta_0p01_nll}
    \end{subfigure}
    
    \caption{Toy classification results for $\alpha=0.05$ and $\beta=0.01$.}
    \label{fig: toy_classification_alpha_0p05_beta_0p01}
\end{figure}

\begin{figure}[htbp]
    \centering
    \begin{subfigure}{\textwidth}
        \centering
        \includegraphics[width=0.8\textwidth]{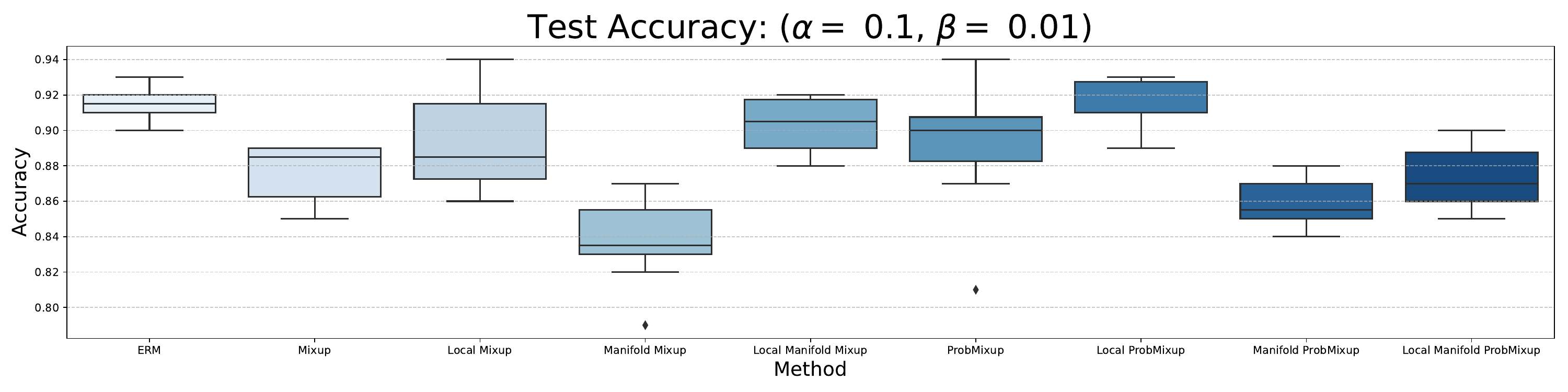}
        \caption{Average accuracy.}
        \label{subfig: toy_classification_alpha_0p1_beta_0p01_acc}
    \end{subfigure}
    
    \vspace{1em} 

    \begin{subfigure}{\textwidth}
        \centering
        \includegraphics[width=0.8\textwidth]{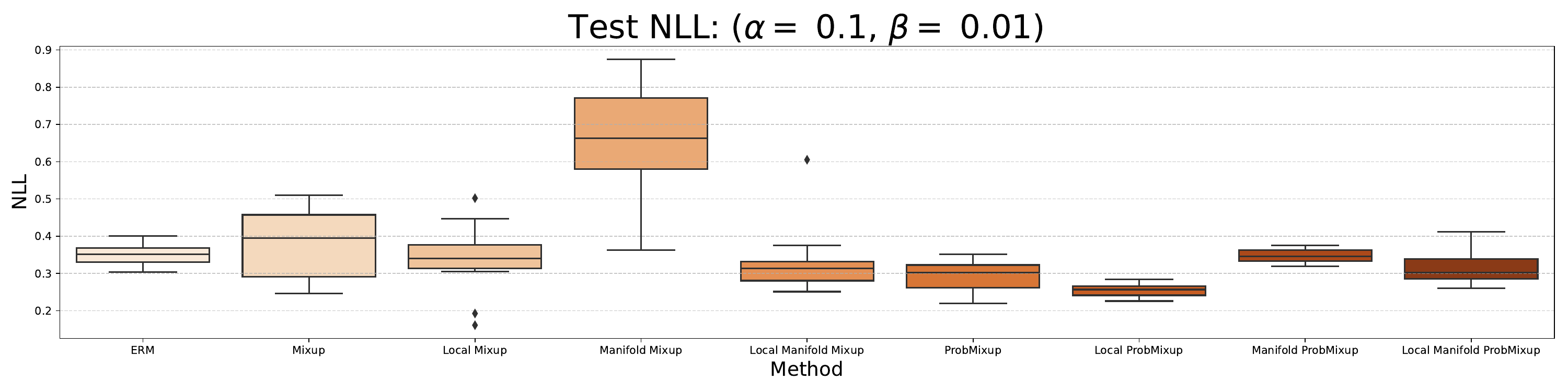}
        \caption{Average negative log-likelihood.}
        \label{subfig: toy_classification_alpha_0p1_beta_0p01_nll}
    \end{subfigure}
    
    \caption{Toy classification results for $\alpha=0.1$ and $\beta=0.01$.}
    \label{fig: toy_classification_alpha_0p1_beta_0p01}
\end{figure}

\begin{figure}[htbp]
    \centering
    \begin{subfigure}{\textwidth}
        \centering
        \includegraphics[width=0.8\textwidth]{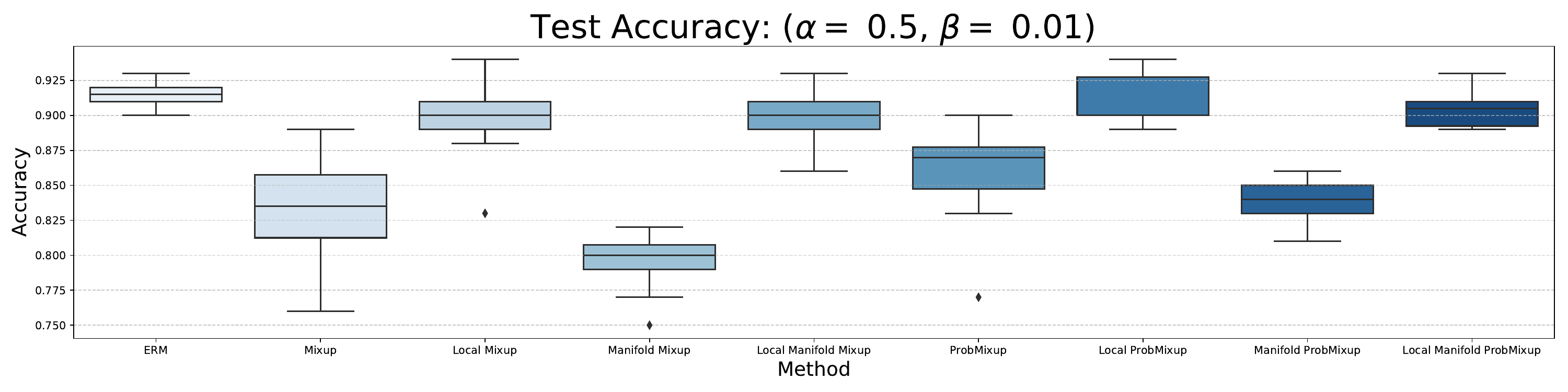}
        \caption{Average accuracy.}
        \label{subfig: toy_classification_alpha_0p5_beta_0p01_acc}
    \end{subfigure}
    
    \vspace{1em} 

    \begin{subfigure}{\textwidth}
        \centering
        \includegraphics[width=0.8\textwidth]{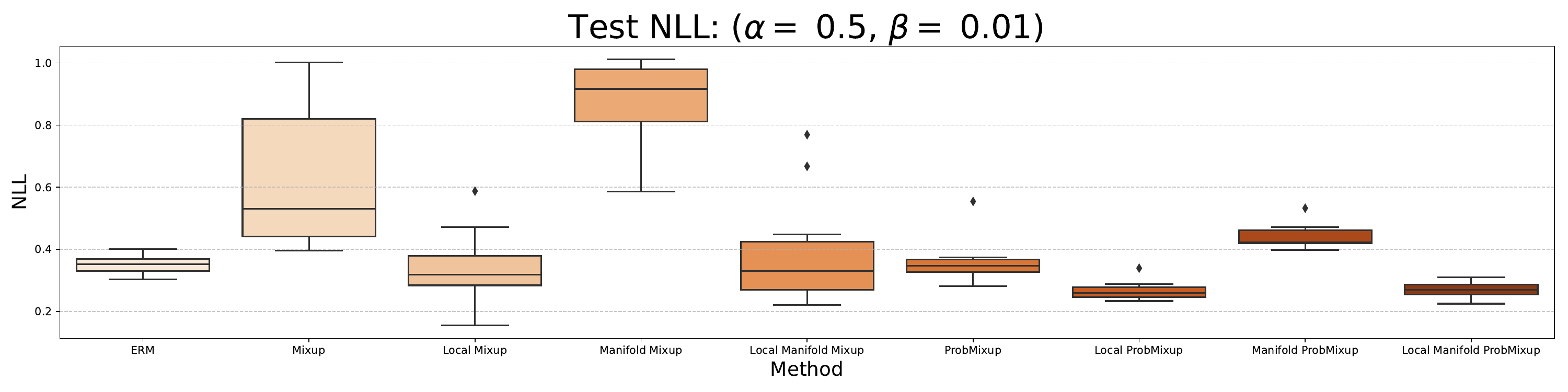}
        \caption{Average negative log-likelihood.}
        \label{subfig: toy_classification_alpha_0p5_beta_0p01_nll}
    \end{subfigure}
    
    \caption{Toy classification results for $\alpha=0.5$ and $\beta=0.01$.}
    \label{fig: toy_classification_alpha_0p5_beta_0p01}
\end{figure}

\begin{figure}[htbp]
    \centering
    \begin{subfigure}{\textwidth}
        \centering
        \includegraphics[width=0.8\textwidth]{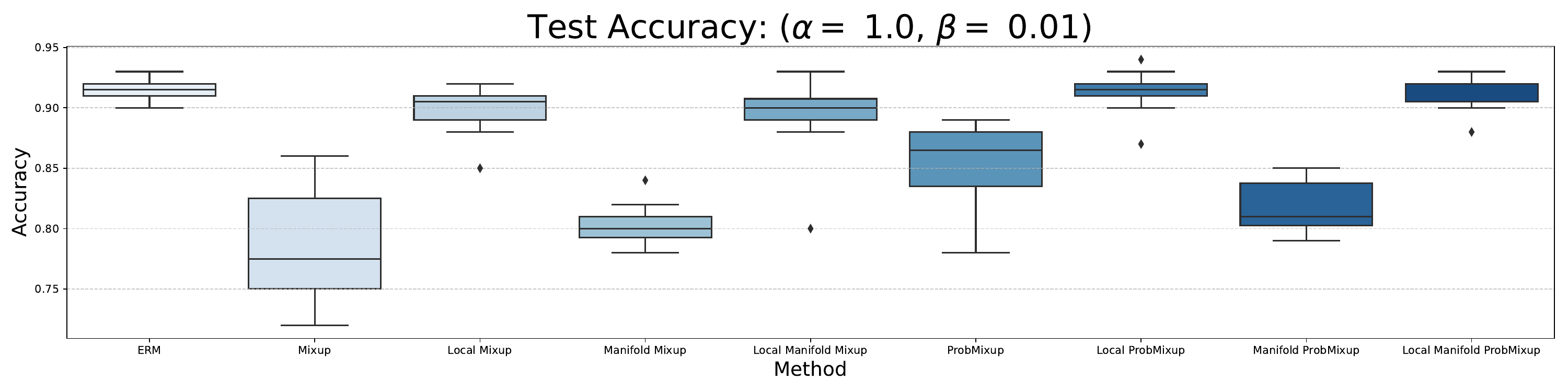}
        \caption{Average accuracy.}
        \label{subfig: toy_classification_alpha_1p0_beta_0p01_acc}
    \end{subfigure}
    
    \vspace{1em} 

    \begin{subfigure}{\textwidth}
        \centering
        \includegraphics[width=0.8\textwidth]{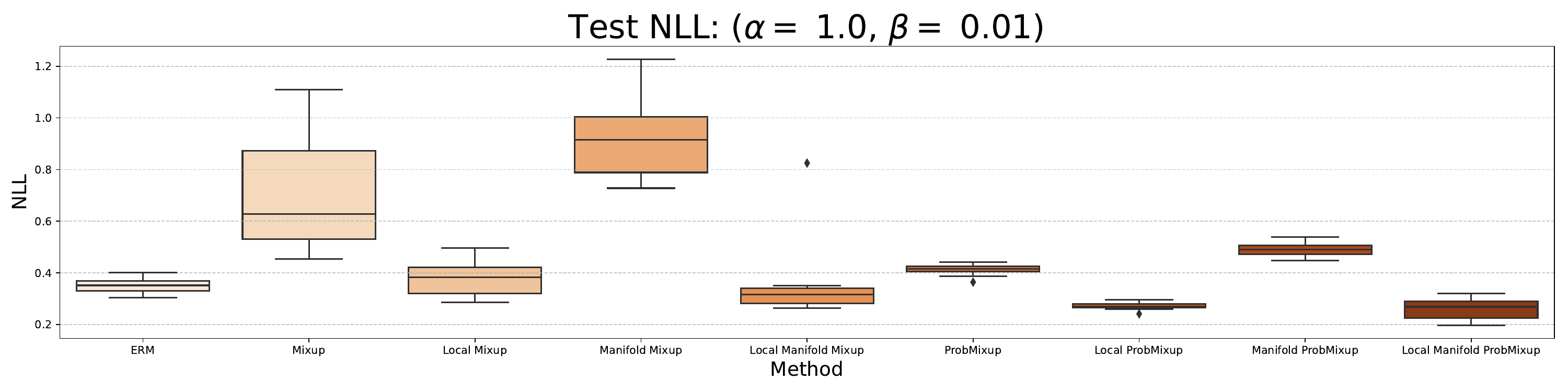}
        \caption{Average negative log-likelihood.}
        \label{subfig: toy_classification_alpha_1p0_beta_0p01_nll}
    \end{subfigure}
    
    \caption{Toy classification results for $\alpha=1$ and $\beta=0.01$.}
    \label{fig: toy_classification_alpha_1p0_beta_0p01}
\end{figure}

\subsection{UCI} \label{sec: app-UCI}

We include details of dataset and feature sizes below:

\begin{itemize}
    \item \textit{bostonHousing}: $n=506, d=13$;
    \item \textit{concrete}: $n=1030, d=8$;
    \item \textit{energy}: $n=768, d=8$;
    \item \textit{kin8nm}: $n=8192, d=8$;
    \item \textit{naval-propulsion-plant}: $n=11934, d=16$;
    \item \textit{power-plant}: $n=9568, d=4$;
    \item \textit{power-plant}: $n=1599, d=11$;
    \item \textit{yacht}: $n=308, d=6$.
\end{itemize}

\begin{table*}[!ht]
\centering
\resizebox{\textwidth}{!}{%
\begin{tabular}{@{}c|c|cccc|cccccc@{}}
\cmidrule{3-12}
\multicolumn{2}{c|}{} & \multicolumn{4}{c|}{Mixup Methods} & \multicolumn{6}{c}{ProbMixup Methods} \\ \midrule
Dataset & ERM & \textsf{Mix.} & \textsf{Loc$^\mathsf{K}$Mix.} & \textsf{M-Mix.} & Loc$^\mathsf{K}$M-Mix. & \textsf{ProbMix.} & \textsf{Loc$^\mathsf{K}$ProbMix.} & \textsf{M-ProbMix.} & \textsf{Loc$^\mathsf{K}$M-ProbMix.} & \textsf{M-ProbMix.}$^\star$ & \textsf{Loc$^\mathsf{K}$M-ProbMix.}$^\star$ \\ \midrule
bostonHousing & 5.77 $\pm$ 3.23 & 4.68 $\pm$ 3.00 & 5.25 $\pm$ 3.33 & 4.65 $\pm$ 2.67 & 4.72 $\pm$ 2.86 & 4.27 $\pm$ 2.75 & 5.00 $\pm$ 3.18 & 4.54 $\pm$ 2.84 & 3.93 $\pm$ 2.52 & 3.27 $\pm$ 0.97 & \textbf{3.23 $\pm$ 1.03} \\
energy & \textbf{0.44 $\pm$ 0.09} & 0.51 $\pm$ 0.10 & 0.45 $\pm$ 0.07 & 0.46 $\pm$ 0.08 & 0.45 $\pm$ 0.08 & 0.47 $\pm$ 0.09 & 0.48 $\pm$ 0.08 & 0.48 $\pm$ 0.07 & 0.49 $\pm$ 0.08 & 1.04 $\pm$ 0.46 & 0.76 $\pm$ 0.16 \\
wine-quality-red & 0.76 $\pm$ 0.10 & 0.72 $\pm$ 0.10 & 0.72 $\pm$ 0.10 & 0.77 $\pm$ 0.09 & 0.71 $\pm$ 0.09 & 0.75 $\pm$ 0.10 & 0.74 $\pm$ 0.10 & 0.77 $\pm$ 0.10 & 0.72 $\pm$ 0.10 & \textbf{0.66 $\pm$ 0.08} & \textbf{0.66 $\pm$ 0.08} \\
concrete & 5.88 $\pm$ 2.67 & 5.10 $\pm$ 0.58 & 5.19 $\pm$ 0.59 & 5.76 $\pm$ 2.73 & 5.73 $\pm$ 2.45 & 5.47 $\pm$ 2.76 & \textbf{5.08 $\pm$ 0.75} & 5.90 $\pm$ 2.50 & 5.88 $\pm$ 2.41 & 5.48 $\pm$ 0.67 & 5.43 $\pm$ 0.54 \\
power-plant & 3.91 $\pm$ 0.18 & 3.93 $\pm$ 0.17 & 3.96 $\pm$ 0.15 & 3.93 $\pm$ 0.17 & 3.97 $\pm$ 0.15 & \textbf{3.85 $\pm$ 0.14} & 3.96 $\pm$ 0.19 & 4.09 $\pm$ 0.22 & 4.13 $\pm$ 0.21 & 4.21 $\pm$ 0.20 & 4.24 $\pm$ 0.20 \\
yacht & 1.26 $\pm$ 2.71 & 1.96 $\pm$ 0.59 & 1.37 $\pm$ 0.37 & 0.70 $\pm$ 0.26 & 0.73 $\pm$ 0.30 & 0.92 $\pm$ 0.35 & 1.21 $\pm$ 0.58 & \textbf{0.68 $\pm$ 0.29} & 0.70 $\pm$ 0.31 & 1.41 $\pm$ 0.65 & 1.73 $\pm$ 0.74 \\
kin8nm$^\dagger$ & 7.41 $\pm$ 0.34 & 7.73 $\pm$ 0.34 & 7.53 $\pm$ 0.37 & \textbf{7.28 $\pm$ 0.26} & 7.62 $\pm$ 0.42 & 7.52 $\pm$ 0.36 & 7.54 $\pm$ 0.42 & 7.66 $\pm$ 0.52 & 7.73 $\pm$ 0.35 & 8.01 $\pm$ 0.43 & 8.12 $\pm$ 0.51 \\
naval-propulsion-plant$^\dagger$ & 0.21 $\pm$ 0.33 & 0.06 $\pm$ 0.01 & 0.09 $\pm$ 0.21 & 0.05 $\pm$ 0.03 & \textbf{0.04 $\pm$ 0.03} & \textbf{0.04 $\pm$ 0.07} & 0.09 $\pm$ 0.13 & 0.06 $\pm$ 0.02 & 0.17 $\pm$ 0.16 & 0.12 $\pm$ 0.01 & 0.13 $\pm$ 0.02 \\
\bottomrule
\end{tabular}
}
\caption{RMSE for UCI regression datasets. $^\dagger$: normalized to a single integer digit. $^\star$: implements separate variance networks, for details see text. In a majority of the datasets, probabilistic mixup methods outperform ERM and mixup methods in terms of RMSE.}
\label{tab: RMSE_uci_full}
\end{table*}

\subsection{Financial Forecasting} \label{app: stocks}

Here, we provide additional tables showing the performance of each method on the stock datasets in terms of average RMSE on the test set (please see Table \ref{tab: RMSE_stock_full}). We also provide the average NLL and average RMSE on the training set (please see Tables \ref{tab: NLL_stock_full_training} and \ref{tab: RMSE_stock_full_training}). Finally, to understand the performance of the regularization techniques across the GME and NVDA stocks, we provide plots of the conditional density estimates corresponding to the best and worst NLL on both the training and test set. In particular, Figure \ref{fig: gme_lstm_predictions} shows the conditional density estimates for the GME dataset using the LSTM predictor for the following regularization techniques: \textsf{Mix}, \textsf{M-Mix}, and \textsf{Loc$^\mathsf{K}$M-ProbMix}. Figure \ref{fig: gme_transformer_predictions} shows the conditional density estimates for the GME dataset using the Transformer predictor for the following regularization techniques: \textsf{Mix} and \textsf{M-ProbMix}. Finally, Figure \ref{fig: nvda_lstm_predictions} shows the conditional density estimates for the NVDA dataset using the LSTM predictor for the following techniques: \textsf{ERM}, \textsf{ProbMix}, and \textsf{M-ProbMix}. All in all, we find that methods that perform well in terms of NLL tend to have more conservative estimates of the variance of the conditional density, while models that performed poorly had both inaccurate means and variances.

\begin{table}[!ht]
\centering
\resizebox{\columnwidth}{!}{%
\begin{tabular}{@{}clllllllll@{}}
\toprule
\multicolumn{2}{c}{LSTM}                                            & \multicolumn{4}{c}{Mixup Methods}                                                                                                                                              & \multicolumn{4}{c}{ProbMixup Methods}                                                                                                                                                         \\ \midrule
\multicolumn{1}{c|}{Dataset} & \multicolumn{1}{c|}{\textsf{ERM}}    & \multicolumn{1}{c}{\textsf{Mix}} & \multicolumn{1}{c}{\textsf{Loc$^\mathsf{K}$Mix}} & \multicolumn{1}{c}{\textsf{M-Mix}} & \multicolumn{1}{c|}{\textsf{Loc$^\mathsf{K}$M-Mix}} & \multicolumn{1}{c}{\textsf{ProbMix}} & \multicolumn{1}{c}{\textsf{Loc$^\mathsf{K}$ProbMix}} & \multicolumn{1}{c}{\textsf{M-ProbMix}} & \multicolumn{1}{c}{\textsf{Loc$^\mathsf{K}$M-ProbMix}} \\ \midrule
\multicolumn{1}{c|}{GME}     & \multicolumn{1}{l|}{$1.64 \pm 0.02$} & $\mathbf{1.48 \pm 0.02}$                  & $1.60 \pm 0.02$                                  & $1.65 \pm 0.01$                    & \multicolumn{1}{l|}{$1.62 \pm 0.03$}                & $1.71 \pm 0.03$                      & $1.65 \pm 0.04$                                      & $1.57 \pm 0.01$                        & $1.57 \pm 0.02$                                        \\
\multicolumn{1}{c|}{GOOG}    & \multicolumn{1}{l|}{$2.06 \pm 0.10$} & $2.13 \pm 0.05$                  & $1.87 \pm 0.11$                                  & $2.12 \pm 0.09$                    & \multicolumn{1}{l|}{$1.98 \pm 0.13$}                & $1.72 \pm 0.16$                      & $\mathbf{1.68 \pm 0.12}$                                      & $1.76 \pm 0.10$                        & $1.81 \pm 0.08$                                        \\
\multicolumn{1}{c|}{NVDA}    & \multicolumn{1}{l|}{$0.25 \pm 0.02$} & $0.33 \pm 0.04$                  & $0.42 \pm 0.04$                                  & $0.31 \pm 0.03$                    & \multicolumn{1}{l|}{$0.44 \pm 0.04$}                & $\mathbf{0.24 \pm 0.02}$                      & $0.39 \pm 0.04$                                      & $0.55 \pm 0.02$                        & $0.40 \pm 0.01$                                        \\
\multicolumn{1}{c|}{RCL}     & \multicolumn{1}{l|}{$0.59 \pm 0.02$} & $\mathbf{0.31 \pm 0.01}$                  & $0.58 \pm 0.02$                                  & $0.65 \pm 0.02$                    & \multicolumn{1}{l|}{$0.54 \pm 0.01$}                & $0.62 \pm 0.01$                      & $0.57 \pm 0.02$                                      & $0.49 \pm 0.02$                        & $0.59 \pm 0.02$                                        \\ \midrule
\multicolumn{2}{c|}{Transformer}                                    & \multicolumn{4}{c|}{Mixup Methods}                                                                                                                                             & \multicolumn{4}{c}{ProbMixup Methods}                                                                                                                                                         \\ \midrule
\multicolumn{1}{c|}{Dataset} & \multicolumn{1}{c|}{\textsf{ERM}}    & \multicolumn{1}{c}{\textsf{Mix}} & \multicolumn{1}{c}{\textsf{Loc$^\mathsf{K}$Mix}} & \multicolumn{1}{c}{\textsf{M-Mix}} & \multicolumn{1}{c|}{\textsf{Loc$^\mathsf{K}$M-Mix}} & \multicolumn{1}{c}{\textsf{ProbMix}} & \multicolumn{1}{c}{\textsf{Loc$^\mathsf{K}$ProbMix}} & \multicolumn{1}{c}{\textsf{M-ProbMix}} & \multicolumn{1}{c}{\textsf{Loc$^\mathsf{K}$M-ProbMix}} \\ \midrule
\multicolumn{1}{c|}{GME}     & \multicolumn{1}{l|}{$1.82 \pm 0.03$} & $1.86 \pm 0.01$                  & $1.85 \pm 0.01$                                  & $1.86 \pm 0.01$                    & \multicolumn{1}{l|}{$1.79 \pm 0.02$}                & $\mathbf{1.78 \pm 0.02}$                      & $1.84 \pm 0.01$                                      & $1.81 \pm 0.01$                        & $1.82 \pm 0.02$                                        \\
\multicolumn{1}{c|}{GOOG}    & \multicolumn{1}{l|}{$2.24 \pm 0.06$} & $2.46 \pm 0.03$                  & $\mathbf{1.68 \pm 0.10}$                                  & $1.99 \pm 0.09$                    & \multicolumn{1}{l|}{$2.06 \pm 0.07$}                & $2.14 \pm 0.10$                      & $1.79 \pm 0.09$                                      & $1.83 \pm 0.13$                        & $2.00 \pm 0.12$                                        \\
\multicolumn{1}{c|}{NVDA}    & \multicolumn{1}{l|}{$0.36 \pm 0.02$} & $\mathbf{0.28 \pm 0.01}$                  & $0.42 \pm 0.02$                                  & $0.36 \pm 0.01$                    & \multicolumn{1}{l|}{$0.35 \pm 0.01$}                & $0.29 \pm 0.01$                      & $0.36 \pm 0.01$                                      & $0.35 \pm 0.02$                        & $0.37 \pm 0.01$                                        \\
\multicolumn{1}{c|}{RCL}     & \multicolumn{1}{l|}{$0.43 \pm 0.01$} & $\mathbf{0.23 \pm 0.01}$                  & $0.45 \pm 0.01$                                  & $0.42 \pm 0.01$                    & \multicolumn{1}{l|}{$0.44 \pm 0.01$}                & $0.42 \pm 0.01$                      & $0.41 \pm 0.01$                                      & $0.40 \pm 0.01$                        & $0.41 \pm 0.01$                                        \\ \bottomrule
\end{tabular}

}
\caption{RMSE for LSTM and Transformer models in the time series datasets. Results show that probabilistic mixup techniques do not lead to improvement in performance in terms of RMSE, as \textsf{ERM} and \textsf{Mix} outperform it on most of the datasets. This implies that while \textsf{ProbMix} may offer benefits in terms of calibrated uncertainty, its performance on uncertainty agnostic metrics like RMSE may not be up to par to ERM and other regularization schemes.}
    \label{tab: RMSE_stock_full}
\end{table}

\begin{table}[!ht]
\centering
\resizebox{\columnwidth}{!}{%
\begin{tabular}{@{}clllllllll@{}}
\toprule
\multicolumn{2}{c}{LSTM}                                             & \multicolumn{4}{c}{Mixup Methods}                                                                                                                                              & \multicolumn{4}{c}{ProbMixup Methods}                                                                                                                                                         \\ \midrule
\multicolumn{1}{c|}{Dataset} & \multicolumn{1}{c|}{\textsf{ERM}}     & \multicolumn{1}{c}{\textsf{Mix}} & \multicolumn{1}{c}{\textsf{Loc$^\mathsf{K}$Mix}} & \multicolumn{1}{c}{\textsf{M-Mix}} & \multicolumn{1}{c|}{\textsf{Loc$^\mathsf{K}$M-Mix}} & \multicolumn{1}{c}{\textsf{ProbMix}} & \multicolumn{1}{c}{\textsf{Loc$^\mathsf{K}$ProbMix}} & \multicolumn{1}{c}{\textsf{M-ProbMix}} & \multicolumn{1}{c}{\textsf{Loc$^\mathsf{K}$M-ProbMix}} \\ \midrule
\multicolumn{1}{c|}{GME}     & \multicolumn{1}{l|}{$-1.39 \pm 0.06$} & $-1.25 \pm 0.03$                 & $-1.74 \pm 0.03$                                 & $-1.42 \pm 0.03$                   & \multicolumn{1}{l|}{$\mathbf{-1.79 \pm 0.02}$}               & $-1.46 \pm 0.03$                     & $-1.70 \pm 0.04$                                     & $-0.94 \pm 0.03$                       & $-1.26 \pm 0.06$                                       \\
\multicolumn{1}{c|}{GOOG}    & \multicolumn{1}{l|}{$-0.80 \pm 0.06$} & $-0.81 \pm 0.02$                 & $\mathbf{-1.19 \pm 0.03}$                                 & $-0.97 \pm 0.04$                   & \multicolumn{1}{l|}{$-0.91 \pm 0.04$}               & $-0.88 \pm 0.04$                     & $-1.02 \pm 0.03$                                     & $-0.51 \pm 0.03$                       & $-0.91 \pm 0.04$                                       \\
\multicolumn{1}{c|}{NVDA}    & \multicolumn{1}{l|}{$-1.56 \pm 0.04$} & $-1.23 \pm 0.04$                 & $-1.76 \pm 0.07$                                 & $-1.57 \pm 0.03$                   & \multicolumn{1}{l|}{$-1.84 \pm 0.03$}               & $-1.50 \pm 0.02$                     & $\mathbf{-1.92 \pm 0.07}$                                     & $-0.94 \pm 0.03$                       & $-1.30 \pm 0.03$                                       \\
\multicolumn{1}{c|}{RCL}     & \multicolumn{1}{l|}{$-1.29 \pm 0.03$} & $-0.98 \pm 0.01$                 & $\mathbf{-1.54 \pm 0.01}$                                 & $-1.31 \pm 0.02$                   & \multicolumn{1}{l|}{$-1.47 \pm 0.03$}               & $-1.11 \pm 0.01$                     & $-1.52 \pm 0.03$                                     & $-0.56 \pm 0.06$                       & $-1.08 \pm 0.04$                                       \\ \midrule
\multicolumn{2}{c|}{Transformer}                                     & \multicolumn{4}{c|}{Mixup Methods}                                                                                                                                             & \multicolumn{4}{c}{ProbMixup Methods}                                                                                                                                                         \\ \midrule
\multicolumn{1}{c|}{Dataset} & \multicolumn{1}{c|}{\textsf{ERM}}     & \multicolumn{1}{c}{\textsf{Mix}} & \multicolumn{1}{c}{\textsf{Loc$^\mathsf{K}$Mix}} & \multicolumn{1}{c}{\textsf{M-Mix}} & \multicolumn{1}{c|}{\textsf{Loc$^\mathsf{K}$M-Mix}} & \multicolumn{1}{c}{\textsf{ProbMix}} & \multicolumn{1}{c}{\textsf{Loc$^\mathsf{K}$ProbMix}} & \multicolumn{1}{c}{\textsf{M-ProbMix}} & \multicolumn{1}{c}{\textsf{Loc$^\mathsf{K}$M-ProbMix}} \\ \midrule
\multicolumn{1}{c|}{GME}     & \multicolumn{1}{l|}{$-1.41 \pm 0.02$} & $-1.28 \pm 0.02$                 & $\mathbf{-1.59 \pm 0.01}$                                 & $-1.44 \pm 0.02$                   & \multicolumn{1}{l|}{$-1.51 \pm 0.03$}               & $-1.34 \pm 0.03$                     & $-1.52 \pm 0.03$                                     & $-1.42 \pm 0.02$                       & $-1.49 \pm 0.01$                                       \\
\multicolumn{1}{c|}{GOOG}    & \multicolumn{1}{l|}{$-0.99 \pm 0.04$} & $-0.85 \pm 0.02$                 & $-1.07 \pm 0.02$                                 & $-1.11 \pm 0.02$                   & \multicolumn{1}{l|}{$-1.05 \pm 0.01$}               & $-0.96 \pm 0.03$                     & $-1.11 \pm 0.01$                                     & $-1.08 \pm 0.01$                       & $\mathbf{-1.15 \pm 0.01}$                                       \\
\multicolumn{1}{c|}{NVDA}    & \multicolumn{1}{l|}{$-1.62 \pm 0.01$} & $-1.52 \pm 0.02$                 & $-1.72 \pm 0.01$                                 & $-1.47 \pm 0.04$                   & \multicolumn{1}{l|}{$-1.69 \pm 0.02$}               & $-1.48 \pm 0.03$                     & $\mathbf{-1.73 \pm 0.03}$                                     & $-1.49 \pm 0.05$                       & $-1.71 \pm 0.01$                                       \\
\multicolumn{1}{c|}{RCL}     & \multicolumn{1}{l|}{$-1.24 \pm 0.02$} & $-0.88 \pm 0.01$                 & $\mathbf{-1.39 \pm 0.03}$                                 & $-1.20 \pm 0.01$                   & \multicolumn{1}{l|}{$-1.31 \pm 0.02$}               & $-1.00 \pm 0.02$                     & $-1.32 \pm 0.02$                                     & $-1.19 \pm 0.01$                       & $-1.26 \pm 0.02$                                       \\ \bottomrule
\end{tabular}
}
\caption{NLL for LSTM and Transformer models in the time series datasets (in-sample performance).}
    \label{tab: NLL_stock_full_training}
\end{table}

\begin{table}[!ht]
\centering
\resizebox{\columnwidth}{!}{%
\begin{tabular}{@{}clllllllll@{}}
\toprule
\multicolumn{2}{c}{LSTM}                                            & \multicolumn{4}{c}{Mixup Methods}                                                                                                                                              & \multicolumn{4}{c}{ProbMixup Methods}                                                                                                                                                         \\ \midrule
\multicolumn{1}{c|}{Dataset} & \multicolumn{1}{c|}{\textsf{ERM}}    & \multicolumn{1}{c}{\textsf{Mix}} & \multicolumn{1}{c}{\textsf{Loc$^\mathsf{K}$Mix}} & \multicolumn{1}{c}{\textsf{M-Mix}} & \multicolumn{1}{c|}{\textsf{Loc$^\mathsf{K}$M-Mix}} & \multicolumn{1}{c}{\textsf{ProbMix}} & \multicolumn{1}{c}{\textsf{Loc$^\mathsf{K}$ProbMix}} & \multicolumn{1}{c}{\textsf{M-ProbMix}} & \multicolumn{1}{c}{\textsf{Loc$^\mathsf{K}$M-ProbMix}} \\ \midrule
\multicolumn{1}{c|}{GME}     & \multicolumn{1}{l|}{$0.09 \pm 0.01$} & $0.08 \pm 0.00$                  & $0.07 \pm 0.00$                                  & $0.09 \pm 0.01$                    & \multicolumn{1}{l|}{$\mathbf{0.05 \pm 0.00}$}                & $0.07 \pm 0.00$                      & $0.06 \pm 0.00$                                      & $0.12 \pm 0.00$                        & $0.08 \pm 0.01$                                        \\
\multicolumn{1}{c|}{GOOG}    & \multicolumn{1}{l|}{$0.18 \pm 0.01$} & $0.11 \pm 0.00$                  & $\mathbf{0.10 \pm 0.00}$                                  & $0.19 \pm 0.02$                    & \multicolumn{1}{l|}{$0.13 \pm 0.00$}                & $0.16 \pm 0.02$                      & $0.11 \pm 0.00$                                      & $0.31 \pm 0.02$                        & $0.22 \pm 0.02$                                        \\
\multicolumn{1}{c|}{NVDA}    & \multicolumn{1}{l|}{$0.07 \pm 0.00$} & $0.10 \pm 0.01$                  & $0.06 \pm 0.00$                                  & $0.06 \pm 0.00$                    & \multicolumn{1}{l|}{$\mathbf{0.05 \pm 0.00}$}                & $0.06 \pm 0.00$                      & $\mathbf{0.05 \pm 0.00}$                                      & $0.10 \pm 0.00$                        & $0.08 \pm 0.00$                                        \\
\multicolumn{1}{c|}{RCL}     & \multicolumn{1}{l|}{$0.09 \pm 0.00$} & $0.09 \pm 0.00$                  & $\mathbf{0.06 \pm 0.00}$                                  & $0.08 \pm 0.00$                    & \multicolumn{1}{l|}{$0.07 \pm 0.00$}                & $0.09 \pm 0.00$                      & $0.07 \pm 0.00$                                      & $0.15 \pm 0.01$                        & $0.10 \pm 0.00$                                        \\ \midrule
\multicolumn{2}{c|}{Transformer}                                    & \multicolumn{4}{c|}{Mixup Methods}                                                                                                                                             & \multicolumn{4}{c}{ProbMixup Methods}                                                                                                                                                         \\ \midrule
\multicolumn{1}{c|}{Dataset} & \multicolumn{1}{c|}{\textsf{ERM}}    & \multicolumn{1}{c}{\textsf{Mix}} & \multicolumn{1}{c}{\textsf{Loc$^\mathsf{K}$Mix}} & \multicolumn{1}{c}{\textsf{M-Mix}} & \multicolumn{1}{c|}{\textsf{Loc$^\mathsf{K}$M-Mix}} & \multicolumn{1}{c}{\textsf{ProbMix}} & \multicolumn{1}{c}{\textsf{Loc$^\mathsf{K}$ProbMix}} & \multicolumn{1}{c}{\textsf{M-ProbMix}} & \multicolumn{1}{c}{\textsf{Loc$^\mathsf{K}$M-ProbMix}} \\ \midrule
\multicolumn{1}{c|}{GME}     & \multicolumn{1}{l|}{$0.08 \pm 0.00$} & $0.08 \pm 0.00$                  & $\mathbf{0.07 \pm 0.00}$                                  & $\mathbf{0.07 \pm 0.00}$                    & \multicolumn{1}{l|}{$\mathbf{0.07 \pm 0.00}$}                & $\mathbf{0.07 \pm 0.00}$                      & $\mathbf{0.07 \pm 0.00}$                                      & $\mathbf{0.07 \pm 0.00}$                        & $\mathbf{0.07 \pm 0.00}$                                        \\
\multicolumn{1}{c|}{GOOG}    & \multicolumn{1}{l|}{$0.15 \pm 0.02$} & $0.12 \pm 0.00$                  & $0.14 \pm 0.02$                                  & $\mathbf{0.10 \pm 0.00}$                    & \multicolumn{1}{l|}{$0.16 \pm 0.01$}                & $0.14 \pm 0.02$                      & $0.12 \pm 0.01$                                      & $\mathbf{0.10 \pm 0.00}$                        & $0.11 \pm 0.00$                                        \\
\multicolumn{1}{c|}{NVDA}    & \multicolumn{1}{l|}{$0.06 \pm 0.00$} & $0.06 \pm 0.00$                  & $\mathbf{0.05 \pm 0.00}$                                  & $0.07 \pm 0.00$                    & \multicolumn{1}{l|}{$\mathbf{0.05 \pm 0.00}$}                & $0.06 \pm 0.00$                      & $0.06 \pm 0.00$                                      & $0.06 \pm 0.00$                        & $\mathbf{0.05 \pm 0.00}$                                        \\
\multicolumn{1}{c|}{RCL}     & \multicolumn{1}{l|}{$0.10 \pm 0.00$} & $0.12 \pm 0.00$                  & $\mathbf{0.08 \pm 0.00}$                                  & $0.10 \pm 0.00$                    & \multicolumn{1}{l|}{$0.09 \pm 0.00$}                & $0.10 \pm 0.00$                      & $0.09 \pm 0.00$                                      & $0.10 \pm 0.00$                        & $0.10 \pm 0.00$                                        \\ \bottomrule
\end{tabular}
}
\caption{RMSE for LSTM and Transformer models in the time series datasets (in-sample performance).}
    \label{tab: RMSE_stock_full_training}
\end{table}

\begin{figure}[htbp]
    \centering
    \begin{subfigure}{\textwidth}
        \centering
        \includegraphics[width=0.5\textwidth]{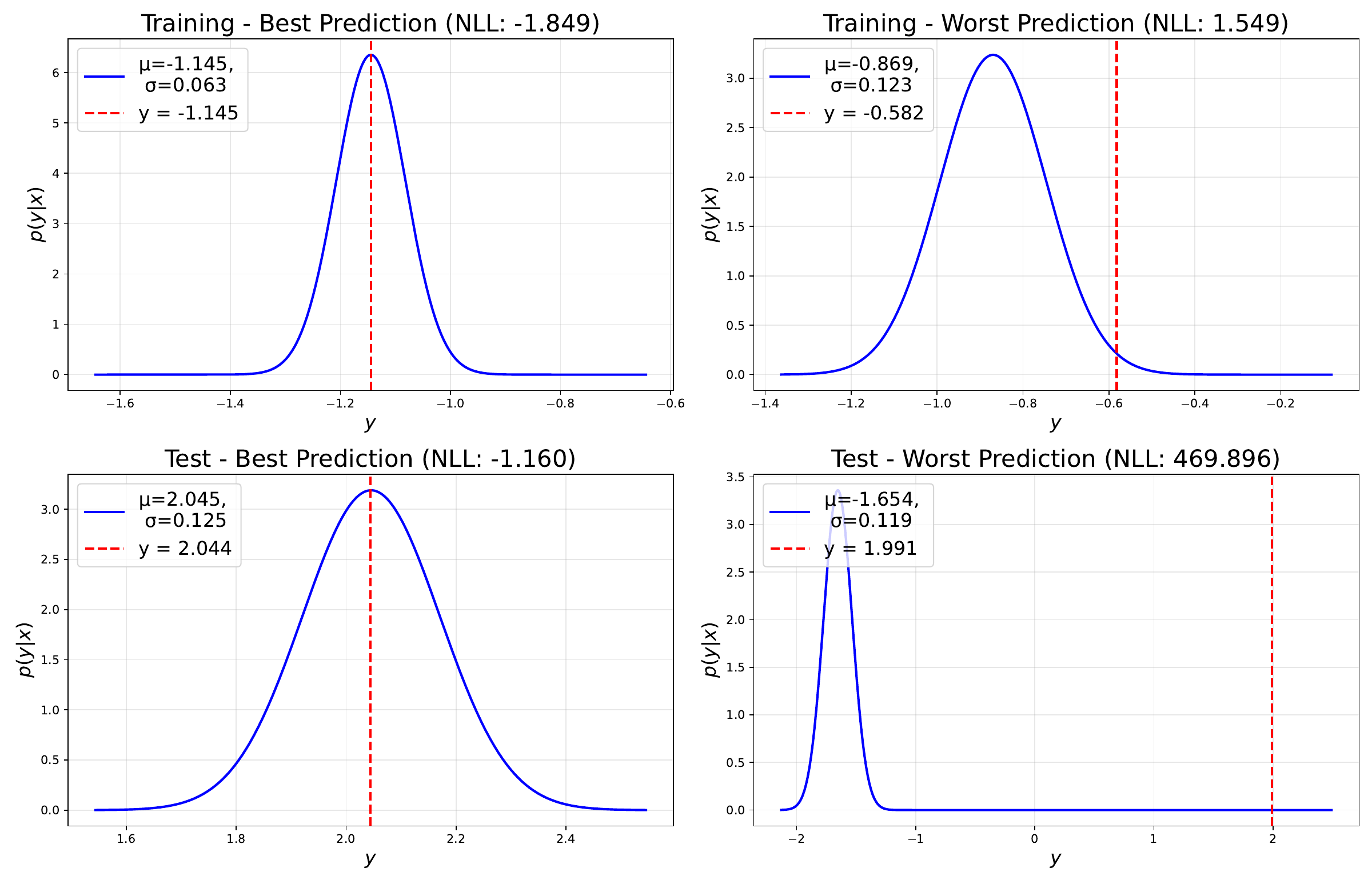}
        \caption{\textsf{Mix}.}
        \label{subfig: gme_lstm_mixup}
    \end{subfigure}
    
    \vspace{1em} 

    \begin{subfigure}{\textwidth}
        \centering
        \includegraphics[width=0.5\textwidth]{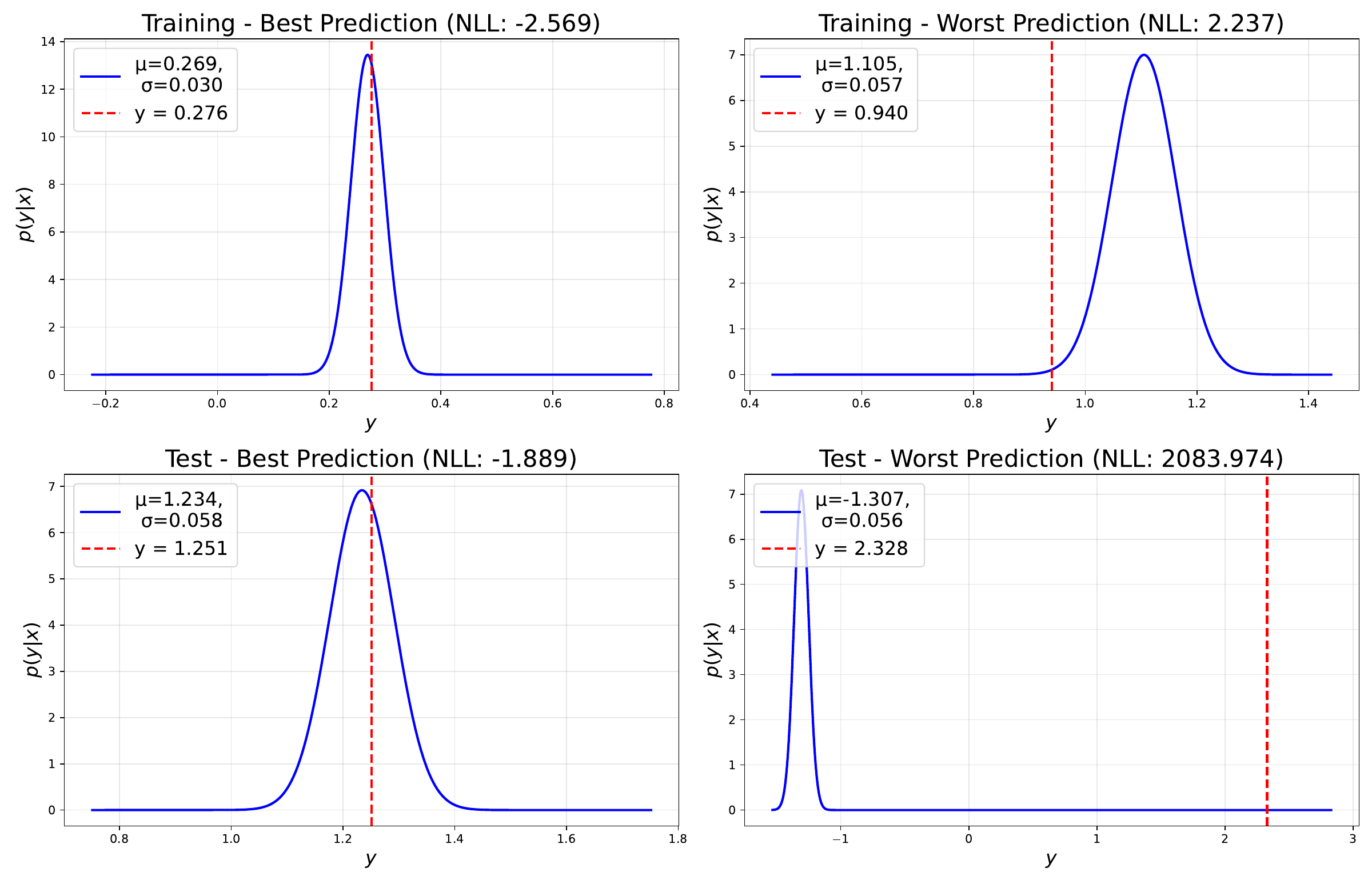}
        \caption{\textsf{M-Mix}.}
        \label{subfig: gme_lstm_manifold_mixup}
    \end{subfigure}

    \vspace{1em} 

    \begin{subfigure}{\textwidth}
        \centering
        \includegraphics[width=0.5\textwidth]{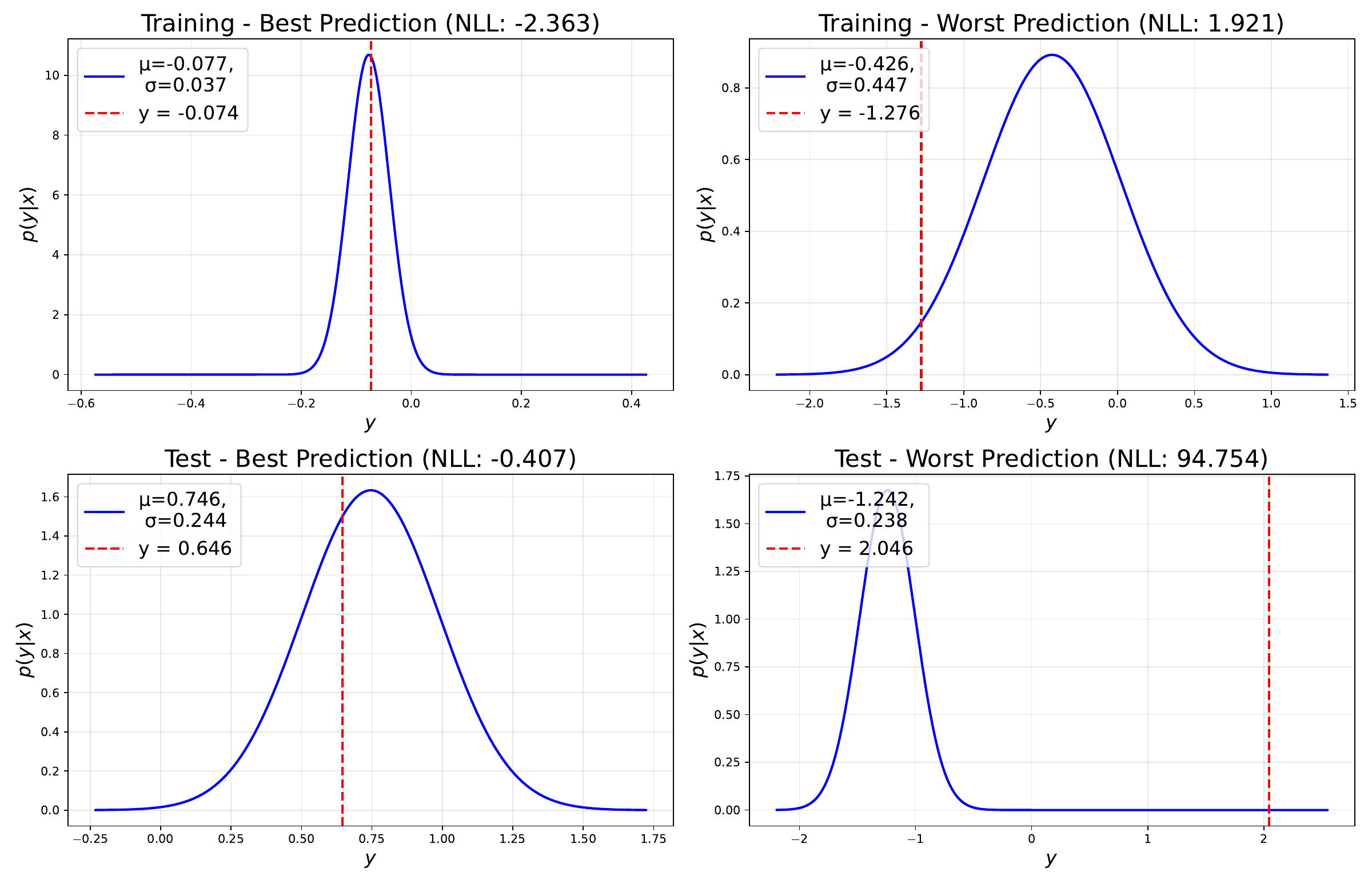}
        \caption{\textsf{Loc$^\mathsf{K}$M-ProbMix}.}
        \label{subfig: gme_lstm_local_manifold_probmix}
    \end{subfigure}
    
    \caption{Predicted conditional densities with best and worst NLL values for LSTM model on GME dataset}
    \label{fig: gme_lstm_predictions}
\end{figure}

\begin{figure}[htbp]
    \centering
    \begin{subfigure}{\textwidth}
        \centering
        \includegraphics[width=0.5\textwidth]{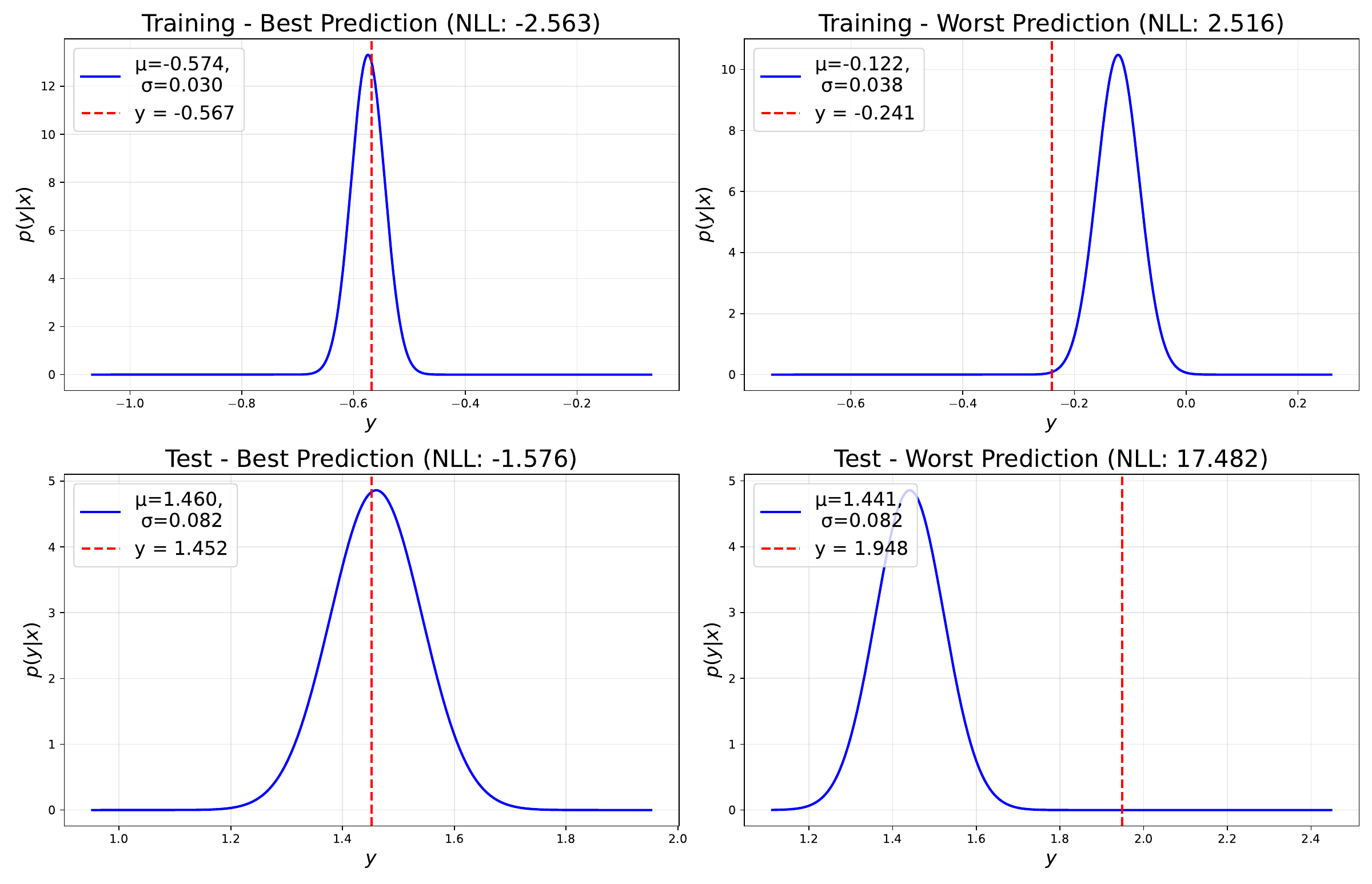}
        \caption{\textsf{ERM}.}
        \label{subfig: nvda_lstm_erm}
    \end{subfigure}
    
    \vspace{1em} 

    \begin{subfigure}{\textwidth}
        \centering
        \includegraphics[width=0.5\textwidth]{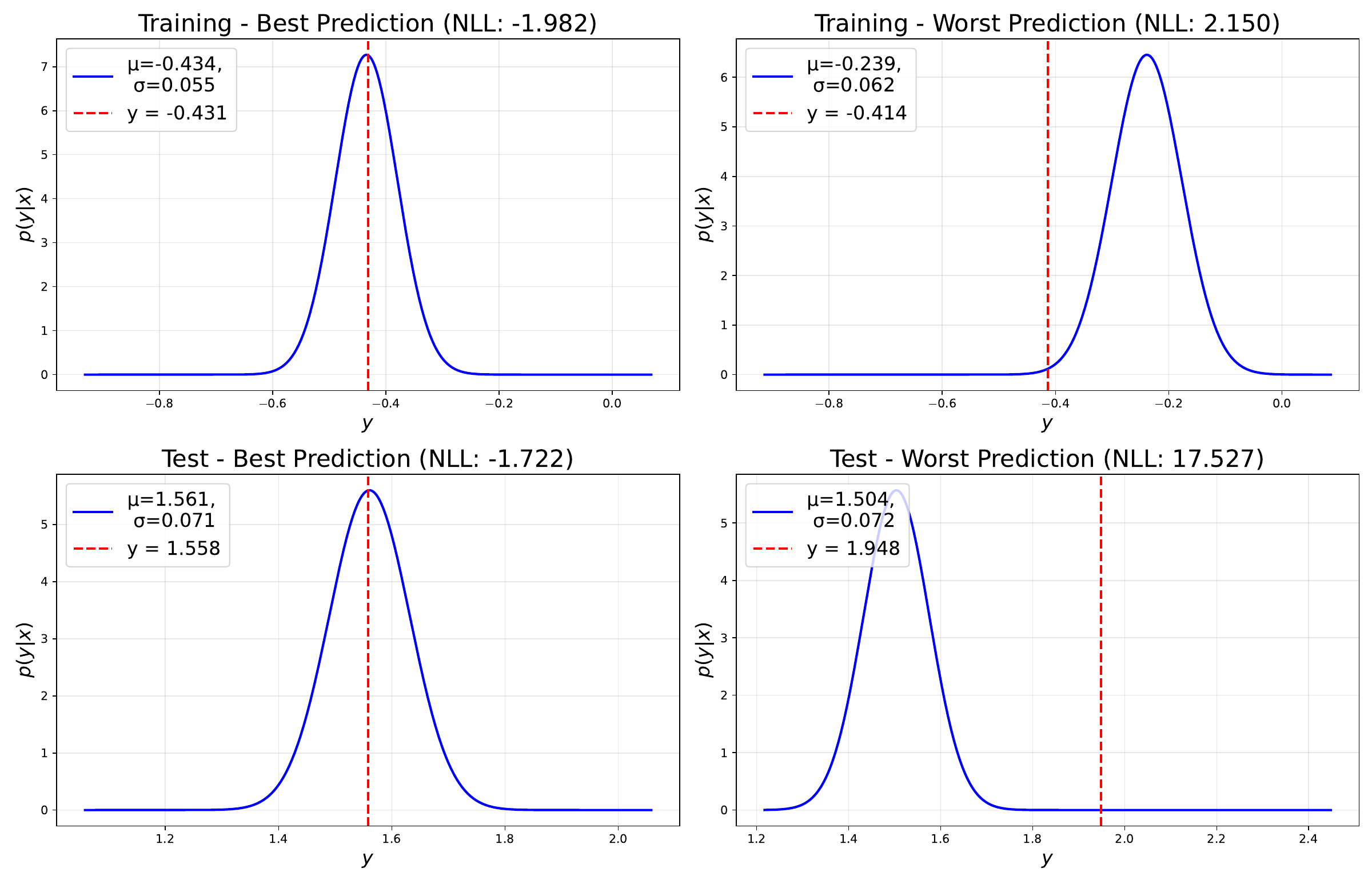}
        \caption{\textsf{ProbMix}.}
        \label{subfig: nvda_lstm_probmix}
    \end{subfigure}

    \vspace{1em} 

    \begin{subfigure}{\textwidth}
        \centering
        \includegraphics[width=0.5\textwidth]{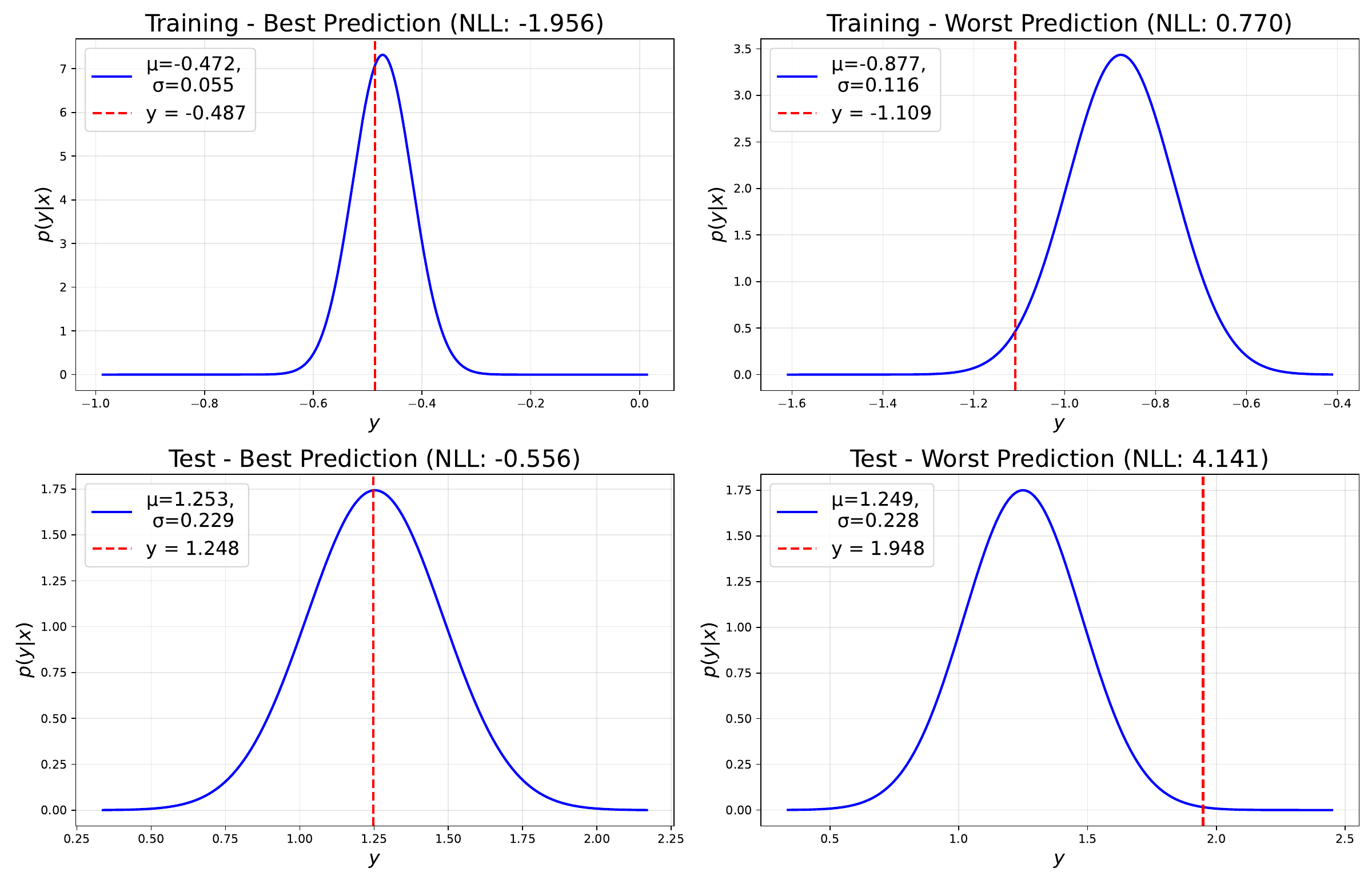}
        \caption{Manifold probabilistic mixup.}
        \label{subfig: nvda_lstm_manifold_probmix}
    \end{subfigure}
    
    \caption{Predicted conditional densities with best and worst NLL values for LSTM model on NVDA dataset}
    \label{fig: nvda_lstm_predictions}
\end{figure}

\begin{figure}[htbp]
    \centering
    \begin{subfigure}{\textwidth}
        \centering
        \includegraphics[width=0.5\textwidth]{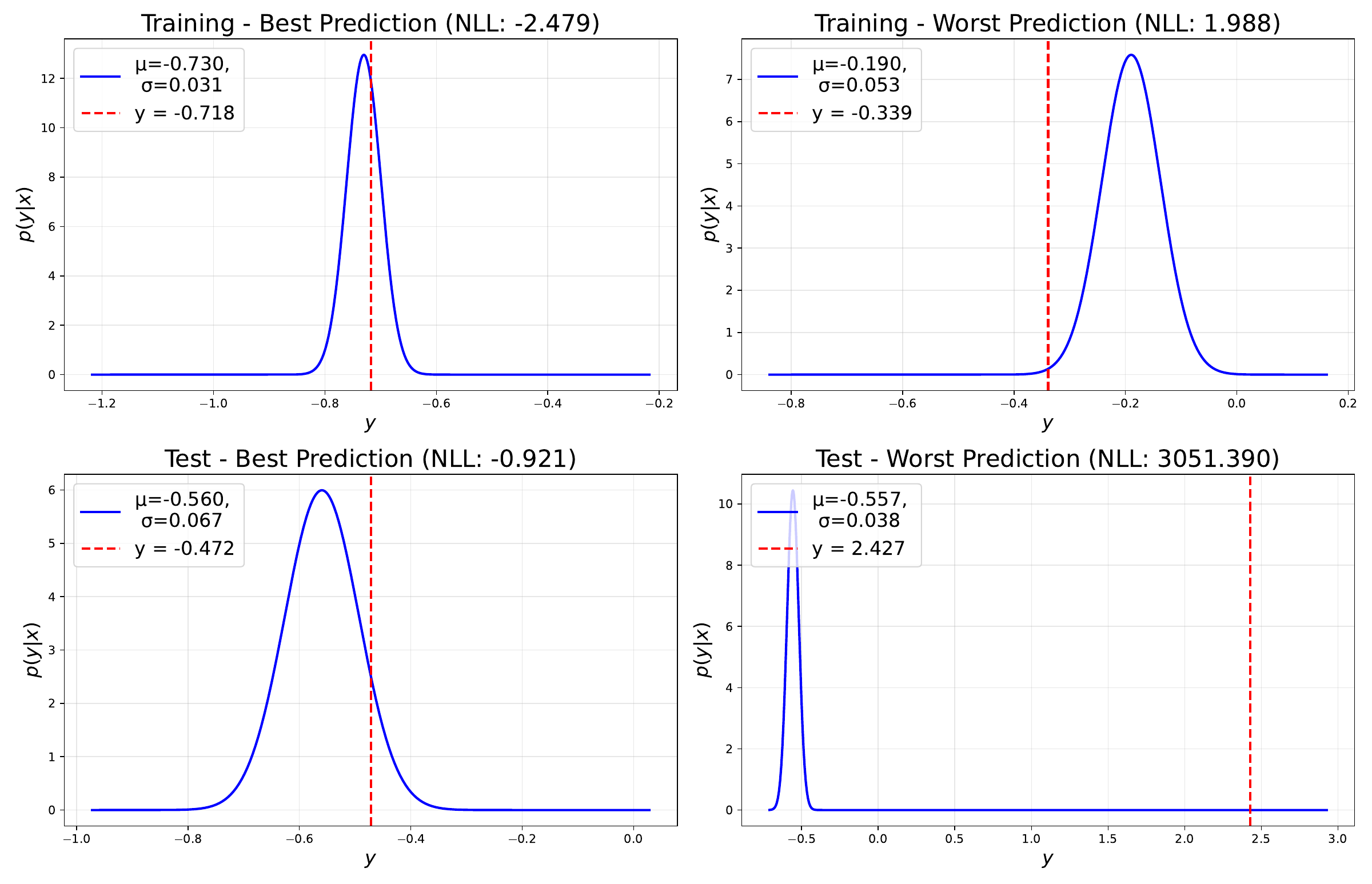}
        \caption{\textsf{Mix}.}
        \label{subfig: gme_transformer_mixup}
    \end{subfigure}
    
    \vspace{1em} 

    \begin{subfigure}{\textwidth}
        \centering
        \includegraphics[width=0.5\textwidth]{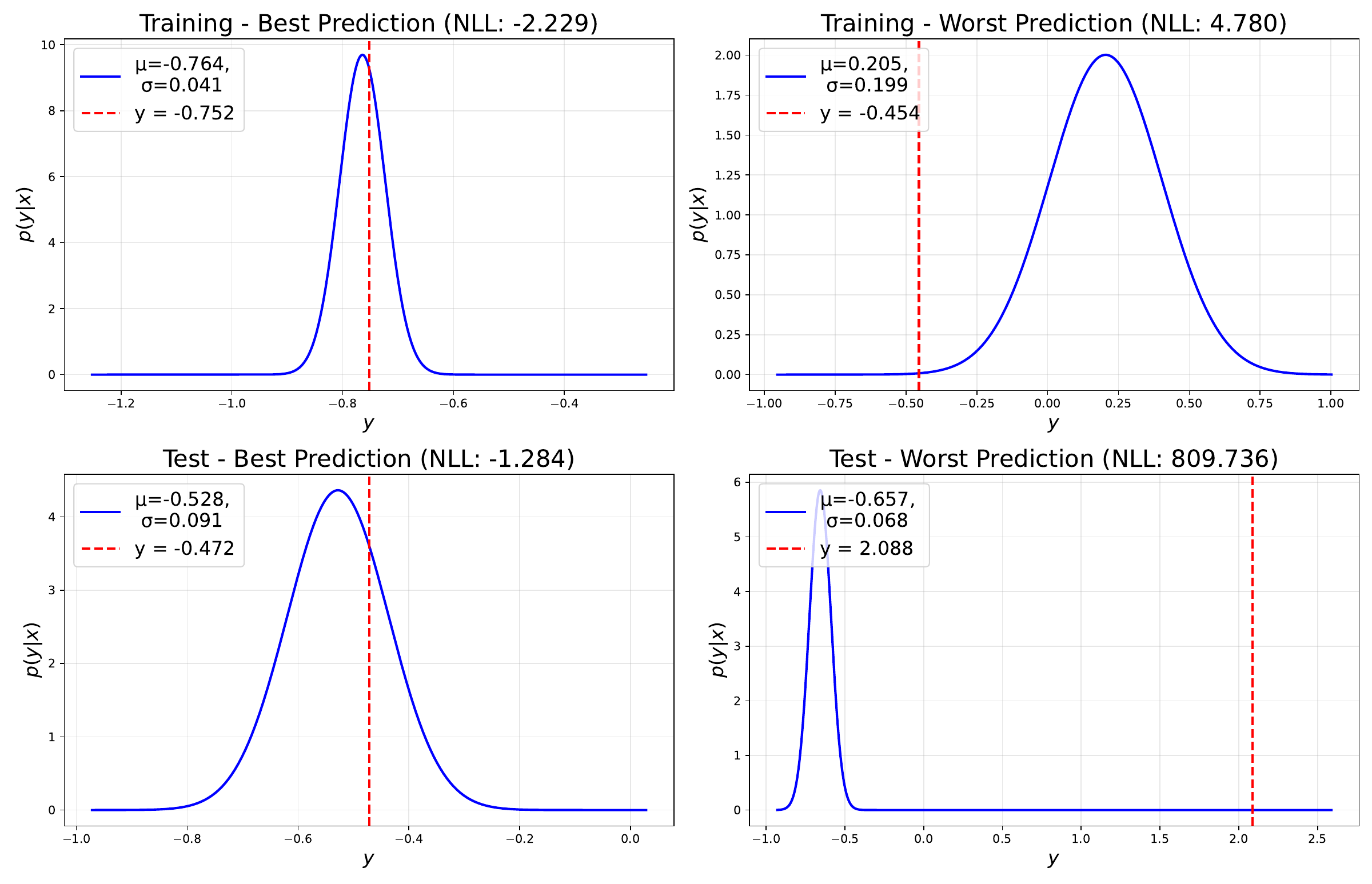}
        \caption{\textsf{M-ProbMix}.}
        \label{subfig: gme_transformer_manifold_probmixup}
    \end{subfigure}
    
    \caption{Predicted conditional densities with best and worst NLL values for Transformer model on GME dataset}
    \label{fig: gme_transformer_predictions}
\end{figure}

\end{document}